\documentclass[11pt,letterpaper]{article}

\usepackage[noend]{algorithmic}

\usepackage{titletoc}

\usepackage[margin=1in]{geometry}
\usepackage{lmodern}
\usepackage[T1]{fontenc}
\usepackage[utf8]{inputenc}
\usepackage[tbtags]{amsmath}
\allowdisplaybreaks
\usepackage{bbm}
\usepackage{amssymb,amsthm}
\usepackage{setspace} %
\usepackage[colorlinks,citecolor=blue,bookmarks=true,linktocpage]{hyperref}
\usepackage{url}
\usepackage[svgnames]{xcolor}
\definecolor{cblue}{rgb}{0,0,0}
\definecolor{cmagenta}{rgb}{0,0,0}

\definecolor{navy}{rgb}{0,0,.5}
\definecolor{cmaroon}{rgb}{0.7,0.1,0.2}

\hypersetup{
    colorlinks = True,
    linkcolor=cmaroon,
    citecolor=navy,
    urlcolor=magenta
}

\usepackage[style=alphabetic, maxbibnames=15, maxcitenames=3, natbib=true,
  maxalphanames=6, backend=biber, sorting=nty, backref=true]{biblatex}
\DefineBibliographyStrings{english}{%
  backrefpage = {page},%
  backrefpages = {pages},%
}

\usepackage{doi}
\usepackage{tikz} 
\usepackage{pgf,pgfplots}
\pgfplotsset{compat=1.14}
\usepgfplotslibrary{fillbetween}
\usepackage{hyphenat}
\usepackage[font=small]{caption}
\usepackage{subcaption}
\usepackage{appendix}%
\usepackage{thm-restate}
\usepackage{comment}

\usepackage[ruled,noend,noline]{algorithm2e}

\SetCommentSty{mycommfont}

\usepackage[capitalise,sort]{cleveref}

\usepackage{dsfont}
\usepackage{microtype,bm}
\usepackage{booktabs}
\usepackage{colortbl}
\usepackage{xcolor}
\usepackage{nicefrac}

\makeatletter
\def\cleartheorem#1{%
    \expandafter\let\csname#1\endcsname\relax
    \expandafter\let\csname c@#1\endcsname\relax
}
\makeatother
\cleartheorem{example}
\cleartheorem{theorem}
\cleartheorem{lemma}
\cleartheorem{proposition}
\cleartheorem{remark}
\cleartheorem{corollary}
\cleartheorem{definition}
\cleartheorem{conjecture}
\cleartheorem{axiom}

\newtheorem{theorem}{Theorem}
\newtheorem*{theorem*}{Theorem}
\newtheorem{lemma}[theorem]{Lemma}
\crefname{lemma}{Lemma}{Lemmas}

\newtheorem{remark}[theorem]{Remark}

\newtheorem{definition}[theorem]{Definition}
\crefname{definition}{Definition}{Definitions}

\newtheorem{fact}[theorem]{Fact}

\newtheorem{assumption}{Assumption} %
\crefname{assumption}{Assumption}{Assumptions}

\newcommand{\R}{\mathbb{R}}

\renewcommand{\O}{\mathcal{O}}

\newcommand{\N}{\mathcal{N}}

\newcommand{\tO}{\widetilde{\O}}

\newcommand{\E}{\mathcal{E}}

\newcommand{\eps}{\varepsilon}
\newcommand{\poly}{\mathrm{poly}}

\usepackage{wasysym}

\newenvironment{custproof}[1][]{%
  \renewcommand{\proofname}{Proof\if #1\empty \else \ (#1) \fi}\proof}{\endproof}

\newcommand{\EXP}[2][]{
    \ifthenelse{\equal{#1}{}}
    {\mathbb{E}\left[#2\right]}
    {\mathop{\mathbb{E}}_{#1}\left[#2\right]}
}
\newcommand{\PRO}[2][]{
    \ifthenelse{\equal{#1}{}}
    {\mathrm{Pr}\left[#2\right]}
    {\mathop{\mathrm{Pr}}_{#1}\left[#2\right]}
}
\newcommand{\parens}[1]{\left(#1\right)}
\newcommand{\braces}[1]{\left\{#1\right\}}

\newcommand{\logp}[1]{\log\left( #1 \right)}

\newcommand{\constOneStep}{c_{\mathrm{L}}}
\newcommand{\cCommon}{c_0}
\newcommand{\tcCommon}{\widetilde{c}_0}

\newcommand{\constBiasVal}{1}
\newcommand{\constBias}{\ifthenelse{\equal{\constBiasVal}{1}}{}{\constBiasVal}}
\newcommand{\constBiasInv}{\ifthenelse{\equal{\constBiasVal}{1}}{}{\frac{1}{\constBiasVal}}}
\newcommand{\constBiasValInv}{\ifthenelse{\equal{\constBiasVal}{1}}{1}{\frac{1}{\constBiasVal}}}
\newcommand{\constBiasSq}{\ifthenelse{\equal{\constBiasVal}{1}}{}{\constBiasVal^2}}
\newcommand{\cbadone}{c_{\mathrm{B}1}}
\newcommand{\cbadtwo}{c_{\mathrm{B}2}}
\newcommand{\Cbadone}{\mathrm{B}_1}

\newcommand{\bias}{\mathrm{bias}_t}

\newcommand{\F}{\mathcal{F}}
\newcommand{\Econd}[2]{\EXP{#2 \mid \F_{#1}}}
\newcommand{\Et}[1]{\Econd{t-1}{#1}}

\newcommand{\1}[1]{\mathbbm{1}{\left\{ #1 \right\}}}

\newcommand{\abs}[1]{\left| #1 \right|}
\newcommand{\norm}[1]{\left\lVert#1\right\rVert}

\newcommand{\normSq}[1]{\norm{#1}^2}

\newcommand{\innerProd}[2]{\left\langle #1, #2 \right\rangle}
\newcommand{\ceil}[1]{\left\lceil #1 \right\rceil}
\newcommand{\floor}[1]{\left\lfloor #1 \right\rfloor}

\newcounter{initialIterateT}
\newcounter{initialIterateTMinusOne}
\newcommand{\initialIterateTime}{\the\value{initialIterateT}}
\newcommand{\initialIterateTimeMinusOne}{\the\value{initialIterateTMinusOne}}
\newcommand{\w}{\mathbf{w}}
\newcommand{\x}{\mathbf{x}}
\newcommand{\s}{\mathbf{s}}

\newcommand{\0}{\mathbf{0}}
\newcommand{\wzero}{\w_{\initialIterateTime}}

\newcommand{\fat}[1]{F(#1)}
\newcommand{\f}[1]{\fat{\w_{#1}}}
\newcommand{\fstar}{F^*}
\newcommand{\fzero}{\f{1}}
\newcommand{\ft}{\f{t}}
\newcommand{\ftplus}{\f{t+1}}

\newcommand{\sigmaZero}{\sigma_0}
\newcommand{\sigmaOne}{\sigma_1}

\newcommand{\Lzero}{L_0}
\newcommand{\Lone}{L_1}
\newcommand{\smoothBoth}{(\Lzero, \Lone)}

\newcommand{\gradat}[1]{\nabla F(#1)}
\newcommand{\grad}[1]{\gradat{\w_{#1}}}

\newcommand{\gradt}{\grad{t}}

\newcommand{\gradnorm}[1]{\norm{\grad{#1}}}
\newcommand{\gradnormat}[1]{\norm{\gradat{#1}}}
\newcommand{\gradtnorm}{\gradnorm{t}}
\newcommand{\gradzeronorm}{\gradnorm{\initialIterateTime}}
\newcommand{\gradsq}[1]{\normSq{\grad{#1}}}
\newcommand{\gradtsq}{\gradsq{t}}
\newcommand{\gradsqat}[1]{\normSq{\gradat{#1}}}
\newcommand{\gradzerosq}{\gradsq{\initialIterateTime}}

\newcommand{\hessat}[1]{\nabla^2 F(#1)}

\newcommand{\multnoiseat}[1]{\xi_{\textrm{mult}\ifthenelse{\equal{#1}{}}{}{,#1}}}
\newcommand{\multnoiset}{\multnoiseat{t}}
\newcommand{\multnoise}{\multnoiseat{}}
\newcommand{\addnoiseat}[1]{\xi_{\textrm{add}\ifthenelse{\equal{#1}{}}{}{,#1}}}

\newcommand{\addnoise}{\addnoiseat{}}
\newcommand{\sgradat}[1]{\boldsymbol{g}(#1)}
\newcommand{\sgrad}[1]{\boldsymbol{g}_{#1}}
\newcommand{\sgradt}{\sgrad{t}}
\newcommand{\sgradzero}{\sgrad{\initialIterateTime}}
\newcommand{\sgradnorm}[1]{\norm{\sgrad{#1}}}
\newcommand{\sgradtnorm}{\sgradnorm{t}}
\newcommand{\sgradsq}[1]{\normSq{\sgrad{#1}}}
\newcommand{\sgradsqat}[1]{\normSq{\sgradat{#1}}}
\newcommand{\sgradtsq}{\sgradsq{t}}

\newcommand{\sgradatj}[1]{\boldsymbol{g}_j(#1)}
\newcommand{\sgradatjp}[1]{\boldsymbol{g}_{j'}(#1)}
\newcommand{\tsgradat}[1]{\widetilde{\boldsymbol{g}}(#1)}

\newcommand{\m}{\boldsymbol{m}}
\newcommand{\mt}{\m_t}
\renewcommand{\u}{\boldsymbol{u}}
\newcommand{\ut}{\u_t}

\newcommand{\tgradnorm}[1]{\|\widetilde{\nabla}_{#1}\|}
\newcommand{\tgradtnorm}{\tgradnorm{t}}
\newcommand{\tgradsq}[1]{\tgradnorm{#1}^2}
\newcommand{\tgradtsq}{\tgradsq{t}}

\newcommand{\etaat}[1]{\eta_{#1}}
\newcommand{\etat}{\etaat{t}}

\newcommand{\tetaat}[1]{\tilde{\eta}_{#1}}
\newcommand{\tetat}{\tetaat{t}}
\newcommand{\tbt}{\widetilde{b}_t}
\newcommand{\tbtp}{\widetilde{b}_{t'}}

\newcommand{\tetalb}{\underline{\tetaat{\stopat{T+1}}}}

\newcommand{\stopatdel}[2]{\tau_{#1}(#2)}
\newcommand{\stopat}[1]{\stopatdel{#1}{\delta}}
\newcommand{\stopt}{\stopat{t}}
\newcommand{\stopthm}{\tau_{T+1}}

\newcommand{\stopetaat}[1]{\tau_{\mathrm{step}}\ifthenelse{\equal{#1}{}}{}{(#1)}}

\newcommand{\varianceEvent}{\E_{\sigmaZero}}

\newcommand{\ncEvent}{\E_{\text{nc}}}

\newcommand{\xClip}{x_{\mathrm{clip}}}

\newcommand{\clipStepMult}{\lambda_{\mathrm{clip}}}

\newcommand{\tN}{\widetilde{N}}
\newcommand{\tu}{\widetilde{u}}

\newcommand{\Sgat}[1]{S_{\mathrm{good}}(#1)}
\newcommand{\Sg}{\Sgat{T}}
\newcommand{\Sgstopped}{\Sgat{\stopat{T+1}}}
\newcommand{\Sgthm}{\Sgat{\stopthm}}
\newcommand{\Scomp}[1]{S_{[#1]}^{\mathrm{comp}}}

\newcommand{\Scat}[1]{S^{\mathrm{comp}}(#1)}

\newcommand{\Scstopped}{\Scat{\stopat{T+1}}}
\newcommand{\Scthm}{\Scat{\stopthm}}

\newcommand{\tSgat}[1]{\widetilde{S}(#1)}

\newcommand{\tSgstopped}{\tSgat{\stopat{T+1}}}
\newcommand{\tSgthm}{\tSgat{\stopthm}}

\newcommand{\Eg}{E_{\mathrm{good}}}
\newcommand{\Eb}{E_{\mathrm{bad}}}
\newcommand{\ttb}[1]{\widetilde{\tau}_{[#1]}}
\newcommand{\tb}[1]{\tau_{[#1]}^{\mathrm{bad}}}
\newcommand{\ttg}[1]{\widetilde{\tau}_{[#1]}^{\mathrm{good}}}
\newcommand{\tg}[1]{\tau_{[#1]}^{\mathrm{good}}}

\newcommand{\ncomp}{n_{\mathrm{comp}}}
\newcommand{\comp}[1]{\mathrm{comp}(#1)}
\newcommand{\compstop}{\comp{\stopat{T+1}}}
\newcommand{\compstopthm}{\comp{\stopthm}}
\newcommand{\compstopcor}{\comp{T}}

\newcommand*\samethanks[1][\value{footnote}]{\footnotemark[#1]}

\author{Matthew Faw\thanks{Equal contribution} \thanks{Department of Electrical and Computer Engineering, University of Texas at Austin.\newline{\tt\{\href{mailto:matthewfaw@utexas.edu}{matthewfaw},\href{mailto:litu.rout@utexas.edu}{litu.rout},\href{mailto:constantine@utexas.edu}{constantine},\href{mailto:sanjay.shakkottai@utexas.edu}{sanjay.shakkottai}\}@utexas.edu}}
\and
Litu Rout\samethanks[1] \samethanks[2]
\and
Constantine Caramanis\samethanks[2]
\and
Sanjay Shakkottai\samethanks[2]
}

\title{Beyond Uniform Smoothness: A Stopped Analysis of Adaptive SGD}

\date{}

\makeatletter
\let\original@algocf@latexcaption\algocf@latexcaption
\long\def\algocf@latexcaption#1[#2]{%
  \@ifundefined{NR@gettitle}{%
    \def\@currentlabelname{#2}%
  }{%
    \NR@gettitle{#2}%
  }%
  \original@algocf@latexcaption{#1}[{#2}]%
}
\makeatother

\begin{document}
\maketitle

\begin{abstract}
This work considers the problem of finding a first-order stationary point of a non-convex function with potentially unbounded smoothness constant using a stochastic gradient oracle. We focus on the class of $(L_0,L_1)$-smooth functions proposed by Zhang et al.\ (ICLR'20). Empirical evidence suggests that these functions more closely captures practical machine learning problems as compared to the pervasive $L_0$-smoothness. This class is rich enough to include highly non-smooth functions, such as $\exp(L_1 x)$ which is $(0,\mathcal{O}(L_1))$-smooth. Despite the richness,  an emerging line of works achieves the $\widetilde{\mathcal{O}}(\nicefrac{1}{\sqrt{T}})$ rate of convergence when the noise of the stochastic gradients is deterministically and uniformly bounded. This noise restriction is not required in the $L_0$-smooth setting, and in many practical settings is either not satisfied, or results in weaker convergence rates with respect to the noise scaling of the convergence rate.

We develop a technique that allows us to prove $\mathcal{O}(\nicefrac{\mathrm{poly}\log(T)}{\sqrt{T}})$ convergence rates for $(L_0,L_1)$-smooth functions without assuming uniform bounds on the noise support. The key innovation behind our results is a carefully constructed stopping time $\tau$ which is simultaneously ``large'' on average, yet also allows us to treat the adaptive step sizes before $\tau$ as (roughly) independent of the gradients. For general $(L_0,L_1)$-smooth functions, our analysis requires the mild restriction that the multiplicative noise parameter $\sigma_1 < 1$. For a broad subclass of $(L_0,L_1)$-smooth functions, our convergence rate continues to hold when $\sigma_1 \geq 1$. By contrast, we prove that many algorithms analyzed by prior works on $(L_0,L_1)$-smooth optimization diverge with constant probability even for smooth and strongly-convex functions when $\sigma_1 > 1$.
\end{abstract}

\section{Introduction}
\label{sec:intro}
A fundamental problem in stochastic optimization is to characterize the convergence behavior of the Stochastic Gradient Descent algorithm:
\begin{align}\label{eq:sgd}\tag{SGD}
    \w_{t+1} = \w_t - \etat \sgradat{\w_t},
\end{align}
where $\etat$ is the step-size schedule, and $\sgradat{\w_t}$ is a stochastic gradient at iterate $\w_t$. Starting from \citep{RM51}, a long line of work has established conditions under which \eqref{eq:sgd} converges to a stationary point. A standard setting since \citep{PT73} used for this purpose has the following properties: $(a)$ The objective function $\fat{\cdot}$ is $\Lzero$-smooth, i.e., has $\Lzero$-Lipschitz gradients; $(b)$ $\fat{\cdot}$ has a finite lower bound, i.e., $\inf_{\w\in\R^d}\fat{\w} \geq \fstar > -\infty$; $(c)$ For each $\w\in\R^d$, the stochastic gradient $\sgradat{\w}$ is unbiased and has variance scaling at most affinely with $\gradsqat{\w}$, i.e.,
\begin{align}\label{eq:affineVarianceIntro}\tag{Affine-var}
    \EXP{\sgradat{\w}} = \gradat{\w}
    \quad\text{and}\quad
    \EXP{\normSq{\sgradat{\w} - \gradat{\w}}} \leq \sigmaZero^2 + \sigmaOne^2 \gradsq{}.
\end{align}

Much of the literature on stochastic optimization, e.g., \citep{NY83,GL13,B15,FSSSSW19},
focuses on a special case of \eqref{eq:affineVarianceIntro} where the variance is uniformly upper-bounded ($\sigmaOne = 0$):
\begin{align}\label{eq:boundedVarianceIntro}\tag{Bounded-var}
    \EXP{\sgradat{\w}} = \gradat{\w}
    \quad\text{and}\quad
    \sup_{\w\in\R^d}\EXP{\normSq{\sgradat{\w} - \gradat{\w}}} \leq \sigmaZero^2.
\end{align}
Rates of convergence to a first-order stationary point in these settings are now well-understood. Under \eqref{eq:boundedVarianceIntro} regime, \citet{GL13} prove an $\O(\nicefrac{\sqrt{\sigmaZero^2 \Lzero (\fzero - \fstar)}}{\sqrt{T}})$ rate of convergence with a fixed step-size schedule. Later, \citet{ACDFSW22} show that this rate is optimal up to constant factors. Further, as noted by \citet{BCN18}, a minor modification to this step-size gives nearly the same rate in the more general \eqref{eq:affineVarianceIntro} setting, i.e., $\sigmaOne > 0$. This rate is obtained by making trivial changes to the proof technique of \citet{GL13}.

One crucial assumption in these lines of work is $\Lzero$-smoothness, i.e., $\Lzero$-Lipschitz gradients of the loss landscape. However, recent works~\citep{ZJFW20,ZHSJ20} provide empirical evidence that this assumption is often not satisfied in practical machine learning problems. For instance, in large-scale language modeling including BERT~\citep{devlin2018bert} and other variants~\citep{radford2021learning,caron2021emerging,liu2023pre}, the loss landscape of transformer architectures either does not satisfy the $\Lzero$-smoothness assumption, or the value of $\Lzero$ becomes so large that it produces a significantly weaker rate of convergence~\citep{ZJFW20,ZHSJ20}. 

Aiming to address these issues, there has been a recent surge of interest in relaxing the standard $\Lzero$-smoothness assumption and characterizing the rate of convergence. One appealing relaxation proposed by \citet{ZHSJ20} is that of $\smoothBoth$-smoothness\footnote{For convenience, we state this assumption in terms of a bound on the hessian of $F$. The requirement that the hessian exists everywhere can be relaxed to a condition on the gradients \citep{ZJFW20}. This relaxation is the one we use for our main results, see \cref{assump:smooth}.}:
\begin{align}\label{eq:generalizedSmoothIntro}\tag{Generalized-smooth}
    \norm{\hessat{\w}} \leq \Lzero + \Lone\norm{\gradat{\w}}.
\end{align}
While every $\Lzero$-smooth function is also $(\Lzero, 0)$-smooth, this relaxation admits functions that grow significantly faster than a quadratic function, e.g., $\fat{w} = w^d$ is $(\nicefrac{d(d-1)}{\Lone^{d-2}}, (d-1)\Lone)$-smooth for any $\Lone > 0$, and $F(w) = \exp(\Lone w)$ is $(0,\Lone)$-smooth. With regards to convergence, recent works \citep{ZHSJ20,CLOZZ22} establish an $\O(\nicefrac{1}{\sqrt{T}})$ rate in the $\smoothBoth$-smooth setting, as long as the noise of the stochastic gradients has \emph{uniformly-bounded support}, i.e.,
\begin{align}\label{eq:boundedSuppIntro}\tag{Bounded-supp}
    \EXP{\sgradat{\w}} = \gradat{\w}
    \quad\text{and}\quad
    \sup_{\w\in\R^d}\norm{\sgradat{\w} - \gradat{\w}}^2 \overset{a.s.}{\leq} B^2.
\end{align}

The algorithms achieving the rate $\O(\nicefrac{1}{\sqrt{T}})$ in this setting use adaptive step size schedules -- i.e., variants of \eqref{eq:sgd} with $\etat$ chosen as a function of $\{\sgrad{s}\}_{s\in[t]}$. \eqref{eq:boundedSuppIntro} is a common assumption in these analyses~\citep{ZHSJ20,ZJFW20,CLOZZ22}. It is typically introduced to reason about the direction of $-\etat \sgradt$ relative to the true descent direction. In the analysis of standard SGD, the fixed step-size schedule does not depend upon the stochastic gradients, so $\EXP{-\etat \sgradt}=-\etat\EXP{\gradt}$. This, however, is not the case for adaptive methods, since $\etat$ depends on $\sgradt$. Thus, it is understandable why prior works~\citep{ZHSJ20,ZJFW20} assume \eqref{eq:boundedSuppIntro} to simplify this issue. Further, \eqref{eq:boundedSuppIntro} is natural in settings where the stochastic gradients satisfy
    $\sgradat{\w} = \gradat{\w} + \xi,$
where the random vector $\xi$ has bounded support (or bounded second moment in the related setting of \eqref{eq:boundedVarianceIntro}).

In many real-world scenarios, the \eqref{eq:boundedSuppIntro} assumption does not hold. For instance, when running SGD in standard least-squares regression settings, the stochastic gradients have multiplicative noise, as noted in~\citep{DFB17,FB17,JKKNS18,JT19}. 
Similar noise assumptions have also been considered, e.g., in convergence of stochastic proximal gradient methods \citep{RVV20}, Hilbert-valued stochastic subgradient methods \citep{BRS07}, and adaptive gradient methods \citep{FTCMSW22}.
Moreover, multiplicative noise naturally arises in machine learning problems with (additive or multiplicative) feature noise~\citep{LW11,H86,CRSC06}.
Thus, we believe that characterizing $\smoothBoth$-smooth functions under \eqref{eq:affineVarianceIntro} is an important step in extending the theory of non-convex stochastic optimization beyond the standard $\Lzero$-smooth setting.

\subsection{Contributions}
A major challenge in the analysis of adaptive stochastic gradient descent is the correlation between the stochastic gradients and the step-size. Here, we develop a technique to simplify this challenge. Our key innovation is a \emph{recursively-defined stopping time} which satisfies two crucial properties: (i) before the stopping time is reached, the step sizes behave roughly independently of the gradients, and (ii) on average, the stopping time is at least a constant fraction of the time horizon. As a consequence, instead of analyzing over the entire time horizon, we conduct the analysis over this sub-interval over which we exploit this convenient almost-independent property. This tool allows us to prove the first $\tO(\nicefrac{1}{\sqrt{T}})$ rate of convergence for $\smoothBoth$-smooth functions beyond the \eqref{eq:boundedSuppIntro} setting. Our main contributions are three-fold:

(a) \textbf{Convergence for $\smoothBoth$-smoothness when $\sigmaOne <1$.} We show in \cref{sec:conv-adagrad-norm} that AdaGrad-Norm converges at a rate $\tO(\nicefrac{1}{\sqrt{T}})$ when the stochastic gradient oracle satisfies \eqref{eq:affineVarianceIntro} with $\sigmaZero\geq 0$ and $\sigmaOne \in [0,1)$. This is the first convergence rate for any algorithm even under \eqref{eq:boundedVarianceIntro} (i.e., $\sigmaOne = 0$) for general $\smoothBoth$-smooth optimization. Note that the scaling of this bound with $T$ matches (up to poly-logarithmic factors) the best-known rate for $\Lzero$-smooth functions -- with a minor caveat that $\sigmaOne < 1$ is not needed in the $\Lzero$-smooth setting. Also, we show that the rate improves to $\tO(\nicefrac{1}{T})$ in the ``small variance'' regime when $\sigmaZero,\sigmaOne \to 0$ even without tuning the step-size.

(b) \textbf{Convergence for all $\sigmaOne$.} We establish a sufficient condition under which AdaGrad-Norm converges at a rate $\tO(\nicefrac{1}{\sqrt{T}})$ when $\sigmaOne \geq 1$, see \cref{sec:key-ideas}. This condition allows us to analyze a broad subset of $\smoothBoth$-smooth functions that includes all $\Lzero$-smooth functions as well as fixed-degree polynomials without any restrictions on $\sigmaOne$. This simultaneously generalizes the result and simplifies a key proof technique of \citet{FTCMSW22} for $\Lzero$-smooth functions. 

(c) \textbf{Negative results for known algorithms.} We prove a set of negative results in \cref{sec:challengesMultNoise} for most algorithms analyzed under $\smoothBoth$-smoothness and \eqref{eq:boundedSuppIntro}. We construct an oracle for Clipped and Normalized SGD~\citep{ZHSJ20,ZJFW20} and Sign SGD with Momentum \citep{CLOZZ22} that leads to failure with constant probability in a wide parameter regime.   
We also prove that AdaGrad-Norm can diverge with constant probability if the step-size is not carefully tuned in the ``large variance'' regime for $\smoothBoth$-smooth functions. By contrast, no parameter tuning is needed in the $\Lzero$-smooth setting in this noise regime.

\section{Related Works}
\label{sec:rel-work}

\textbf{Stochastic gradient descent.} \eqref{eq:sgd} has been well-studied for many decades \citep{RM51}. 
\citet{PT73} proved almost-sure convergence to a first-order stationary point of \eqref{eq:sgd} for non-convex and $\Lzero$-smooth functions with $\fat{\w} \geq \fstar$ with stochastic gradient oracle satisfying (a slightly weaker condition than) \eqref{eq:affineVarianceIntro}. \citet{BT00} extended the result to a setting where $\fat{\w}$ does not have a uniform lower-bound. \citet{GL13} proved that \eqref{eq:sgd} with step-size 
$\etat = \eta = \min\braces{1/\Lzero, \sqrt{\nicefrac{2(\fzero - \fstar)}{(\Lzero \sigmaZero^2 T)}}}$
achieves a convergence rate to a first-order stationary point of $\O\parens{\sqrt{\nicefrac{\Lzero \sigmaZero^2 (\fzero - \fstar)}{T}}}$, assuming $\Lzero$-smoothness and \eqref{eq:boundedVarianceIntro}. \citet{DS20} proved that this is the optimal rate for \eqref{eq:sgd} without further assumptions. Recently, \citet{ACDFSW22} proved that the convergence rate of \citep{GL13} is optimal among \emph{all first-order} methods, not just SGD.

\textbf{AdaGrad step-sizes.} This paper builds on a long line of work studying (variants of) the AdaGrad step size schedule introduced by \citet{DHS11,MS10}. In particular, we focus on the so-called AdaGrad-Norm step-size, which was introduced in \citet{SM10}. While these works focused on the setting of online convex optimization, \citet{WWB20} demonstrated that AdaGrad-Norm converges at a rate $\tO(\nicefrac{1}{\sqrt{T}})$ in the context of $\Lzero$-smoothness,  \eqref{eq:boundedVarianceIntro}, and $M$-Lipschitzness, i.e., $\sup_{\w\in\R^d} \gradnormat{\w} \leq M$.
Around the same time, \citet{LO19} proved that AdaGrad-Norm achieves an $\tO(\nicefrac{1}{\sqrt{T}})$ rate without $M$-Lipschitzness. But their analysis needs tuning of the step-size with respect to the smoothness constant $\Lzero$. Later, \citet{KLC22} proved that AdaGrad-Norm converges at rate $\tO(\nicefrac{1}{\sqrt{T}})$ without tuning the step-size (as in \citet{LO19}) or assuming $M$-Lipschitz objective (as in \citet{WWB20}). However, their analysis holds only when the noise of the stochastic gradients is uniformly sub-Gaussian. In a concurrent work, \citet{FTCMSW22} proved that AdaGrad-Norm achieves $\tO(\nicefrac{1}{\sqrt{T}})$ in a setting identical to standard SGD (i.e., $\Lzero$-smooth objective with stochastic gradients satisfying \eqref{eq:affineVarianceIntro}), and without tuning the step-size with respect to $\Lzero$, $\sigmaZero$, or $\sigmaOne$. This work thus established that AdaGrad-Norm is parameter-free and enjoys nearly the same convergence rate as SGD in the standard non-convex setting.

\textbf{$\smoothBoth$-smoothness in the \eqref{eq:boundedSuppIntro} regime.}
Recent work by \citet{ZHSJ20} argued that the $\Lzero$-smoothness assumption is not realistic for many practical machine learning tasks, e.g., large-scale natural language processing using transformer architectures. Instead, they demonstrated that $\smoothBoth$-smooth functions \eqref{eq:generalizedSmoothIntro} better capture the loss landscape, and proved that the gradient clipping algorithm converges at a rate $\O(\nicefrac{1}{\sqrt{T}})$ in the \eqref{eq:boundedSuppIntro} regime. \citet{ZJFW20} later proved convergences for a generalized class of gradient clipping algorithms. They used a slightly weaker definition of $\smoothBoth$-smoothness, which we use in \cref{assump:smooth}. Very recently, \citet{CLOZZ22} considered a ``coordinate-wise'' generalization of $\smoothBoth$-smoothness, and proved that a ``generalized SignSGD'' algorithm converges at a rate $\tO(\nicefrac{1}{\sqrt{T}})$. By contrast, they proved that gradient descent with fixed step-sizes must scale linearly in $M \Lone$, where $M = \sup\braces{\gradnormat{\w} : \fat{\w} \leq \fat{\wzero}}$ is the largest gradient in the sublevel set $\fat{\w} \leq \fat{\wzero}$. Interestingly, this line of work establishes that adaptive step-size schedules can avoid this dependence on $M$.

\section{Problem Setting}
\label{sec:prob-set}
We are interested in finding a first-order stationary point of a non-convex function, given access to a stochastic gradient oracle, using \eqref{eq:sgd}. 
For compactness, let $\sgradt := \sgradat{\w_t}$. Our objective function $\fat{\w}$ satisfies the following:
\begin{assumption}[Lower-boundedness]\label{assump:lowerBounded}
    There exists an $\fstar > -\infty$ such that $\inf_{\w\in\R^d} \fat{\w} \geq \fstar.$
\end{assumption}
\begin{assumption}[$\smoothBoth$-smooth objective]\label{assump:smooth}
The objective function $\fat{\w}$ is $\smoothBoth$-smooth, i.e., for every $\w,\w' \in\R^d$ such that $\norm{\w - \w'} \leq \nicefrac{1}{\Lone}$
\begin{align*}
    \norm{\gradat{\w} - \gradat{\w'}} \leq (\Lzero + \Lone \norm{\gradat{\w'}})\norm{\w - \w'}.
\end{align*}
\end{assumption}
We note that $\smoothBoth$-smoothness was originally defined in \citep{ZHSJ20} as a bound on the Hessian of $\fat{\cdot}$, as in \eqref{eq:generalizedSmoothIntro}. Following \citep[Remark 2.3]{ZJFW20}, we choose to adopt the alternative condition in \cref{assump:smooth} for two reasons. First, \cref{assump:smooth} is strictly weaker than $\Lzero$-smoothness, since $(\Lzero, 0)$-smoothness implies the gradients are $\Lzero$-Lipschitz. Second, whenever the objective is twice-differentiable, \cref{assump:smooth} implies \eqref{eq:generalizedSmoothIntro} (up to constant factors in the definitions of $\Lzero$ and $\Lone$):
\begin{restatable}{proposition}{restateSmoothnessPartialEquiv}\label{prop:smoothPartialEquiv}
    A function satisfying $\smoothBoth$-smoothness as per \eqref{eq:generalizedSmoothIntro} is also $(2\Lzero, (e-1)\Lone)$-smooth as per \cref{assump:smooth}. If $\fat{\cdot}$ is twice continuously differentiable and $\smoothBoth$-smooth as per \cref{assump:smooth}, then it is also $\smoothBoth$-smooth as per \eqref{eq:generalizedSmoothIntro}.
\end{restatable}

Let $\F_t$ be the sigma-algebra generated by the interaction between the algorithm and stochastic gradient oracle for $t$ rounds, i.e., $\F_t := \sigma\braces{\wzero, \sgradzero,\ldots,\w_t, \sgradt, \w_{t+1}}$. We impose the following conditions on the stochastic gradients:
\begin{assumption}[Unbiased gradients]\label{assump:unbiasedGrad}
The stochastic gradients satisfy
    $\Et{\sgradt} = \gradt.$
\end{assumption}

\begin{assumption}[Affine variance]\label{assump:affineVariance}
There exist constants $\sigmaZero, \sigmaOne \geq 0$ such that the variance of each stochastic gradient $\sgradt$ is bounded above as:
    $\Et{\normSq{\sgradt - \gradt}} \leq \sigmaZero^2 + \sigmaOne^2 \gradtsq.$
\end{assumption}
\cref{assump:unbiasedGrad,assump:affineVariance} imply the following bound on the stochastic gradients in terms of the true gradient:
\begin{equation}\label{eq:sgradBound}
    \Et{\sgradtsq} \leq \sigmaZero^2 + (1+\sigmaOne^2)\gradtsq.
\end{equation}
We are interested in studying algorithms which require as little hyper-parameter tuning as possible and, simultaneously, can handle potentially unbounded smoothness constant. To achieve this, we analyze AdaGrad-Norm \citep{SM10}, a step-size sequence $\etat$ for \eqref{eq:sgd} which, at each time $t$, depends on the current and past stochastic gradients $\braces{\sgrad{s}}_{s\in[t]}$:
\begin{equation}\tag{AG-Norm}\label{eq:alg}
    \etat = \frac{\eta}{b_t}, \quad \text{where} \quad b_t^2 = b_0^2 + \sum_{s=1}^t \sgradsq{s} = b_{t-1}^2 + \sgradtsq.
\end{equation}
As is increasingly common in the analysis of (variants of) \eqref{eq:sgd} with adaptive step-sizes \citep{WWB20,FTCMSW22,defossez2022a}, our analysis will rely on a ``decorrelated'' step-size $\tetat$. The key property of $\tetat$ is that it is independent of $\sgradt$ when conditioned on the filtration $\F_{t-1}$.
\begin{definition}[Decorrelated step-sizes]\label{def:stepSizeProxy}
For each step-size $\etat$ at time $t\geq 1$, the decorrelated step size $\tetat$ is defined to be 
$\tetaat{t} := \nicefrac{\eta}{\tbt},$
where $\tbt^2 := b_{t-1}^2 + \tgradtsq$, $b_0^2 > 0$, and $\tgradtsq := \sigmaZero^2 + \gradtsq$.
\end{definition}
This ``decorrelated'' step-size serves as a proxy in our analysis for the true step-size $\etat$. The main reason for its introduction is that, although $\EXP{\etat\sgradt}\neq\EXP{\etat\gradt}$ (since $\etat$ depends on $\sgradt$), the proxy satisfies $\Et{\tetat\sgradt}=\tetat\gradt$.

\section{Convergence of AdaGrad-Norm on $\smoothBoth$-smooth functions}
\label{sec:conv-adagrad-norm}

Our main results, \cref{thm:mainInformal,thm:polyBoundedConvergenceMain}, both establish $\tO(\nicefrac{1}{\sqrt{T}})$ convergence rates for \eqref{eq:alg} in the $\smoothBoth$-smooth regime under \eqref{eq:affineVarianceIntro}. \cref{thm:mainInformal} holds for any $\smoothBoth$-smooth function under a mild restriction that $\sigmaOne < 1$. It is easy to extend this result for $\sigmaOne \geq 1$ by computing mini-batch gradients with a batch size $B\approx \sigmaOne^2$, refer \cref{fact:minibatching} for a proof. Despite the restriction of \cref{thm:mainInformal} to $\sigmaOne < 1$, we emphasize that, prior to our work, no proof of convergence even for the \eqref{eq:boundedVarianceIntro} setting (i.e., $\sigmaOne = 0$) was known for a general class of $\smoothBoth$-smooth functions. Besides, \cref{thm:polyBoundedConvergenceMain} holds for all $\sigmaOne$ and a subclass of $\smoothBoth$-smooth functions, i.e., excluding functions like $\exp(\Lone x)$.

\begin{theorem}[Informal statement of \cref{thm:main}]\label{thm:mainInformal}
Fix any constants $\eps,\eps',\eps'',\eps''' \in (0,1)$ such that $\eps+\eps'+\eps''+\eps''' < 1$.
Consider \eqref{eq:alg} with any parameters $\eta \leq \nicefrac{2\eps'}{5\Lone}$ and $b_0^2 > 0$, running on an objective function satisfying \cref{assump:smooth}, and given access to a stochastic gradient oracle satisfying \cref{assump:unbiasedGrad,assump:affineVariance}. Assuming that $\sigmaOne \leq \constBiasInv\parens{1-(\eps+\eps'+\eps''+\eps''')}$, then for any $T\geq 1$ and $\delta' \in (0,1)$, with probability at least $1-\delta'$, the iterates of \eqref{eq:alg} satisfy 
\begin{align*}
    \min_{t\in[T]} \gradtsq 
    &\lesssim 
    \frac{\sigmaZero}{(\delta')^2 \sqrt{T}} h(T)
    +\frac{\sigmaOne\sqrt{\sigmaZero}}{(\delta')^{2.25}\sqrt{T}}h(T)^{\nicefrac{3}{2}}
    +\frac{\sigmaOne\sqrt{(1+\sigmaOne^2)}}{(\delta')^{2.5}\sqrt{T}}h(T)^2\\
    &+ \frac{1}{(\delta')^2 T}h(T)^2
    + \frac{b_0}{(\delta')^2 T}h(T)
    + \frac{\sigmaOne\sqrt{b_0 + \eta\Lzero} h(T)^{1.5}}{(\delta')^2 T},\\
    \text{where}\quad
    h(T) &\propto \frac{1}{\eps'''}\parens{\frac{\fzero - \fstar}{\eta} + \frac{\eps''\sigmaZero}{1+\sigmaOne^2} + \parens{\frac{\sigmaZero}{\eps} + \eta\Lzero}\log(g(T))},\\
    g(T) &\propto \frac{T (1+\sigmaOne^2) \parens{\frac{\sigmaZero}{\eps} + \eta\Lzero}}{\eps''b_0}.
\end{align*}
\end{theorem}

To extend our convergence proofs beyond $\sigmaOne < 1$, we consider a subclass of $\smoothBoth$-smooth functions which satisfy the following additional assumption:

\begin{restatable}{definition}{restatePolyBoundedDef}\label{def:polynomiallyBounded}
    A function $\fat{\cdot}$ is \emph{$k$-polynomially bounded} for $k\geq 2$ if  $\forall \w,\w'\in\R^d$, then there are constants $c_k\geq 1$ and $c_k', \Lzero> 0$ such that:
    \begin{align*}
        \gradnormat{\w} - c_k \gradnormat{\w'} \leq \max\braces{c_k' \norm{\w - \w'}^{k-1}, \Lzero\norm{\w - \w'}}.
    \end{align*}
\end{restatable}

Notice that, whereas \cref{assump:smooth} is a local constraint on the objective, \cref{def:polynomiallyBounded} enforces a global polynomial growth constraint -- thus ruling out such $\smoothBoth$-smooth functions as exponentials, while capturing a significantly broader class of functions than $\Lzero$-smoothness. We refer the interested reader to \cref{prop:polyBoundProperties} for some properties of this class of functions. Using this definition, we are able to prove the following:

\begin{theorem}[Informal statement of \cref{cor:polyBoundedConvergence}]\label{thm:polyBoundedConvergenceMain}
    Fix any constants $\eps,\eps',\eps'',\eps''' \in (0,1)$ such that $\eps+\eps'+\eps''+\eps''' < 1$.
    Consider \eqref{eq:alg} with any parameters $\eta \leq \nicefrac{2\eps'}{\Lone(4 + \sigmaOne^2)}$ and $b_0^2 > 0$, running on an objective function satisfying \cref{assump:smooth} and \cref{def:polynomiallyBounded} for some constants $k\geq 2, c_k\geq 1, c_k'>0$, and given access to a stochastic gradient oracle satisfying \cref{assump:unbiasedGrad,assump:affineVariance} for any $\sigmaZero, \sigmaOne \geq 0$. Then, for any $T\geq 1$ and $\delta' \in (0,1)$, with probability at least $1-\delta'-\tO(\nicefrac{1}{T})$, the iterates of \eqref{eq:alg} satisfy
    \begin{align*}
        \min_{t\in[T]} \gradtsq
        &\lesssim
        \frac{\sigmaZero}{(\delta')^2 \sqrt{T}} \widetilde{h}(T)
        +\frac{\sigmaOne\sqrt{\sigmaZero(1+c_k^2)}}{(\delta')^{2.25}\sqrt{T}}\widetilde{h}(T)^{\nicefrac{3}{2}}
        +\frac{\sigmaOne(1+c_k^2)\sqrt{(1+\sigmaOne^2)}}{(\delta')^{2.5}\sqrt{T}}\widetilde{h}(T)^2\\
        &\quad+\frac{\sigmaOne\sqrt{1 + c_k^2}\sqrt[4]{(1+\sigmaOne^2)\cbadone} \widetilde{h}(T)^{1.5}}{(\delta')^{2.25} T^{\nicefrac{3}{4}}}\\
        &+ \frac{\parens{b_0 + \sqrt{(1+\sigmaOne^2)\cbadone}}\widetilde{h}(T)}{(\delta')^2 T}
        + \frac{\sigmaOne\sqrt{(1+c_k^2)(b_0 + \eta\Lzero)} \widetilde{h}(T)^{1.5}}{(\delta')^2 T}
        + \frac{(1+c_k^2)\widetilde{h}(T)^2}{(\delta')^2 T}
    \end{align*}
    where $\widetilde{h}(T) = h(T) + \nicefrac{\compstopcor}{\eps'''\eta}$, where $h(T)$ is the function defined in \cref{thm:mainInformal}, and
    \begin{align*}
        \compstopcor &\propto \eta\sigmaOne c_k\parens{\gradzeronorm \ell_1(T) + (c_k'\eta^{k-1}+\Lzero\eta)\parens{\nicefrac{4 c_k^3\sigmaOne}{\eps'''}}^{k-1}\ell_k(T)}\\
        \cbadone &\propto (c_k'\eta^{k-1}+\eta\Lzero)^2(1 + \nicefrac{c_k^3 \sigmaOne}{\eps'''})^{2k-1}\ell_{2k-1}(T)
        + c_k^2\gradzerosq (1 + \nicefrac{c_k^3 \sigmaOne}{\eps'''})\ell_1(T)\\
        \ell_k(T) &\propto \parens{\frac{(k+1)\sigmaOne^2\logp{e + 8e\frac{\sigmaZero^2 T^2 + (1+\sigmaOne^2)(T^2 c_k^2 \gradzerosq + (c_k' \eta^{k-1} + \eta\Lzero)^2T^{2k-1})}{b_0^2 \delta'}}}{(1-(\eps+\eps'+\eps''+\eps'''))^2} }^k
    \end{align*}
\end{theorem}

There are several notable takeaways from the above results.

\textbf{Noise adaptivity.}
Both \cref{thm:mainInformal,thm:polyBoundedConvergenceMain} provide ``noise-adaptive'' convergence rates, in a sense that as $\sigmaZero,\sigmaOne \to 0$, the convergence rates automatically improve from $\tO(\nicefrac{1}{\sqrt{T}})$ to $\tO(\nicefrac{1}{T})$ without any additional hyperparameter tuning.

\textbf{Less hyperparameter tuning.} These rates hold without tuning the parameters $b_0$ or $\eta$ with respect to $\sigmaZero$ or $\Lzero$, unlike all prior algorithms for $\smoothBoth$-smoothness that we are aware of \citep{ZHSJ20,ZJFW20,CLOZZ22}\footnote{This feature, however, manifests into a worse dependence on $\Lone$ unlike \citep{ZJFW20,CLOZZ22}.}. 

\textbf{Generalization of prior work.} We remark that \cref{thm:polyBoundedConvergenceMain} strictly generalizes the result of \citep{FTCMSW22} beyond the uniform $\Lzero$-smooth setting. Further, our stopped analysis simplifies their ``recursive improvement'' technique \citep[Lemma 13]{FTCMSW22}.

\textbf{Large variance regime.} One may wonder if tuning the \eqref{eq:alg} step-size with respect to $\Lone(1+\sigmaOne^2)$ is necessary. When $\sigmaOne = \poly\log(T)$, the answer is yes. As we prove in \cref{lem:divergenceAGNormMain}, if $\eta \geq \nicefrac{1}{\Lone\sqrt{\sigmaOne}}$, then it can diverge with constant probability. By contrast, no tuning is necessary for it to converge for $\Lzero$-smooth functions in this noise regime.

\section{Key technical ideas}
\label{sec:key-ideas}
As discussed earlier, the main technical tool we use to obtain our convergence rates in \cref{thm:mainInformal,thm:polyBoundedConvergenceMain} is a recursively-defined stopping time. Before we are ready to define this time and discuss its utility, we first give a brief overview of the main initial steps of our analysis.

The standard first step in the analysis of SGD-like algorithms for $\Lzero$-smooth non-convex optimization is to prove that, at least on average, each update makes sufficient progress. This argument typically relies on the following inequality for $\Lzero$-smooth functions: for any $\w,\w'\in\R^d$,
\begin{align*}
    \fat{\w'} - \fat{\w} \leq \innerProd{\gradat{\w}}{\w'-\w} + \frac{\Lzero}{2}\normSq{\w'-\w}.
\end{align*}
This inequality is no longer true for $\smoothBoth$-smooth functions. Indeed, it is clearly not satisfied for all $\w,\w'$ on the $(0,\Lone)$-smooth function $\exp(\Lone x)$. However, \citet{ZHSJ20,ZJFW20} note that a similar variant holds ``locally'' for $\norm{\w-\w'}\leq \nicefrac{1}{\Lone}$ (see \cref{lem:localSmoothnessBound}). Using this variant, we obtain the following inequality, which is our first tool for studying the convergence of \eqref{eq:alg}.

\begin{restatable}{lemma}{restateStartingPoint}\label{lem:startingPoint}
Fix any $\eps,\eps'\in (0,1)$.
Suppose that $\eta \leq \frac{2\eps'}{\Lone(4+\sigmaOne^2)}$. Then, for any $t$,
\begin{align*}
    \Et{\ftplus - \ft}
    &\leq -\tetat\parens{1-\eps-\eps' - \sigmaOne\bias}\gradtsq + \tcCommon \Et{\nicefrac{\sgradtsq}{b_t^2}},
\end{align*}
where $\tcCommon = \frac{\eta\sigmaZero}{2\eps} + \eta^2\frac{\Lzero + \sigmaZero\Lone}{2}$ and $\bias = \constBias\sqrt{\Et{\nicefrac{\sgradtsq}{b_t^2}}}$.
\end{restatable}

Notice that \cref{lem:startingPoint} only guarantees that the algorithm makes progress on average moving from $\w_t$ to $\w_{t+1}$ when $\sigmaOne\bias < 1$, and is essentially vacuous otherwise. To handle this issue, we use the notion of ``good times'' from \citep{FTCMSW22}:
\begin{restatable}[Good times]{definition}{restateGoodTimesMain}\label{def:goodTimesMain}
A time $t\in [T]$ is ``good'' if, for fixed parameters $\eps,\eps',\eps'',\eps'''\in (0,1)$ satisfying $\eps+\eps'+\eps''+\eps'''<1$, $1-\eps-\eps'-\eps''-\sigmaOne\bias \geq \eps'''$.
We denote, for any stopping time $\tau$ with respect to $(\F_{s-1})_{s\geq 1}$, $\Sgat{\tau} = \braces{1 \leq t < \tau : \text{$t$ is ``good''}}$ as the set of all such ``good'' times before $\tau$, and $\Sgat{\tau}^c = [\tau-1]\setminus\Sgat{\tau}$ to be the remaining ``bad'' times before $\tau$.
\end{restatable}

Intuitively, the ``good'' times are those times when \cref{lem:startingPoint} is non-vacuous. Using \cref{def:goodTimesMain}, we sum the expression \cref{lem:startingPoint} until any stopping time $\tau$ to obtain the following more useful form.

\begin{restatable}[Descent lemma]{lemma}{restateDescentLemma}\label{lem:descentLemma}
Fix any $\eps,\eps',\eps'',\eps'''\in (0,1)$ such that $\eps + \eps'+\eps''+\eps'''<1$.
For any $\smoothBoth$-function, if we run AdaGrad-Norm with parameters $\eta \leq \frac{2\eps'}{\Lone\parens{4+\sigmaOne^2}}$ and $b_0^2 > 0$ for $T$ time steps, then, for any stopping time $\tau \in [2,T+1]$ with respect to $(\F_{s-1})_{s\geq 1}$, and any $\tSgat{\tau} \subseteq \Sgat{\tau}$:
\begin{align*}
    \eps'''\EXP{\sum_{t\in\tSgat{\tau}}\tetat\gradtsq} 
    &\leq \fzero - \fstar
    + 2\tcCommon\logp{\frac{(2+\sigmaOne^2)\tcCommon \EXP{\tau-1}}{\eta\eps'' b_0}}
    + \frac{2\eta\eps''\sigmaZero}{(2+\sigmaOne^2)}
    + \comp{\tau},
\end{align*}
where $\comp{\tau} := \EXP{\sum\limits_{t\in\Sgat{\tau}^c}(\sigmaOne-(1-\eps-\eps'))\tetat\gradtsq - \sum\limits_{t'\in\Scat{\tau}} \eps'''\tetaat{t'}\gradsq{t'}}$, the set $\Scat{\tau} := \Sgat{\tau} \setminus \tSgat{\tau}$ consists of the ``good'' times used to compensate for the bad times $\Sgat{\tau}^c$, and $\tcCommon = \frac{\eta \sigmaZero}{2\eps} + \eta^2 \frac{\Lzero + \sigmaZero \Lone}{2}$.
In particular, whenever $\sigmaOne \leq 1-(\eps+\eps')$, then $\comp{\tau} \leq 0$, and when $\sigmaOne \leq 1-(\eps+\eps'+\eps''+\eps''')$, then additionally $\Sgat{\tau} = [\tau-1]$.
\end{restatable}

\subsection{Using the descent lemma when $\sigmaOne < 1$}
Let us first analyze \cref{lem:descentLemma} in the simpler setting where $\sigmaOne \leq 1-(\eps+\eps'+\eps''+\eps''')$. Recall from \cref{lem:descentLemma} that this implies $\comp{\tau} \leq 0$ and we can take $\tSgat{\tau} = \Sgat{\tau} = [\tau-1]$. Thus, \cref{lem:descentLemma} loosely becomes $\EXP{\sum_{t < \tau} \tetat\gradtsq} \lesssim \log(T)$. At this point, if we choose $\tetat \approx \nicefrac{1}{\sqrt{T}}$ and $\tau=T+1$, then we could conclude that $\EXP{\frac{1}{T}\sum_{t\in[T]} \gradtsq} \lesssim \nicefrac{\log(T)}{\sqrt{T}}$. Unfortunately, as we discussed earlier, the first step of our analysis relies on the inequality $\norm{\w_{t+1} - \w_t} = \norm{\etat\sgradt}\leq \nicefrac{1}{\Lone}$ -- a condition which is clearly no longer satisfied when $\etat$ is a fixed constant independent of the gradients. We thus need a different idea to make use of \cref{lem:descentLemma}.

We leverage the fact that \cref{lem:descentLemma} holds for \emph{any} stopping time $\tau\in[2,T+1]$ as follows. Suppose there were some stopping time $\tau \in [2,T+1]$ such that:
\begin{align}\label{eq:idealStepSizeProperty}
    \EXP{\sum_{t<\tau}\tetat \gradtsq} \geq \EXP{\tetaat{\tau}}\EXP{\sum_{t < \tau} \gradtsq}.
\end{align} 
Notice that this inequality would imply that, until $\tau$, we may treat $\tetat$ and $\gradtsq$ as roughly uncorrelated.
If \eqref{eq:idealStepSizeProperty} were true, we could apply Jensen's inequality and \cref{assump:affineVariance} to obtain:
\begin{align*}
    \EXP{\sum_{t < \tau} \tetat\gradtsq}
    \geq \frac{\EXP{\sum_{t < \tau}\gradtsq}}{\sqrt{\EXP{b_0^2 + \sigmaZero^2 T + (1+\sigmaOne^2)\EXP{\sum_{t < \tau}\gradtsq}}}}.
\end{align*}
This, combined with \cref{lem:descentLemma}, yields a quadratic inequality in $\sqrt{\EXP{\sum_{t<\tau}\gradtsq}}$, which can be solved to obtain $\EXP{\sum_{t<\tau} \gradtsq} \lesssim (1+\sigmaOne^2)\log(T)^2 + \log(T)\sqrt{b_0^2 + \sigmaZero^2 T}$.
Thus, if we additionally knew that $\EXP{\tau} = \Omega(T)$, then a straightforward application of Markov's inequality would imply that, with constant probability, $\min_{t\in[T]} \gradtsq \lesssim \tO(\nicefrac{1}{\sqrt{T}})$.

It turns out that constructing a time $\tau$ (roughly) satisfying \eqref{eq:idealStepSizeProperty} is possible -- however, there is a tension in simultaneously satisfying this and $\EXP{\tau} = \Omega(T)$, as the following construction reveals.
\begin{restatable}[Nice stopping]{definition}{restateNiceStoppingDef}\label{def:niceStopping}
Fix any $\delta \in (0,1]$, and consider the following sequence of random times $\stopt$ defined recursively as follows: let $X_0(\delta) = 1$, and define, for every $t\geq 1$ (denoting $\constOneStep = 2(1+\eta\Lone)^2$):
\begin{align*}
    \stopt 
    &= \min\braces{t,\min\braces{s\geq 0 : X_s(\delta) = 0}}\\
    S_t(\delta) &= \sum_{s =1}^{\stopat{t}-1} \sgradsq{s} + \constOneStep\gradsq{s}\\
    X_{t}(\delta) &= X_{t-1}(\delta)\1{S_t(\delta) \leq \nicefrac{\EXP{S_t(\delta)}}{\delta}}.
\end{align*}
\end{restatable}
Notice that $\stopat{1}=1$, $S_1(\delta)=0$, $X_1(\delta)=1$, and $\stopat{2}=2$ deterministically. Further, one can show that $S_t(\delta), X_t(\delta)$, and $\stopat{t+1}$ are $\F_{t-1}$-measurable for every $t\geq 1$.
Intuitively, $\stopat{T+1}$ is the first time that the sum of stochastic gradient norms is significantly larger than its expectation, where the expectation is crucially over the random summation range (refer to \cref{rem:niceStopping} for a further discussion). The following result shows the utility of this recursive construction:
\begin{lemma}[Key properties of nice stopping; Simplified version of \cref{lem:niceStopping}]\label{lem:niceStoppingMain}
    For any $T\geq 1$ and $\delta \in (0,1]$, let $\stopat{T+1}$ be the stopping time from \cref{def:niceStopping}. Then, we have the following:
    \begin{enumerate}
        \item $\stopat{T+1}$ is a stopping time with respect to $(\F_{s-1})_{s\geq 1}$, i.e., $\forall s\geq 1$, $\braces{s < \stopat{T+1}} \in \F_{s-1}$.
        \item $\stopat{T+1} \in [2,T+1]$, and $\EXP{\stopat{t+1}} \geq (T+1)(1 - \nicefrac{\delta T}{2})$.
        \item For every $s < \stopat{T+1}$, denoting $a = b_0^2 + 2\eta^2\Lzero^2$ and $b=1+\sigmaOne^2+\constOneStep$,
        \begin{align*}
            \tetaat{s} \overset{\text{a.s.}}{\geq} \frac{\eta}{\sqrt{a + \frac{T\sigmaZero^2 + b\EXP{\sum_{\ell < \stopat{T+1}}\gradsq{\ell}}}{\delta}}},
        \end{align*}
    \end{enumerate}
\end{lemma}
Notice that an immediate consequence of \cref{lem:niceStoppingMain} is that:
\begin{align}\label{eq:niceStoppingConsequenceSmallNoise}
    \EXP{\sum_{t < \stopat{T+1}} \tetat\gradtsq}
    & \geq \frac{\eta\EXP{\sum_{t < \stopat{T+1}} \gradtsq}}{\sqrt{a + \frac{T\sigmaZero^2 + b\EXP{\sum_{t < \stopat{T+1}}\gradtsq}}{\delta}}}.
\end{align}
When $\nicefrac{1}{\delta} = \O(1)$, then \eqref{eq:niceStoppingConsequenceSmallNoise} essentially has the desired form \eqref{eq:idealStepSizeProperty}. Recall that we also needed $\EXP{\tau} = \Omega(T)$ to use \eqref{eq:idealStepSizeProperty}. However, \cref{lem:niceStoppingMain} gives a vacuous lower bound on $\EXP{\stopat{T+1}}$ when $\delta \geq \nicefrac{2}{T}$.

Nevertheless, choosing $\delta = \Theta(\nicefrac{1}{T})$ and solving the resulting quadratic inequality as before, \eqref{eq:niceStoppingConsequenceSmallNoise} implies $\EXP{\sum_{t < \stopat{T+1}} \gradtsq} \lesssim T\poly\log(T)$.
Given that $\EXP{\stopat{T+1}} \gtrsim T$ in this $\delta$ regime, this bound tells us something quite strong -- that the sum of gradients before this stopping time scales (roughly) linearly in expectation. This is an \emph{exponential} improvement over the worst-case growth of $\smoothBoth$-smooth functions after $T$ time steps, which is approximately $\exp(\Lone \eta T)$. Moreover, this bound implies (via Jensen's inequality) that $\EXP{\tetaat{\stopat{T+1}}} \gtrsim \nicefrac{1}{\sqrt{T\poly\log(T)}}$! Thus, at least in expectation, the step sizes that we care about for our analysis are essentially scaling as $\nicefrac{1}{\sqrt{T}}$. It turns out that this scaling is crucial to obtain \cref{thm:mainInformal} in the regime of $\sigmaOne < 1$.

\subsection{Using the descent lemma when $\sigmaOne \geq 1$}
The arguments discussed above heavily relied on being able to take $\comp{\tau} \leq 0$ and $\Sgat{\tau} = [\stopat{T+1}-1]$, which were trivially true for any stopping time when $\sigmaOne < 1$. However, when $\sigmaOne \geq 1$, then new ideas are needed, since \cref{lem:startingPoint} does not guarantee any meaningful descent inequality for $t\not\in\Sgat{\tau}$. In the context of $\Lzero$-smooth optimization, \citep{FTCMSW22} showed how to circumvent this issue -- indeed, they showed that $\comp{T} \lesssim \EXP{|\Sg^c|^2}$ and $\EXP{|\Sg^c|^2}\lesssim \log(T)$. At the core of their proofs for these arguments was the fact that, by $\Lzero$-smoothness and properties of \eqref{eq:alg}, $\abs{\gradtnorm - \gradnorm{t'}} \lesssim \eta\Lzero|t - t'|$.

General $\smoothBoth$-smooth functions clearly violate this inequality. Indeed, $\gradtnorm$ can potentially be a multiplicative factor of $\exp(\eta\Lone \abs{t-t'})$ times larger than $\gradnorm{t'}$ (for instance, when the $(0,\Lone)$-smooth objective is $\exp(\Lone x)$. Thus, even if we could guarantee deterministically that only the first $\O(\log(T))$ time-steps are ``bad'', the objective function (and also the norm of the gradient) could grow by polynomial factor in $T$ during this interval! In fact, this is exactly the intuition behind our negative result for \eqref{eq:alg} in the ``large $\sigmaOne$'' regime (see \cref{lem:divergenceAGNormMain}).

In spite of this, not every $\smoothBoth$-smooth function is an exponential function, as polynomials of constant degree also satisfy $\smoothBoth$-smoothness for constant $\Lzero, \Lone$ (see \cref{prop:polyBoundedScaling}). Motivated by this, \cref{def:polynomiallyBounded} aims to generalize the inequality $\abs{\gradtnorm - \gradnorm{t'}} \lesssim \eta\Lzero\abs{t-t'}$ to allow this difference to have larger polynomial scaling in $t-t'$. Indeed, the constraint of \cref{def:polynomiallyBounded} allows us to bound $\comp{\tau}$ as follows:

\begin{restatable}{lemma}{restateCompBound}\label{lem:compBound}
    Suppose that $\fat{\cdot}$ satisfies \cref{def:polynomiallyBounded} for some constants $k\geq 2$, $c_k\geq 1$, and $c_k' > 0$. Let $\tau\in[2,T+1]$ be any (possibly random) time. Then, recalling $\comp{\tau}$ and $\Scat{\tau}$ from \cref{lem:descentLemma}, there is an explicit construction of $\Scat{\tau}$ (the subset of ``good'' times used to compensate for $\Sgat{\tau}^c$) such that, for any $\eps,\eps',\eps'''\in (0,1)$ such that $\eps+\eps'<1$ and $\ncomp = \ceil{\nicefrac{4c_k^3(\sigmaOne-(1-\eps-\eps'))_+}{\eps'''}}$ (and taking $(x)_+:=\max\braces{0,x}$) $\comp{\tau}$ can be bounded as follows:
    \begin{align*}
        \comp{\tau}
        &\leq \eta(\sigmaOne - (1-\eps-\eps'))_+ c_k\gradzeronorm \EXP{|\Sgat{\tau}^c|}\\
        &\quad+ \eta\ncomp^{k-1}\max\braces{c_k'\eta^{k-1},\Lzero\eta}\parens{(\sigmaOne - (1-\eps-\eps'))_+ + \frac{\eps'''\ncomp}{2 c_k^3}}\EXP{|\Sgat{\tau}^c|^k}.
    \end{align*}
\end{restatable}

\cref{lem:compBound} reveals that, as long as $\EXP{|\Sgstopped^c|^k}$ can be bounded by $\poly\log(T)$ for any constant $k\geq 1$, then it is still possible to bound $\compstop$, even when the function is not $\Lzero$-smooth!

\begin{restatable}{lemma}{restateBadSetBound}\label{lem:badSetBound}
    Let $\stopat{T+1} \leq T+1$ be the stopping time with respect to $(\F_{s-1})_{s\geq 1}$ from \cref{def:niceStopping}.
    Recall the set $\Sgstopped$ from \cref{def:goodTimes}, and denote $\Sgstopped^c=[\stopat{T+1}-1]\setminus\Sgstopped$. For any $k\geq 1$, the iterates of \eqref{eq:alg} satisfy (under \cref{assump:affineVariance}):
    \begin{align*}
        \EXP{|\Sgstopped^c|^k} \leq \parens{\frac{(k+1) \sigmaOne^2 \log(f(\stopat{T+1}))}{(1- (\eps+\eps'+\eps''+\eps'''))^2}}^k,
    \end{align*}
    where $f(T) = e + \frac{e\sigmaZero^2 (T-1) + e(1+\sigmaOne^2 + \constOneStep)\EXP{\sum_{t < \stopat{T}}\gradtsq}}{b_0^2\delta}$.
\end{restatable}
Notice that \cref{lem:badSetBound} does not explicitly require that \cref{def:polynomiallyBounded} be satisfied. However, we use this constraint on the objective to easily guarantee that $f(T) = \O(\nicefrac{T^{2(k-1)}}{\delta})$, and thus that $\EXP{|\Sgstopped^c|^k}\lesssim \poly\log(T)$.
\cref{lem:badGradBound} demonstrates that the bound of $\EXP{|\Sg^c|^2}$ from \citep{FTCMSW22} can be generalized to any moment $k$, and requires bounding only $\EXP{\sum_{t<\stopat{T}}\gradtsq}$ instead of the sum over the entire time horizon, which might give tighter bounds in some scenarios.
With the bounds from \cref{lem:compBound,lem:badSetBound} in place, it is now clear that a useful descent inequality is still obtainable from \cref{lem:descentLemma} when $\sigmaOne \geq 1$, at least under the added assumption of \cref{def:polynomiallyBounded}. There is still a (small) problem in translating these results into a convergence result. Indeed, the analogous bound from \eqref{eq:niceStoppingConsequenceSmallNoise} now becomes:
\begin{align}\label{eq:niceStoppingConsequenceLargeNoise}
    \EXP{\sum_{t \in \tSgstopped} \tetat\gradtsq}
    & \geq \frac{\eta\EXP{\sum_{t \in \tSgstopped} \gradtsq}}{\sqrt{a + \frac{T\sigmaZero^2 + b\EXP{\sum_{t < \stopat{T+1}}\gradtsq}}{\delta}}}.
\end{align}
Specifically, while the numerator depends on a sum over $\tSgstopped$, the denominator depends on the sum of these good times, as well as the compensating ``good'' times $\Scstopped$ and the ``bad'' times before $\stopat{T+1}$, $\Sgstopped^c$. \citet{FTCMSW22} dealt with a similar issue by using the fact that, by $\Lzero$-smoothness and properties of \eqref{eq:alg}, $\gradtsq \lesssim T\log(\nicefrac{T}{\delta})$ with probability at least $1-\delta$, and $\gradtsq\lesssim T^2$ deterministically. Combining this with their bound $\EXP{|\Sg^c|}\lesssim\log(T)$, they proved that the sum of ``bad'' gradients satisfies $\EXP{\sum_{t\in\Sg^c} \gradtsq} \lesssim T\poly\log(T)$. However, it is not clear how to prove such a bound in our setting, since $\gradtnorm$ can scale as $t^{k-1}$, which is too large to be useful. Instead, we prove the following \emph{relative} upper bound, which is sufficient for our purposes:
\begin{align*}
    \sum_{t \in \tSgstopped^c}\gradtsq \leq \Cbadone + \cbadtwo\sum_{t\in\tSgstopped}\gradtsq,
\end{align*}
where $\EXP{\Cbadone} \lesssim\poly\log(T)$. As a consequence, \eqref{eq:niceStoppingConsequenceLargeNoise} becomes:
\begin{align*}
    \EXP{\sum_{t \in \tSgstopped} \tetat\gradtsq}
    & \geq \frac{\eta\EXP{\sum_{t \in \tSgstopped} \gradtsq}}{\sqrt{a + \frac{T\sigmaZero^2 + \EXP{\Cbadone} + b\EXP{\sum_{t \in \tSgstopped}\gradtsq}}{\delta}}}.
\end{align*}
Since the numerator and denominator both depend on the same summation, we can apply essentially the same arguments from the $\sigmaOne < 1$ case to obtain our convergence rate for $\sigmaOne \geq 1$ in \cref{thm:polyBoundedConvergenceMain}.

\section{The challenges of multiplicative noise for $\smoothBoth$-smooth optimization}
\label{sec:challengesMultNoise}
Given our positive results for \eqref{eq:alg} from the previous sections, we now turn our focus to algorithms which have been analyzed in prior works on $\smoothBoth$-smooth optimization.
Some of the first-studied algorithms for $\smoothBoth$-smooth optimization take the following forms: for parameters $\eta > 0$ and $\gamma \geq 0$:
\begin{align}\label{eq:typicalAlgs}
    \w_{t+1} \leftarrow \w_t - \frac{\eta \sgradt}{\gamma + \sgradtnorm}
    \quad\text{and}\quad
    \w_{t+1} \leftarrow \w_t - \frac{\eta \sgradt}{\max\braces{\gamma,\sgradtnorm}}.
\end{align}
These closely-related updates are referred to as Normalized SGD and Clipped SGD respectively. One motivation for considering these specific updates, at least in the noiseless setting where $\sgradt = \gradt$, comes through a comparison with the natural SGD step-size for $\Lzero$-smooth non-convex optimization. Indeed, \citet{GL13} show that a constant step-size of $\etat=\nicefrac{1}{\Lzero}$ yields a $\nicefrac{1}{T}$ rate of convergence to a first-order stationary point. Further, a simple extension of this result (see, e.g., \citep{BCN18} for a proof) is that, under $\Lzero$-smoothness and \cref{assump:affineVariance} with $\sigmaZero=0$ and $\sigmaOne \geq 0$, the step size $\etat = \nicefrac{1}{\Lzero(1+\sigmaOne^2)}$ still achieves the $\nicefrac{1}{T}$ convergence rate. Thus, by analogy, in the $\smoothBoth$-smooth setting, $\etat = \nicefrac{1}{(\Lzero + \Lone\gradtnorm)}$ (and, in the multiplicative noise regime, $\etat = \nicefrac{1}{(1+\sigmaOne^2)(\Lzero + \Lone\gradtnorm)}$) is a natural candidate step size. 

A number of works, including \citep{ZHSJ20,ZJFW20,CLOZZ22}, have proved that (variants of) these algorithms converge whenever the noise of the stochastic gradient satisfies \eqref{eq:boundedSuppIntro}. It turns out, however, that these algorithms can diverge under the noise model considered in this paper, \eqref{eq:affineVarianceIntro}. To see this, it is useful to consider a specific stochastic gradient oracle which satisfies \cref{assump:unbiasedGrad,assump:affineVariance}:
\begin{restatable}[A stochastic gradient oracle satisfying \cref{assump:affineVariance}]{proposition}{restateStochasticOracleConstruction}\label{lem:stochasticOracleConstruction}
    Fix any $\sigmaZero,\sigmaOne \geq 0$, and consider the following stochastic gradient oracle: fix any $\eps\geq 0$, and let, for every $\w\in\R^d$:
    \begin{align*}
        \multnoise(\w) = \begin{cases}
            \parens{1 + \frac{\sigmaOne^2}{1+\eps}} & \text{w.p. } \delta = \frac{1}{1 + \nicefrac{\sigmaOne^2}{(1+\eps)^2}}\\
            -\eps & \text{w.p. } 1-\delta
        \end{cases}
        \quad\text{and}\quad
        \addnoise(\w) \sim \N(0,\sigmaZero^2 I_{d\times d}).
    \end{align*}
    We can then take the output of the oracle to be $\sgradat{\w} := \addnoise(\w) + \multnoise(\w) \gradat{\w}$.
    Then, this construction satisfies \cref{assump:unbiasedGrad,assump:affineVariance} with the specified $\sigmaZero$ and $\sigmaOne$.
\end{restatable}
Consider the above oracle with $\sigmaZero = 0$ and $\sigmaOne \gg 1+\eps$. This oracle outputs stochastic gradients with the same sign as the true gradient for only roughly a $\nicefrac{1}{\sigmaOne^2}$ fraction of the times it is queried. The majority of stochastic gradients thus have the \emph{opposite} sign of the true gradient! This turns out to be quite problematic for algorithms of the form \eqref{eq:typicalAlgs}. Indeed, consider the behavior of \eqref{eq:typicalAlgs} when $\sgradtnorm \geq \gamma$. In this regime, both algorithms discard the magnitude of the stochastic gradients $\sgradt$, and use only their sign to perform updates. Since the stochastic gradients $\sgradt$ of \cref{lem:stochasticOracleConstruction} have the opposite sign of $\gradt$ for almost all time steps $t$, one can prove that algorithms of the form \eqref{eq:typicalAlgs} do not converge to a stationary point with constant probability under \cref{assump:unbiasedGrad,assump:affineVariance}, even when the objective function is a $1$-dimensional quadratic function (i.e., both smooth and strongly-convex). We give a proof of (a slightly more general version of) this fact in \cref{lem:divergenceNormSGD}.

\section*{Acknowledgements}
This research is supported in part by NSF Grants 2019844 and 2112471, the Machine Learning Lab (MLL) at UT Austin, and the Wireless Networking and Communications Group (WNCG) Industrial Affiliates Program.

\printbibliography

\newpage

\appendix
\section*{ Overview of Appendix}
\startcontents[sections]
\printcontents[sections]{l}{1}{\setcounter{tocdepth}{2}}
\newpage

\section{Auxiliary Lemmas}

\subsection{Useful facts for AdaGrad}

\begin{fact}\label{fact:logSumIneq}
Let $\{a_i\}_{i=1}^\infty$ be a sequence of non-negative integers such that $a_1 > 0$. Then, for any $T$,
\begin{align*}
    \sum_{t\in[T]} \frac{a_t}{\sum_{s=1}^t a_s} \leq 1 + \logp{\frac{\sum_{t\in[T]} a_t}{a_1}}.
\end{align*}
\end{fact}
\begin{proof}
We proceed via induction. The base case of $T=1$ holds trivially, with equality. Assuming the hypothesis holds at some time $T\geq 1$, we have that
\begin{align*}
    \sum_{t\in[T+1]} \frac{a_t}{\sum_{s=1}^t a_s} 
    &\leq 1 + \logp{\frac{\sum_{t\in[T]} a_t}{a_1}} + \frac{a_{T+1}}{\sum_{s\in[T+1]} a_s}.
\end{align*}
Now, using the fact that $\exp(x) \leq \nicefrac{1}{(1-x)}$ for any $x<1$, we have that
\begin{align*}
    \frac{a_{T+1}}{\sum_{s\in[T+1]} a_s}
    =  \logp{\exp\parens{\frac{a_{T+1}}{\sum_{s\in[T+1]} a_s}}}
    \leq \logp{\frac{1}{1-\frac{a_{T+1}}{\sum_{s\in[T+1]} a_s}}}
    = \logp{\frac{\sum_{s\in[T+1]} a_s}{\sum_{s\in[T]} a_s}}.
\end{align*}
Combining these two bounds, we conclude that
\begin{align*}
    \sum_{t\in[T+1]} \frac{a_t}{\sum_{s=1}^t a_s} 
    \leq 1 + \logp{\frac{\sum_{t\in[T]} a_t}{a_1}} + \logp{\frac{\sum_{s\in[T+1]} a_s}{\sum_{s\in[T]} a_s}}
    = 1 + \logp{\frac{\sum_{t\in[T+1]} a_t}{a_1}},
\end{align*}
so the claim holds also for $T+1$.
Thus, the claim holds for all $T$ by induction.
\end{proof}

\begin{lemma}[Log sum inequality]\label{lem:logSumIneq}
    The \eqref{eq:alg} step-sizes satisfy, for any (possibly random) times $1\leq t_0 \leq t_1$ and $s \geq 0$,
    \begin{align*}
        \sum_{t = t_0}^{t_1} \frac{\sgradtsq}{b_t^2} \leq s + \logp{\frac{b_0^2 + \sum_{t=t_0}^{t_1 - s}\sgradtsq}{b_0^2}}
    \end{align*}
\end{lemma}
\begin{proof}
We first note that, by definition of $b_t$:
\begin{align*}
    \sum_{t = t_0}^{t_1} \frac{\sgradtsq}{b_t^2}
    \leq \sum_{t = t_0}^{t_1} \frac{\sgradtsq}{b_0^2 + \sum_{\ell=t_0}^t}
    \leq s + \sum_{t = t_0}^{t_1 - s} \frac{\sgradtsq}{b_0^2 + \sum_{\ell=t_0}^t\sgradsq{\ell}}.
\end{align*}
Thus, applying \cref{fact:logSumIneq}, with $a_1 = b_0^2$ and $a_{\ell+1} = \sgradsq{t_0 + \ell - 1}$ for $\ell \geq 1$, we obtain the claimed inequality.
\end{proof}

\begin{fact}[Bounded Steps]\label{fact:boundedStep}
The iterates $\braces{\w_s}_{s=1}^\infty$ generated by \eqref{eq:alg} satisfy, for every $t\geq 1$,
\begin{align*}
    \norm{\w_{t+1} - \w_t} \leq \eta.
\end{align*}
Moreover, for any $k\geq 2$ and $t > t'$,
\begin{align*}
    \norm{\w_{t} - \w_{t'}}^{k-1} \leq \eta^{k-1}(t-t')^{k-1}.
\end{align*}
\end{fact}
\begin{proof}
By definition of \eqref{eq:alg},
\begin{align*}
    \norm{\w_{t+1} - \w_t} 
    = \etat \sgradtnorm
    = \eta \frac{\sgradtnorm}{\sqrt{b_0^2 + \sum_{s=1}^t \sgradsq{s}}}
    \leq \eta,
\end{align*}
which establishes the first inequality. To obtain the second, we apply the first, together with Jensen's inequality (noting that $\norm{\cdot}^{k-1}$ is convex), to obtain:
\begin{align*}
    \norm{\w_{t} - \w_{t'}}^{k-1}
    =(t-t')^{k-1}\norm{\frac{1}{t-t'}\sum_{s=t'}^{t-1}\w_{s+1} - \w_{s}}^{k-1}
    &\leq(t-t')^{k-2}\sum_{s=t'}^{t-1}\norm{\w_{s+1} - \w_{s}}^{k-1}\\
    &\leq \eta^{k-1}(t-t')^{k-1},
\end{align*}
as claimed.
\end{proof}

\subsection{Useful facts for $\smoothBoth$-smooth optimization}

\begin{restatable}[Local smoothness bound]{lemma}{restateLocalSmoothnessBound}\label{lem:localSmoothnessBound}
For any function $F$ satisfying \cref{assump:smooth}, the sequence of iterates $\{\w_s\}_{s=1}^{\infty}$ generated by \eqref{eq:alg} with $\eta \leq \nicefrac{1}{\Lone}$ satisfy
\begin{align*}
    \fat{\w_{t+1}} \leq \fat{\w_t} + \innerProd{\gradat{\w_t}}{\w_{t+1} - \w_t} + \frac{\Lzero + \Lone\gradnorm{t}}{2} \normSq{\w_{t+1}-\w_t}
\end{align*}
\end{restatable}

\begin{proof}
By \citep[Lemma A.3]{ZJFW20}, we know that, for any function $F(\cdot)$ satisfying \cref{assump:smooth}, and for any $\w,\w' \in \R^d$ satisfying $\norm{\w - \w'}\leq \nicefrac{1}{\Lone}$,
\begin{align*}
    \fat{\w'} \leq \fat{\w} + \innerProd{\gradat{\w}}{\w' - \w} + \frac{\Lzero + \Lone\gradnormat{\w}}{2} \normSq{\w'-\w}.
\end{align*}
Thus, by choosing $\eta \leq \nicefrac{1}{\Lone}$, the claim is an immediate consequence of \cref{fact:boundedStep}.
\end{proof}

\begin{lemma}[One-step gradient bound]\label{lem:instantaneousGradBnd}
For any $\smoothBoth$-smooth function $F(\cdot)$, assuming that $\eta \leq \nicefrac{1}{\Lone}$, the gradient $\gradtsq$ evaluated at the iterate of \eqref{eq:alg} at time $t$ satisfies:
\begin{align*}
    \gradtsq \leq 2\eta^2\Lzero^2 + 2(1+\eta\Lone)^2\gradsq{t-1}.
\end{align*}
\end{lemma}
\begin{proof}
Since $\eta \leq \nicefrac{1}{\Lone}$, $\norm{\w_{t+1}-\w_t} \leq \nicefrac{1}{\Lone}$ by \cref{fact:boundedStep}.
Thus, we may apply \cref{assump:smooth} to obtain
\begin{align*}
    \gradtnorm 
    &\leq \gradnorm{t-1} + \norm{\gradt - \grad{t-1}}\\
    &\leq \gradnorm{t-1} + (\Lzero + \Lone \gradnorm{t-1})\etaat{t-1}\sgradnorm{t-1},
\end{align*}
from which we conclude that:
\begin{align*}
    \gradtsq \leq 2\eta^2\Lzero^2 + 2(1+\eta\Lone)^2 \gradsq{t-1}.
\end{align*}
\end{proof}

\restateSmoothnessPartialEquiv*
\begin{proof}
    The proof of the first statement is from \citep[Corollary A.4]{ZJFW20}. The proof of the second statement closely follows the analogous proof for $\Lzero$-smooth functions from \citep[Lemma 1.2.2]{N03}. We give a proof of this claim for completeness.

    Consider any $\x,\s\in\R^d$ such that $0<\norm{\s} \leq \nicefrac{1}{\Lone}$, and let $\alpha\in (0,1]$. Then, by \cref{assump:smooth},  
    \begin{align*}
        \norm{\frac{\gradat{\x + \alpha\s}-\gradat{\x}}{\alpha}} \leq (\Lzero + \Lone\gradnormat{\x})\norm{\s}.
    \end{align*}
    Therefore, we have the following:
    \begin{align*}
        &\lim_{\alpha\to 0} \norm{\frac{\gradat{\x + \alpha\s}-\gradat{\x}}{\alpha}}\\
        &= \norm{\lim_{\alpha\to 0} \frac{\gradat{\x + \alpha\s}-\gradat{\x}}{\alpha}} && \text{by continuity of $\norm{\cdot}$ and twice differentiability of $\fat{\cdot}$}\\
        &= \norm{\hessat{\x} \cdot \s} && \text{by definition of directional derivative}
    \end{align*}
    Hence, by the limit inequality theorem, we have that, for any $0<\norm{\s}\leq\nicefrac{1}{\Lone}$,
    \begin{align*}
        \frac{\norm{\hessat{\x} \cdot \s}}{\norm{\s}}
        \leq \Lzero + \Lone \gradnormat{\x}.
    \end{align*}
    In particular, by taking the supremum over all such $\s$, we conclude that
    \begin{align*}
        \norm{\hessat{\x}} = \norm{\hessat{\x}^\top} \leq \Lzero + \Lone \gradnormat{\x},
    \end{align*}
    as claimed, where the first equality follows by observing that $\hessat{\x}\hessat{\x}^\top$ and $\hessat{\x}^\top\hessat{\x}$ have the same non-zero eigenvalues (since all entries of $\hessat{\x}$ are real, and by appealing to the singular value decomposition), which implies that $\hessat{x}$ and $\hessat{x}^\top$ have the same spectral norm.
\end{proof}

\subsection{A note on enforcing $\sigmaOne < 1$}

\begin{restatable}[Reducing $\sigmaOne$ through mini-batching]{fact}{restateMinibatchingFact}\label{fact:minibatching}
    Suppose that the stochastic gradient oracle satisfies \cref{assump:unbiasedGrad,assump:affineVariance} for some $\sigmaZero \geq 0$ and $\sigmaOne \geq 1$. Then, assuming this oracle returns independent stochastic gradients each time $\sgradat{\w}$ is sampled, one can construct, for any $\eps \in (0,1)$, a new stochastic gradient oracle from this one through mini-batching which satisfies \cref{assump:unbiasedGrad,assump:affineVariance} with $\widetilde{\sigmaZero}\leq \sigmaZero$ and $\widetilde{\sigmaOne}=1-\eps$, and where each call to the new gradient requires only $B = \ceil{\nicefrac{\sigmaOne^2}{(1-\eps)^2}}$ calls to the old one.
\end{restatable}

\begin{proof}
    Fix any $\eps\in (0,1)$ and $\w\in\R^d$. Let $B=\ceil{\nicefrac{\sigmaOne^2}{(1-\eps)^2}}$, and let $\braces{\sgradatj{\w}}_{j\in [B]}$ be a set of $B$ independent stochastic gradients corresponding to $\gradat{\w}$ from an oracle satsifying \cref{assump:unbiasedGrad,assump:affineVariance} with $\sigmaZero \geq 0$ and $\sigmaOne \geq 1$. Then, we take the response of the new oracle as:
    \begin{align*}
        \tsgradat{\w} := \frac{1}{B}\sum_{j\in [B]} \sgradatj{\w}.
    \end{align*}
    Now, since $\EXP{\sgradatj{\w}}=\gradat{\w}$ and applying linearity of expectation, $\EXP{\tsgradat{\w}}=\gradat{\w}$. Further, notice that:
    \begin{align*}
        \EXP{\normSq{\tsgradat{\w}-\gradat{\w}}} 
        &=\EXP{\normSq{\frac{1}{B}\sum_{j\in [B]} \sgradatj{\w} - \gradat{\w}}}\\
        &=\frac{1}{B^2}\sum_{j\in [B]}\EXP{\normSq{ \sgradatj{\w} - \gradat{\w}}}\\
        &\quad+\frac{2}{B^2}\sum_{B\geq j > j'\geq 1}\EXP{\innerProd{\sgradatj{\w}-\gradat{\w}}{ \sgradatjp{\w} - \gradat{\w}}}\\
        &\leq \frac{1}{B^2}\sum_{j\in [B]}\sigmaZero^2 + \sigmaOne^2\gradsqat{\w}\\
        &\leq \frac{(1-\eps)^2\sigmaZero^2}{\sigmaOne^2} + (1-\eps)^2\gradsqat{\w},
    \end{align*}
    where the first inequality follows by \cref{assump:unbiasedGrad} and since $\sgradatj{\w}$ and $\sgradatjp{\w}$ are independent. The second inequality follows by \cref{assump:affineVariance} and our choice of $B \geq \nicefrac{\sigmaOne^2}{(1-\eps)^2}$. Thus, \cref{assump:affineVariance} is satisfied with $\widetilde{\sigmaZero}^2 = \nicefrac{(1-\eps)^2\sigmaZero^2}{\sigmaOne^2} \leq \sigmaZero^2$ and $\widetilde{\sigmaOne}^2 = (1-\eps)^2$.
\end{proof}

\section{Proofs for general $\smoothBoth$-smooth functions}

\subsection{Deriving the descent inequality}

\restateStartingPoint*
\begin{proof}
An immediate consequence of \cref{lem:localSmoothnessBound} and \citep[Lemma 5]{FTCMSW22} is that, as long as $\eta \leq \nicefrac{1}{\Lone}$,
\begin{align}
    \Et{\ftplus - \ft} 
    &\leq - \tetat\parens{1-\eps - \sigmaOne\bias}\gradtsq
    + \cCommon \Et{\frac{\sgradtsq}{b_t^2}}\nonumber\\
    &\quad+\frac{\Lone \gradtnorm}{2} \Et{\etat^2 \sgradtsq},\label{eq:intermediateStartingPoint}
\end{align}
where 
\begin{align*}
    \cCommon = \frac{\eta \sigmaZero}{2\eps} + \frac{\eta^2 \Lzero}{2}
    \quad\text{and}\quad
    \bias = \constBias\sqrt{ \Et{\frac{\sgradtsq\parens{\tgradtnorm+\sgradtnorm}^2}{b_t^2(\tbt+b_t)^2}}}.
\end{align*}
We provide a proof of this inequality in \cref{lem:intermediateStartingPoint}\footnote{A careful reader may notice that the inequality in \cref{lem:intermediateStartingPoint} is actually slightly smaller than the one from \citep[Lemma 5]{FTCMSW22}, since the dependence on constants is strictly better.}.

Now, let's focus on bounding the final term above. We start by rewriting it as follows:
Let us take $\varianceEvent = \{\gradtnorm > \sigmaZero\}$. Then,
we can decompose the final term (trivially) as
\begin{align*}
    \frac{\Lone\gradtnorm}{2} \Et{\etat^2\sgradtsq}
    = \frac{\Lone\gradtnorm}{2} \Et{\etat^2\sgradtsq}(\1{\varianceEvent^c} + \1{\varianceEvent}).
\end{align*}
Now, whenever $\varianceEvent$ is false, then this expression is easy to bound, since
\begin{align*}
    \frac{\Lone\gradtnorm}{2} \Et{\etat^2\sgradtsq}\1{\varianceEvent^c}
    &\leq 
    \frac{\Lone\sigmaZero}{2} \Et{\etat^2\sgradtsq}.
\end{align*}
Notice that this term can be absorbed into the second term in \eqref{eq:intermediateStartingPoint}.
The case when $\varianceEvent$ is true requires slightly more care.
However, we can deal with this case by adding and subtracting $\tetat^2$, and using the bound \eqref{eq:sgradBound}:
\begin{align*}
    \frac{\Lone\gradtnorm}{2} \Et{\etat^2\sgradtsq}\1{\varianceEvent}
    &= \frac{\Lone}{2}  \tetat^2 \gradtnorm\Et{\sgradtsq}\1{\varianceEvent}\\
    &\quad+ \frac{\Lone}{2}  \gradtnorm\Et{(\etat^2 - \tetat^2)\sgradtsq}\1{\varianceEvent}\\
    &\leq \frac{\Lone(2+\sigmaOne^2)}{2} \tetat^2 \gradtnorm^{3}\\
    &\quad+ \frac{\Lone}{2}  \gradtnorm\Et{(\etat^2 - \tetat^2)\sgradtsq}\1{\varianceEvent}\\
    &\leq \frac{\eta \Lone(2+\sigmaOne^2)}{2}   \tetat \gradtnorm^{2}\\
    &\quad+ \frac{\Lone}{2}  \gradtnorm\Et{(\etat^2 - \tetat^2)\sgradtsq}\1{\varianceEvent}.
\end{align*}
Notice that the first term above can be absorbed into the first term in \eqref{eq:intermediateStartingPoint}, assuming $\eta$ is sufficiently small. For the remaining term, we begin by noticing that
\begin{align*}
    \frac{\etat^2 - \tetat^2}{\eta^2}\1{\varianceEvent} 
    &= \frac{\1{\varianceEvent}}{b_{t-1}^2 + \sgradtsq} - \frac{\1{\varianceEvent}}{b_{t-1}^2 + \tgradtsq}\\
    &= \frac{(\tgradtsq - \sgradtsq)\1{\varianceEvent}}{(b_{t-1}^2 + \sgradtsq)(b_{t-1}^2 + \tgradtsq)}\\
    &\leq \frac{\tgradtsq\1{\varianceEvent}}{(b_{t-1}^2 + \sgradtsq)(b_{t-1}^2 + \tgradtsq)}\\
    &\leq \frac{2\gradtsq}{(b_{t-1}^2 + \sgradtsq)(b_{t-1}^2 + \tgradtsq)}\\
    &\leq \frac{2\gradtnorm}{(b_{t-1}^2 + \sgradtsq)\sqrt{b_{t-1}^2 + \tgradtsq}},
\end{align*}
which implies that
\begin{align*}
    \frac{\Lone(\etat^2 - \tetat^2)}{2}\gradtnorm \sgradtsq\1{\varianceEvent}
    &\leq \eta \Lone \tetat \gradtsq \frac{\sgradtsq}{b_t^2}\\
    &\leq \eta \Lone \tetat \gradtsq.
\end{align*}
Therefore, collecting these results and choosing $\eta < \nicefrac{2\eps'}{\Lone(4+\sigmaOne^2)}$,
we have that
\begin{align*}
    & \Et{\ftplus - \ft}
    \leq -\tetat\parens{1 - \eps - \sigmaOne\bias - \frac{\eta\Lone(4+\sigmaOne^2)}{2}}\gradtsq + \tcCommon \Et{\frac{\sgradtsq}{b_t^2}}\\
    &\leq -\tetat\parens{1-\eps-\eps' - \sigmaOne\bias}\gradtsq + \tcCommon \Et{\frac{\sgradtsq}{b_t^2}},
\end{align*}
where $\tcCommon = \cCommon + \eta^2\Lone \sigmaZero$.
\end{proof}

\begin{restatable}{lemma}{restateRefinedStartingPoint}\label{lem:refinedStartingPoint}
Fix any $\eps''\in(0,1)$, and let $2 \leq \tau \leq T+1$ be any stopping time with respect to $(\F_{s-1})_{s\geq 1}$. Then, we have that
\begin{align*}
    \tcCommon\Et{\sum_{t=1}^{\tau-1}\frac{\sgradtsq}{b_t^2}}
    &\leq \eps'' \EXP{\sum_{t=1}^{\tau-1}\tetat\gradtsq}
    + 2\tcCommon \logp{\frac{(2 + \sigmaOne^2)\tcCommon \EXP{\tau-1}}{\eta \eps'' b_0}}
    + \frac{2\eta\eps'' \sigmaZero}{2+\sigmaOne^2}.
\end{align*}
\end{restatable}
\begin{proof}
Let us define, for a parameter $\lambda$ to be determined,
\begin{align*}
    \stopetaat{\lambda} = \min\braces{T+1,\min\braces{t \geq 1 : \etat \leq \lambda}}.
\end{align*}
By construction, $\stopetaat{\lambda}$ is the first time when the step size $\etat$ is smaller than some threshold $\lambda$ (or $T+1$ in the case that $\etat$ remains larger than $\lambda$ for every $t\in[T]$). Observe that $\etat > \lambda$ is equivalent to $b_t < \nicefrac{\eta}{\lambda}$. Thus, we divide our analysis into two phases: times before $\stopetaat{\lambda}$, and those after. For the earlier times, since we have $b_{\stopetaat{\lambda}-1} < \nicefrac{\eta}{\lambda}$, we can bound these using \cref{lem:logSumIneq}. We use the fact that $\etat \leq \lambda$ together with \eqref{eq:sgradBound} to handle the remaining terms.

More specifically, for any $t$, we can decompose
\begin{align*}
    \tcCommon \frac{\sgradtsq}{b_t^2}
    =\tcCommon \frac{\sgradtsq}{b_t^2}(\1{t < \stopetaat{\lambda}} + \1{t\geq \stopetaat{\lambda}}).
\end{align*}
Now, note that, by definition of $\stopetaat{\lambda}$, $\etaat{\stopetaat{\lambda} - 1} > \lambda$, i.e., $b_{\stopetaat{\lambda}-1} < \nicefrac{\eta}{\lambda}$. Hence, by \cref{lem:logSumIneq},
\begin{align*}
    \sum_{t=1}^{\tau-1} \frac{\sgradtsq}{b_t^2} \1{t < \stopetaat{\lambda}}
    \leq \sum_{t < \stopetaat{\lambda}} \frac{\sgradtsq}{b_t^2}
    \leq 2\log(\nicefrac{b_{\stopetaat{\lambda}-1}}{b_0})
    \leq 2\logp{\frac{\eta}{\lambda b_0}}.
\end{align*}
In the other case, for any fixed $t\in[T]$, we have that
\begin{align*}
    \etat^2 \sgradtsq\1{t \geq \stopetaat{\lambda}}
    &\leq \lambda \etat \sgradtsq\\
    &= \lambda (\etat - \tetat) \sgradtsq + \lambda \tetat \sgradtsq\\
    &\leq \eta\lambda \frac{\tgradtsq}{b_t \sqrt{b_{t-1}^2 + \tgradtsq}(b_t + \sqrt{b_{t-1}^2 + \tgradtsq})} \sgradtsq
    + \lambda \tetat \sgradtsq\\
    &\leq \lambda\tetat \gradtsq + \lambda\eta\sigmaZero + \lambda \tetat\sgradtsq,
\end{align*}
where, in the first inequality, we used the fact that $\etat \leq \lambda$ for every $t \geq \stopetaat{\lambda}$, and in the second, we used the fact that
\begin{align*}
    \frac{\etat - \tetat}{\eta} 
    = \frac{1}{\sqrt{b_{t-1}^2 + \sgradtsq}} - \frac{1}{\sqrt{b_{t-1}^2 + \tgradtsq}}
    &= \frac{\tgradtsq - \sgradtsq}{b_t\sqrt{b_{t-1}^2 + \tgradtsq}(b_t+\sqrt{b_{t-1}^2 + \tgradtsq})} \\
    &\leq \frac{\tgradtsq}{b_t^2\sqrt{b_{t-1}^2 + \tgradtsq}},
\end{align*}
and in the third, we used the fact that $\tgradtsq = \sigmaZero^2 + \gradtsq$.
Then, noting that
\begin{align*}
    \lambda \tetat \Et{\sgradtsq} 
    \leq \lambda(1+\sigmaOne^2)\tetat\gradtsq + \lambda \tetat \sigmaZero^2
    \leq \lambda(1+\sigmaOne^2)\tetat\gradtsq + \lambda \eta \sigmaZero,
\end{align*}
we have that
\begin{align*}
    \Et{\etat^2\sgradtsq \1{t\geq \stopetaat{\lambda}}}
    \leq 2\lambda\eta\sigmaZero + \lambda(2+\sigmaOne^2)\tetat \gradtsq.
\end{align*}
Combining these bounds, we obtain:
\begin{align*}
    \tcCommon\Et{\frac{\sgradtsq}{b_t^2}}
    \leq \tetat\frac{\lambda\tcCommon(2+\sigmaOne^2)}{\eta^2}\gradtsq
    + \tcCommon\Et{\frac{\sgradtsq}{b_t^2}\1{t < \stopetaat{\lambda}}}
    + \frac{2\tcCommon\lambda\sigmaZero}{\eta} 
\end{align*}
Thus, if we choose $\lambda = \frac{\eps''\eta^2}{(2+\sigmaOne^2)\tcCommon (\tau-1)},$
then we obtain
\begin{align*}
    \tcCommon\Et{\frac{\sgradtsq}{b_t^2}}
    &\leq \eps''\tetat\gradtsq 
    + \tcCommon\Et{\frac{\sgradtsq}{b_t^2}\1{t<\stopetaat{\lambda}}}
    + \frac{2\eta\eps'' \sigmaZero}{(2+\sigmaOne^2)(\tau-1)}.
\end{align*}
Now, summing over $t\in [\tau-1]$, and using the fact that $\braces{t < \tau}\in\F_{t-1}$ by assumption on $\tau$, we have:
\begin{align*}
    \tcCommon\EXP{\sum_{t=1}^{\tau - 1} \frac{\sgradtsq}{b_t^2}}
    &=\tcCommon\sum_{t=1}^{T} \EXP{\Et{\frac{\sgradtsq}{b_t^2}\1{t < \tau}}}\\
    &=\sum_{t=1}^{T} \EXP{\tcCommon\Et{\frac{\sgradtsq}{b_t^2}}\1{t < \tau}}\\
    &\leq \EXP{\sum_{t=1}^{\tau - 1} \eps''\tetat\gradtsq} 
    + \tcCommon\EXP{\sum_{t=1}^{\tau-1} \frac{\sgradtsq}{b_t^2}\1{t < \stopetaat{\lambda}} + \frac{2\eta\eps''\sigmaZero}{(2+\sigmaOne^2)(\tau-1)}}.
\end{align*}
Focusing on the last term in the above inequality, and recalling that (deterministically) $\tau > 1$ by assumption, we may apply the above bounds together with Jensen's inequality to obtain:
\begin{align*}
    \tcCommon\EXP{\sum_{t=1}^{\tau-1} \frac{\sgradtsq}{b_t^2}\1{t < \stopetaat{\lambda}} + \frac{2\eta\eps''\sigmaZero}{(2+\sigmaOne^2)(\tau-1)}}
    &\leq 2\tcCommon\EXP{\logp{\frac{(2+\sigmaOne^2)\tcCommon(\tau-1)}{\eta\eps'' b_0}}} + \frac{2\eta\eps''\sigmaZero}{2+\sigmaOne^2}\\
    &\leq 2\tcCommon\logp{\frac{(2+\sigmaOne^2)\tcCommon \EXP{\tau-1}}{\eta \eps'' b_0}}
    + \frac{2\eta\eps''\sigmaZero}{2+\sigmaOne^2}.
\end{align*}
Combining these bounds yields the claimed inequality.
\end{proof}

In the following, we restate \cref{def:goodTimesMain} with an equivalent characterization that is sometimes more convenient for our analysis.
\begin{restatable}[Good times (extended version of \cref{def:goodTimesMain})]{definition}{restateGoodTimes}\label{def:goodTimes}
A time $t\in [T]$ is ``good'' if, for fixed parameters $\eps,\eps',\eps'',\eps'''\in (0,1)$, satisfying $\eps+\eps'+\eps''+\eps''' < 1$
\begin{align*}
    1-\eps-\eps'-\eps'' - \sigmaOne\bias \geq \eps'''
    \quad\text{or, equivalently,}\quad
    \Et{\frac{\sgradtsq}{b_t^2}} \leq \frac{(1-(\eps+\eps'+\eps''+\eps'''))^2}{\sigmaOne^2}.
\end{align*}
We take, for any stopping time $\tau$ with respect to $(\F_{s-1})_{s\geq 1}$, the set $\Sgat{\tau} = \braces{1 \leq t < \tau : \text{$t$ is ``good''}}$ to be the ``good'' times before $\tau$, and $\Sgat{\tau}^c = [\tau-1]\setminus\Sgat{\tau}$ to be the remaining ``bad'' times before $\tau$.
\end{restatable}

\begin{lemma}[Bounds for ``good'' and ``bad'' times]\label{lem:boundGoodBad}
Consider the same setting as \cref{lem:startingPoint,lem:refinedStartingPoint}. Let $\tau \in [T+1]$ be any stopping time with respect to $(\F_{s-1})_{s\geq 1}$. Then, for any $t\in\Sgat{\tau}$, 
\begin{align*}
    (\eps''+\eps''')\tetat\gradtsq \leq \Et{\ft - \ftplus} 
    + \tcCommon \Et{\frac{\sgradtsq}{b_t^2}}.
\end{align*}
For every other time $t\not\in\Sgat{\tau}$, we have that
\begin{align*}
    \Et{\ftplus - \ft} 
    &\leq \constBias\tetat\parens{\sigmaOne - \constBiasInv\parens{1-(\eps+\eps')}}\gradtsq + \tcCommon \Et{\frac{\sgradtsq}{b_t^2}}.
\end{align*}
In particular, whenever $\sigmaOne \leq  \constBiasInv\parens{1-\eps-\eps'}$, then we have the following bound for each $t\not\in\Sgat{\tau}$.
\begin{align*}
    \Et{\ftplus - \ft} \leq \tcCommon \Et{\frac{\sgradtsq}{b_t^2}} \leq \tcCommon.
\end{align*}
\end{lemma}
\begin{proof}
We note that, by construction, $\braces{t < \tau}, \braces{t\in\Sgat{\tau}} \in \F_{t-1}$, since $\tau$ is a stopping time with respect to $(\F_{s-1})_{s\geq 1}$ and  $\Et{\nicefrac{\sgradtsq}{b_t^2}}$ is $\F_{t-1}$-measurable. Since the inequalities we wish to prove are in expectation conditioned on $\F_{t-1}$, the condition that a time $t$ is ``good'' or ``bad'' is (effectively) deterministic.

The proof of the first inequality is an immediate consequence of \cref{lem:startingPoint,lem:refinedStartingPoint,def:goodTimes}. The second follows immediately from \cref{lem:startingPoint}, noting that $\bias = \constBias\sqrt{\Et{\nicefrac{\sgradtsq}{b_t^2}}} \leq \constBiasVal$. The final follows from the second, noting in this case that $\sigmaOne - \constBiasInv\parens{1-(\eps+\eps')}\leq 0$.
\end{proof}

\restateDescentLemma*
\begin{proof}
The proof follows straightforwardly by combining the inequalities from \cref{lem:boundGoodBad}, together with noting that, since $\tau$ is a stopping time with respect to $(\F_{s-1})_{s\geq 1}$, $\braces{s < \tau}\in\F_{s-1}$. Indeed, since $[\tau-1] = \Sgat{\tau} \cup \Sgat{\tau}^c$, we may apply the tower rule and linearity of expectation to conclude that
\begin{align*}
    \EXP{\fat{\w_{\tau}}-\fzero}
    &= \EXP{\sum_{t < \tau} \ftplus - \ft}\\
    &= \sum_{t\in [T]} \EXP{\Et{(\ftplus - \ft)\1{t < \tau}}}\\
    &= \sum_{t\in [T]} \EXP{\Et{\ftplus - \ft}\1{t < \tau}}\\
    &= \sum_{t\in [T]} \EXP{\Et{\ftplus - \ft}\1{t\in\Sgat{\tau}}}\\
    &\quad+ \sum_{t\in[T]}\EXP{\Et{\ftplus - \ft}\1{t\in\Sgat{\tau}^c}}
\end{align*}
Now, we may use the first and second inequalities in \cref{lem:boundGoodBad} to bound the sum over ``good'' and ``bad'' times, respectively, and, collecting terms, we obtain
\begin{align*}
    \EXP{\fat{\w_{\tau}}-\fzero}
    &\leq -(\eps''+\eps''')\EXP{\sum_{t\in\Sgat{\tau}} \tetat\gradtsq}\\
    &\quad+ \EXP{\sum_{t\in\Sgat{\tau}^c} \constBias\parens{\sigmaOne - \constBiasInv\parens{1-(\eps+\eps')}}\tetat\gradtsq}\\
    & \quad + \tcCommon\EXP{\sum_{t\in[\tau-1]}\frac{\sgradtsq}{b_t^2}}
\end{align*}
Thus, applying \cref{lem:refinedStartingPoint} to bound the final term above, and using the fact that $\EXP{\fzero - \fat{\w_\tau}} \leq \fzero - \fstar$ by \cref{assump:lowerBounded}, we obtain:
\begin{align*}
    \eps'''\EXP{\sum_{t\in\Sgat{\tau}}\tetat\gradtsq} 
    &\leq \fzero - \fstar
    + 2\tcCommon\logp{\frac{(2+\sigmaOne^2)\tcCommon \EXP{\tau-1}}{\eta\eps'' b_0}}
    + \frac{2\eta\eps''\sigmaZero}{(2+\sigmaOne^2)}\\
    &\quad+\constBias\EXP{\sum_{t\in\Sgat{\tau}^c}\parens{\sigmaOne-\constBiasInv\parens{1-\eps-\eps'}}\tetat\gradtsq)}.
\end{align*}
Thus, for any $\tSgat{\tau} \subset \Sgat{\tau}$, we can subtract $\eps'''\EXP{\sum_{t\in\Sgat{\tau}\setminus\tSgat{\tau}}\tetat\gradtsq}$ from both sides of the above inequality to obtain the first claimed inequality.

The second follows immediately by noting in this case that $\sigmaOne - \constBiasInv\parens{1-(\eps+\eps')}\leq 0$.
The third follows immediately from the second, recalling that, whenever $\sigmaOne\leq \constBiasInv\parens{1-(\eps + \eps' + \eps'' + \eps''')}$, then $\Sgat{\tau}=[\tau-1]$ by \cref{def:goodTimes}.
\end{proof}

\subsection{Constructing the ``nice'' stopping time}
Let us recall the definition of $\stopat{T+1}$, the ``nice'' stopping time:
\restateNiceStoppingDef*

Here, we show that these random variables are well-defined, and enumerate the crucial properties that they satisfy.

\begin{restatable}[Nice stopping; Full version of \cref{lem:niceStoppingMain}]{lemma}{restateNiceStopping}\label{lem:niceStopping}
For any $\delta\in (0,1]$ and $t\geq 1$, let $\stopat{t}$, $S_t(\delta)$, and $X_t(\delta)$ be recursively-defined random variables from \cref{def:niceStopping}.
Then, we have that, for all $t\geq 1$,
\begin{enumerate}
    \item\label{item:niceStopping1} $\stopt$ is $\F_{t-2}$-measurable, and $S_t(\delta), X_t(\delta)$ are each $\F_{t-1}$-measurable (where we take $\F_0 = \F_{-1}$ to be the trivial $\sigma$-algebra).
    \item\label{item:niceStopping2} $\stopt$ is a stopping time with respect to $(\F_{s-1})_{s\geq 1}$, i.e., for all $s\geq 0$ $\{s < \stopt\} \in \F_{s-1}$.
    \item\label{item:niceStopping3} For all $t\geq 1$, $\stopat{t+1}\geq\stopt$, $S_{t+1}(\delta) \geq S_{t}(\delta)$, and $X_{t+1}(\delta) \leq X_t(\delta)$.
    \item\label{item:niceStopping4} $\EXP{S_t(\delta)} \leq \EXP{\sum_{s < \stopt} \sigmaZero^2 + (1 + \sigmaOne^2 + \constOneStep)\gradsq{s}}$
    \item\label{item:niceStopping5} $S_{\stopt-1}(\delta) \overset{\text{a.s.}}{\leq} \nicefrac{\EXP{S_{t-1}(\delta)}}{\delta}$.
    \item\label{item:niceStopping6} $\stopat{\stopt-1} \overset{\text{a.s.}}{=} \stopt-1$
    \item\label{item:niceStopping7} For every $s < \stopt$, the following inequalities hold deterministically:
    \begin{align*}
    \tetaat{s} &\geq \frac{\eta}{\sqrt{b_{0}^2 + 2\eta^2\Lzero^2 + \sigmaZero^2 + S_{\stopt-1}(\delta)}} \\&\geq  \frac{\eta}{\sqrt{b_{0}^2 + 2\eta^2\Lzero^2 + \frac{(t-1)\sigmaZero^2 + (1+\sigmaOne^2 + \constOneStep)\EXP{\sum_{\ell < \stopat{t-1}}\gradsq{\ell}}}{\delta}}}
    \end{align*}
    \item\label{item:niceStopping8} $t \geq \EXP{\stopt} \geq t\parens{1 - \nicefrac{\delta (t-1)}{2}}$
\end{enumerate}
\end{restatable}
Before proving this result, let us briefly discuss an alternative construction to \cref{def:niceStopping} which is (perhaps) more natural and easier to define, but does not satisfy a property we rely on to prove \cref{lem:linearSumBound}:
\begin{remark}\label{rem:niceStopping}
    One might attempt to define the stopping times $\stopat{t}$ from \cref{def:niceStopping} in the following simpler manner. First, denote $\widetilde{S}_t = \sum_{s=1}^{t-1} \sgradsq{s} + \constOneStep\gradsq{s}$. Then, let $\widetilde{\tau}_t(\delta) := \min\braces{t, \min\braces{s \geq 0 : \widetilde{S}_t > \nicefrac{\EXP{\widetilde{S}_t}}{\delta}}}$. On a first impression, this stopping time might seem to capture the same properties as \cref{def:niceStopping}. Unfortunately, this is not the case. To see this, let us examine the quantity $\widetilde{S}_{\widetilde{\tau}_t(\delta)-1}$. This stopping time guarantees the following:
    \begin{align}\label{eq:remarkNiceStoppingAlt}
        \widetilde{S}_{\widetilde{\tau}_t(\delta)-1}
        = \sum_{\ell=1}^{t-1}\widetilde{S}_{\ell}(\delta)\1{\widetilde{\tau}_t(\delta)-1=\ell}
        &\leq \sum_{\ell=1}^{t-1}\frac{\EXP{\widetilde{S}_{\ell}}}{\delta}\1{\widetilde{\tau}_t(\delta)-1=\ell}\nonumber\\
        &= \sum_{\ell=1}^{t-1}\frac{\sum_{s=1}^{\ell}\EXP{\sgradsq{s} + \constOneStep\gradsq{s}}}{\delta}\1{\widetilde{\tau}_t(\delta)-1=\ell}\nonumber\\
        &= \frac{\sum_{s=1}^{\widetilde{\tau}_t(\delta)-1}\EXP{\sgradsq{s} + \constOneStep\gradsq{s}}}{\delta}.
    \end{align}
    Thus, we are only guaranteed deterministically that $\widetilde{S}_{\widetilde{\tau}_t(\delta)-1} \leq \nicefrac{\sum_{\ell=1}^{t-1}\EXP{\sgradsq{\ell} + \constOneStep\gradsq{\ell}}}{\delta}$ (indeed, this is the only inequality we know on any sample path where $\widetilde{\tau}_t(\delta)=t$).
    By contrast, by \cref{item:niceStopping5} of \cref{lem:niceStopping}, we know that, deterministically:
    \begin{align}\label{eq:remarkNiceStopping}
        S_{\tau_t(\delta)-1}(\delta) 
        \leq \frac{\EXP{S_{t-1}(\delta)}}{\delta}
        = \frac{\EXP{\sum_{s=1}^{\stopat{t-1}-1}\sgradsq{s}+\constOneStep\gradsq{s}}}{\delta}
    \end{align}
    Notice that \eqref{eq:remarkNiceStopping} is true no matter the realization of $\stopat{t}$. Indeed, for any realization of $\stopat{t}$, the bound on the right-hand side still involves a random index \emph{inside} the expectation. This is not the case with \eqref{eq:remarkNiceStoppingAlt} (there, the random index is \emph{outside} of the expectation). This difference is crucial, and this special property of $S_{\stopat{t}-1}$ is actually what makes the proof of \cref{lem:linearSumBound} possible.
\end{remark}
\begin{custproof}[of \cref{lem:niceStopping}]
We prove the first claim via induction. The base case of $t=1$ holds trivially, since $X_0(\delta)=1$ deterministically by definition, which implies that $\stopat{1} =1$ and $S_1(\delta) = 0$, and thus $X_1(\delta) = 1$ (so are all measurable in the trivial $\sigma$-algebra). Assuming the claim holds for times $1,\ldots,t$, then we have that $\stopat{t+1}$ is $\F_{t-1}$-measurable, since it depends only on $X_0(\delta),\ldots,X_t(\delta)$, each of which is $\F_{t-1}$-measurable by the induction hypothesis.
Thus, since $S_{t+1}(\delta)$ depends only on $\stopat{t+1}$ and $\braces{\sgradsq{s},\gradsq{s}}_{s=1}^{\stopat{t+1}-1} \subseteq \braces{\sgradsq{s},\gradsq{s}}_{s=1}^{t}$, $S_{t+1}(\delta)$ is $\F_t$-measurable. Further, since $X_t(\delta)$ is $\F_{t-1}\subset\F_t$-measurable and $S_{t+1}(\delta)$ is $\F_t$-measurable, and by definition, $X_{t+1}(\delta) = X_t(\delta) \1{S_{t+1}(\delta) \leq \nicefrac{\EXP{S_{t+1}(\delta)}}{\delta}}$, we conclude that $X_{t+1}(\delta)$ is $\F_t$-measurable. 
Thus, the claim holds by induction.

For the second claim, it suffices to consider $0\leq s\leq t-2$ (since we just established that $\stopt$ is $\F_{t-2}$-measurable, and $\F_{t-2} \subset\F_{t'-2}$ for any $t'\geq t$). Now, for any such $s$, since $s<t$, we have that
\begin{align*}
    \braces{s \geq \stopt} = \cup_{\ell = 0}^s\braces{X_\ell(\delta) = 0} \in \F_{s-1},
\end{align*}
since $X_\ell(\delta)$ is $\F_{s-1}$-measurable for every $\ell \leq s$. Thus, since $\F_{s-1}$ is a $\sigma$-algebra, and hence closed under complements, $\braces{s < \stopt} = \braces{s \geq \stopt}^c \in \F_{s-1}$.

For the third claim, the inequality $\stopat{t+1}\geq \stopt$ follows immediately from the definition, since if $\stopt = s$ for some $s\in[t]$, then either $X_s = 0$, in which case $\stopat{t+1} = s = \stopt$, or $s=t$ and $X_t = 1$, in which case $\stopt = t$ and $\stopat{t+1}=t+1 > \stopt$. The inequality $S_{t+1}(\delta)\geq S_t(\delta)$ follows since $\stopat{t+1}\geq \stopt$ and $S_{t}(\delta)$ is a sum of non-negative terms over the interval $[1,\stopt)$, each of which is contained in the sum $S_{t+1}(\delta)$. The inequality $X_{t+1}(\delta) \leq X_t(\delta)$ follows immediately from the definition, since $\1{S_{t+1}(\delta) \leq \nicefrac{\EXP{S_{t+1}(\delta)}}{\delta}} \in \braces{0,1}$.

For the fourth claim, we have that, by definition of $S_t(\delta)$ and the tower rule of expectation,
\begin{align*}
    \EXP{S_t(\delta)} 
    &= \EXP{\sum_{s=1}^{t-1} (\sgradsq{s} + \constOneStep\gradsq{s})\1{s < \stopt}}\\
    &= \sum_{s=0}^{t-1} \EXP{\Econd{s-1}{(\sgradsq{s} + \constOneStep\gradsq{s})\1{s < \stopt}}}.
\end{align*}
Now, since $\braces{s < \stopt}\in\F_{s-1}$, and applying \eqref{eq:sgradBound},
\begin{align*}
    &\EXP{\Econd{s-1}{(\sgradsq{s} + \constOneStep\gradsq{s})\1{s < \stopt}}}\\
    &= \EXP{\Econd{s-1}{\sgradsq{s} + \constOneStep\gradsq{s}}\1{s < \stopt}}\\
    &\leq \EXP{(\sigmaZero^2 + (1+\sigmaOne^2+ \constOneStep)\gradsq{s})\1{s < \stopt}}.
\end{align*}
Summing the above expression over $s\in[t-1]$, we conclude that
\begin{align*}
    \EXP{S_t(\delta)} 
    \leq \EXP{\sum_{s < \stopt}\sigmaZero^2 + (1+\sigmaOne^2+ \constOneStep)\gradsq{s}},
\end{align*}
establishing the third claim.

For the fifth claim, notice that, by definition of $\stopt$, if $\stopt = s$, then $X_{s-1}(\delta) = 1$, which implies that $S_{s-1}(\delta) \leq \nicefrac{\EXP{S_{s-1}(\delta)}}{\delta}$ by construction. Therefore,
\begin{align*}
    S_{\stopt - 1}(\delta) 
    = \sum_{s=1}^t S_{s-1}(\delta)\1{\stopt = s}
    &= \sum_{s=1}^t S_{s-1}(\delta)\1{\stopt = s, S_{s-1}(\delta)\leq \nicefrac{\EXP{S_{s-1}(\delta)}}{\delta}}\\
    &\leq \sum_{s=1}^t \frac{\EXP{S_{s-1}(\delta)}}{\delta}\1{\stopt = s}
    \leq \frac{\EXP{S_{t-1}(\delta)}}{\delta}.
\end{align*}

For the sixth claim, we note that
\begin{align*}
    \stopat{\stopt-1} 
    &= \sum_{s=1}^t \stopat{s-1} \1{\stopt = s}
    = \sum_{s=1}^{t} \stopat{s-1} \1{X_{0}(\delta)=\ldots=X_{s-1}(\delta)=1, \stopt=s}\\
    &= \sum_{s=1}^{t} (s-1) \1{\stopt=s}
    = \sum_{s=0}^{t-1} s \1{\stopt-1=s}
    = \stopt - 1.
\end{align*}

For the seventh claim, assuming $s < \stopt$, we have that, by \cref{lem:localSmoothnessBound},
\begin{align*}
    \tetaat{s} 
    & = \frac{\eta}{\sqrt{b_{0}^2 + \sigmaZero^2 + \sum_{\ell=1}^{s-1} \sgradsq{\ell} + \gradsq{s}}} \\
    & \geq \frac{\eta}{\sqrt{b_{0}^2 + \sigmaZero^2 + 2\eta^2\Lzero^2 + \sum_{\ell < \stopt - 1} \sgradsq{\ell} + \constOneStep\gradsq{\ell}}}.
\end{align*}
Further, since $\stopat{\stopt-1} = \stopt - 1$, and by definition of $S_t(\delta)$, we have that
\begin{align*}
    \tetaat{s} 
    &\geq \frac{\eta}{\sqrt{b_{0}^2 + \sigmaZero^2 + 2\eta^2\Lzero^2 + \sum_{\ell < \stopat{\stopt - 1}} \sgradsq{\ell} + \constOneStep\gradsq{\ell}}} \\
    & = \frac{\eta}{\sqrt{b_{0}^2 + \sigmaZero^2 + 2\eta^2\Lzero^2 + S_{\stopt - 1}(\delta)}}.
\end{align*}
Therefore, since $S_{\stopt - 1}(\delta) \leq \nicefrac{\EXP{S_{t-1}(\delta)}}{\delta}$ almost surely by \cref{item:niceStopping5}, together with our upper-bound on $\EXP{S_t(\delta)}$ from \cref{item:niceStopping4}, we conclude that
\begin{align*}
    \tetaat{s}
    \geq \frac{\eta}{\sqrt{b_{0}^2 + 2\eta^2\Lzero^2 + \frac{(t-1)\sigmaZero^2 + (1+\sigmaOne^2 + \constOneStep)\EXP{\sum_{\ell < \stopat{t-1}}\gradsq{\ell}}}{\delta}}},
\end{align*}
as claimed.

For the final claim, we note that $\stopt \leq t$ deterministically, by construction. Thus, we focus on the lower bound. Indeed, notice that, since $\stopt \in [t]$,
\begin{align*}
    \stopt
    = \sum_{s = 1}^t s\1{\stopt = s}
    &= \sum_{s = 1}^t \1{\stopt = s}\sum_{\ell = 0}^{s-1} \1{\stopt > \ell}\\
    &= \sum_{\ell = 0}^{t-1}\sum_{s=\ell+1}^t \1{\stopt = s} \1{\stopt > \ell}\\
    &= \sum_{\ell = 0}^{t-1}\1{\stopt > \ell}\sum_{s=\ell+1}^t \1{\stopt = s} 
    = \sum_{\ell = 0}^{t-1}\1{\stopt > \ell}
\end{align*}
Next, notice that $X_s(\delta) = 1$ iff $\stopt > s$, which implies that $X_s(\delta) = \1{\stopt > s}$. Additionally, recall that $X_0(\delta) = 1$, and $X_s(\delta) = \1{\cap_{\ell=1}^s \braces{S_\ell(\delta) \leq \nicefrac{\EXP{S_\ell(\delta)}}{\delta}}}$. Hence, we have that
\begin{align*}
    \EXP{\stopt} 
    = \sum_{s=0}^{t-1} \EXP{X_s(\delta)}
    = \sum_{s=0}^{t-1} \PRO{X_s(\delta) = 1}
    &= 1 + \sum_{s=1}^{t-1} 1 - \PRO{X_s(\delta) = 0}\\
    &= 1 + \sum_{s=1}^{t-1} 1 - \PRO{\cup_{\ell=1}^s \braces{S_\ell(\delta) > \nicefrac{\EXP{S_\ell(\delta)}}{\delta}}}.
\end{align*}
Therefore, by applying the union bound and Markov's inequality, we conclude that
\begin{align*}
    \EXP{\stopt} 
    \geq t - \sum_{s=1}^{t-1}\sum_{\ell=1}^s \PRO{S_\ell(\delta) > \nicefrac{\EXP{S_\ell(\delta)}}{\delta}}
    \geq t - \sum_{s=1}^{t-1}\sum_{\ell=1}^s \delta
    = t - \delta \frac{t(t-1)}{2}
    = t\parens{1 - \frac{\delta (t-1)}{2}},
\end{align*}
which establishes the final claim.
\end{custproof}

\subsection{The key consequence of the nice stopping time construction}

The following result is the most crucial place where the properties of \cref{def:niceStopping} are utilized. It tells us that, as long as the sum of ``bad'' gradients is comparable to the sum of ``good'' ones, and as long as the descent inequality (\cref{lem:descentLemma}) holds, then the sum of gradients scales (roughly) as $\O(\nicefrac{b(T)^2}{\delta} + b(T)\sqrt{\nicefrac{T}{\delta}})$. One can compare this result to that of \citet[Lemma 13]{FTCMSW22}, which obtained a similar bound in the simpler $\Lzero$-smooth setting. Their argument utilized a technique they termed ``recursive improvement,'' which required recursively invoking gradually improving bounds in order to reach their desired conclusion after infinitely many calls. Moreover, their argument crucially relies on properties of $\Lzero$-smoothness in order to obtain worst-case upper bounds on the sum of gradients, which are no longer true in our setting. Through our construction of the stopping time $\stopat{T+1}$, we are able to obtain a similar bound as in their setting, but with an (arguably) significantly simpler and more general proof which works even in the $\smoothBoth$-smooth setting.

\begin{restatable}{lemma}{restateLinearSumBound}\label{lem:linearSumBound}
Recall the stopping time $\stopat{T+1}$ from \cref{def:niceStopping} and the set of ``good'' times before $\stopat{T+1}$, $\Sgstopped$ from \cref{def:goodTimes}.
Let $\tSgstopped \subseteq \Sgstopped$ be any (random) subset.
Suppose that the following two conditions are satisfied: (i) for some $\cbadone,\cbadtwo \geq 0$ (possibly dependent on $T$):
\begin{align}\label{eq:linearSumBoundAssump1}
    \EXP{\sum_{t\in\tSgstopped^c} \gradtsq} \leq \cbadone + \cbadtwo\EXP{\sum_{t \in\tSgstopped} \gradtsq}
\end{align}
and (ii) for some $b(T)\geq 0$,
\begin{align}\label{eq:linearSumBoundAssump2}
    \EXP{\sum_{t\in\tSgstopped} \tetat\gradtsq}\leq b(T).
\end{align}
Then, we obtain the inequality given below:
\begin{align*}
    \EXP{\sum_{t\in \tSgstopped} \gradtsq} 
    &\leq \frac{2(1+\cbadtwo)(1+\constOneStep+\sigmaOne^2)b(T)^2}{\eta^2\delta} \\
    &\quad+ \frac{2b(T)}{\eta}\sqrt{b_0^2 + 2\eta^2 \Lzero^2 + \frac{T\sigmaZero^2 + (1+\constOneStep+\sigmaOne^2)\cbadone}{\delta}}.
\end{align*}
\end{restatable}
\begin{proof}
Let $\tSgstopped \subseteq \Sgstopped$ be any (possibly random) subset.
By \eqref{eq:linearSumBoundAssump2},
\begin{align*}
    b(T)
    \geq \EXP{\sum_{t\in \tSgstopped} \tetat \gradtsq}.
\end{align*}
Now, by \cref{item:niceStopping7} in \cref{lem:niceStopping}, since $t < \stopt$ for any $t\in\tSgstopped$, 
\begin{align*}
    &\EXP{\sum_{t\in \tSgstopped} \tetat \gradtsq}\\
    &\geq \EXP{\sum_{t\in \tSgstopped} \frac{\eta\gradtsq}{\sqrt{b_0^2 + 2\eta\Lzero^2 + \frac{T\sigmaZero^2 + (1+\sigmaOne^2+\constOneStep)\EXP{\sum_{\ell < \stopat{T+1}} \gradsq{\ell}}}{\delta}}}}\\
    &= \frac{\EXP{\sum_{t\in \tSgstopped} \eta\gradtsq}}{\sqrt{b_0^2 + 2\eta\Lzero^2 + \frac{T\sigmaZero^2 + (1+\sigmaOne^2+\constOneStep)\EXP{\sum_{\ell < \stopat{T+1}} \gradsq{\ell}}}{\delta}}}\\
    &= \frac{\eta\Eg}{\sqrt{b_0^2 + 2\eta\Lzero^2 + \frac{T\sigmaZero^2 + (1+\sigmaOne^2+\constOneStep)(\Eg + \Eb)}{\delta}}},
\end{align*}
where
$\Eg = \EXP{\sum_{t\in\tSgstopped}\gradtsq}$ and $\Eb = \EXP{\sum_{t\in\tSgstopped^c}\gradtsq}$.
Rearranging, we have the following inequality:
\begin{align*}
    \eta \Eg \leq \sqrt{b_0^2 + 2 \eta^2 \Lzero^2 + \frac{T\sigmaZero^2 + (1 + \constOneStep + \sigmaOne^2)\Eb +(1 + \constOneStep + \sigmaOne^2)\Eg}{\delta}}b(T).
\end{align*}
Notice that this is a quadratic inequality in $\sqrt{\Eg}$. Assuming that
\begin{align*}
    \Eb \leq \cbadone + \cbadtwo\Eg,
\end{align*}
then we may solve this inequality to conclude that
\begin{align*}
    \sqrt{\Eg} 
    &\leq \frac{\sqrt{(1+\cbadtwo)(1+\constOneStep + \sigmaOne^2)}b(T)}{2\eta\sqrt{\delta}} \\
    &\quad+ \frac{1}{2\eta}\sqrt{\frac{(1+\cbadtwo)(1+\constOneStep + \sigmaOne^2)b(T)^2}{\delta} + 4\eta\sqrt{b_0^2 + 2\eta^2 \Lzero^2 + \frac{T\sigmaZero^2 + (1+\constOneStep + \sigmaOne^2)\cbadone}{\delta}}b(T)}\\
    &\leq \frac{\sqrt{(1+\cbadtwo)(1+\constOneStep + \sigmaOne^2)}b(T)}{\eta\sqrt{\delta}} + \sqrt{\frac{b(T)}{\eta}}\sqrt[4]{b_0^2 + 2\eta^2 \Lzero^2 + \frac{T\sigmaZero^2 + (1+\constOneStep+\sigmaOne^2)\cbadone}{\delta}},
\end{align*}
from which we conclude
\begin{align*}
    \EXP{\sum_{t\in\tSgstopped} \gradtsq} = \Eg
    &\leq \frac{2(1+\cbadtwo)(1+\constOneStep+\sigmaOne^2)b(T)^2}{\eta^2\delta}\\
    &\quad+ \frac{2b(T)}{\eta}\sqrt{b_0^2 + 2\eta^2 \Lzero^2 + \frac{T\sigmaZero^2 + (1+\constOneStep+\sigmaOne^2)\cbadone}{\delta}}.
\end{align*}
\end{proof}

\subsection{Convergence for $\smoothBoth$-smooth functions}

Here, we provide our main theorem for $\smoothBoth$-smooth functions. We emphasize that, unlike in the statement of \cref{thm:mainInformal} from the main body, this theorem does not (directly) require $\sigmaOne < 1$. Instead, it requires that $\sum_{t\in\tSgstopped^c}\gradtsq$, $\EXP{|\tSgstopped^c|}$, and $\compstopthm$ can each be upper-bounded by sufficiently-small quantities. While these quantities can each be (trivially) upper-bounded when $\sigmaOne < 1$, this is not a necessary condition. Indeed, we prove in \cref{cor:polyBoundedConvergence} convergence for a subset of $\smoothBoth$-smooth functions without a restriction on $\sigmaOne$ using this theorem as well.

\begin{restatable}[Formal statement of \cref{thm:mainInformal}]{theorem}{restateMainConvergenceThm}\label{thm:main}
Fix any $\eps,\eps',\eps'',\eps'''\in (0,1)$.
Consider \eqref{eq:alg} with any parameters $\eta \leq \nicefrac{2\eps'}{\Lone\parens{4 + \sigmaOne^2}}$ and $b_0^2 > 0$, running for $T\geq 1$ time steps on an objective function satisfying \cref{assump:smooth}, and given access to a stochastic gradient oracle satisfying \cref{assump:unbiasedGrad,assump:affineVariance}. 
Let, for any $\delta'\in(0,1)$, $\stopthm := \stopatdel{T+1}{\nicefrac{\delta'}{4 T}}$ be the stopping time from \cref{def:niceStopping}.
Let $\Sgthm$ by the set of ``good times'' from \cref{def:goodTimes}, let $\tSgthm \subseteq \Sgthm$, and denote $\Scthm := \Sgthm \setminus\tSgthm$ to be the compensating ``good'' times for the bad times $\Sgthm^c$. Suppose there is a (possibly random) $\Cbadone \geq 0$ and constant $\cbadtwo \geq 0$ such that $\EXP{\Cbadone} \leq \cbadone < \infty$ and which (deterministically) satisfy:
\begin{align*}
    \sum_{t\in\tSgthm^c} \gradtsq
    \leq \Cbadone + \cbadtwo\sum_{t\in\tSgthm} \gradtsq.
\end{align*}
then for any $T \geq 1$ and $\delta'\in (0,1)$, with probability at least $1-\delta' - \nicefrac{2 \EXP{|\tSgthm^c|}}{T}$, \eqref{eq:alg} satisfies: 
\begin{align*}
    &\min_{t\in[T]} \gradtsq\\
    &\leq \frac{32(1+\cbadtwo)b(T)^2}{\eta^2(\delta')^2 T} + \frac{16 b(T)}{\eta(\delta')^2 T}\sqrt{b_0^2 + \sigmaZero^2 + 2(1 + \sigmaOne^2)\cbadone + \frac{4 b(T)}{\eta}\sigmaOne^2(1+\cbadtwo)\sqrt{b_0^2 + 2\eta^2\Lzero^2}}\\
    &\quad+\frac{32 b(T)^{\nicefrac{3}{2}}}{\eta^{\nicefrac{3}{2}}(\delta')^{2.25} T^{\nicefrac{3}{4}}}\sqrt{2\sigmaOne^2(1+\cbadtwo)\sqrt{(1 + \constOneStep+\sigmaOne^2)\cbadone}}\\
    &\quad+\frac{16 b(T)}{\eta (\delta')^2 \sqrt{T}}\sqrt{2\sigmaZero^2 + \frac{8\sigmaOne^2(1+\cbadtwo)b(T)}{\eta\sqrt{\delta'}}\parens{\frac{2(1+\cbadtwo)(1+\constOneStep+\sigmaOne^2)b(T)}{\eta \sqrt{\delta'}} + \sigmaZero}},
\end{align*}
where $\constOneStep = 2(1+\eta\Lone)^2$,
\begin{align*}
    b(T) := \frac{1}{\eps'''}\parens{\fzero - \fstar + 2 \tcCommon\logp{\frac{(2+\sigmaOne^2)\tcCommon \EXP{\stopthm-1}}{\eta \eps''b_0}} + \frac{2\eta\eps'' \sigmaZero}{(2+\sigmaOne^2)} + \compstopthm},
\end{align*}
and
\begin{align*}
    \compstopthm = \EXP{\sum_{t\in\Sgthm^c} (\sigmaOne - (1-\eps-\eps'))\tetat\gradtsq - \sum_{t'\in\Scthm}\eps'''\tetaat{t'}\gradsq{t'}},
\end{align*}
and $\tcCommon = \frac{\eta\sigmaZero}{2\eps} + \eta^2 \frac{\Lzero + \sigmaZero \Lone}{2}$.

In particular, whenever $\sigmaOne \leq \constBiasInv\parens{1-(\eps+\eps' + \eps'' + \eps''')}$, then we have that $\Sgthm = [\stopthm-1]$, so we can take $\tSgthm = \Sgthm$ so that $\cbadone = 0 = \cbadtwo$ and $\compstopthm\leq 0$, so, with probability at least $1-\delta'$, the following inequality holds:
\begin{align*}
    \min_{t\in[T]} \gradtsq
    &\leq \frac{32b(T)^2}{\eta^2(\delta')^2 T} + \frac{16 b(T)}{\eta(\delta')^2 T}\sqrt{b_0^2 + \sigmaZero^2 + \frac{4 b(T)}{\eta}\sigmaOne^2\sqrt{b_0^2 + 2\eta^2\Lzero^2}}\\
    &\quad+\frac{16 b(T)}{\eta (\delta')^2 \sqrt{T}}\sqrt{2\sigmaZero^2 + \frac{8\sigmaOne^2b(T)}{\eta\sqrt{\delta'}}\parens{\frac{2(1+\constOneStep+\sigmaOne^2)b(T)}{\eta \sqrt{\delta'}} + \sigmaZero}},
\end{align*}
\end{restatable}

\begin{proof}
Let us assume that $\eta \leq \nicefrac{2\eps'}{\Lone(4+\sigmaOne^2)}$.
Let $\tSgstopped \subseteq\Sgstopped$.
Then, by \cref{lem:descentLemma}, and using the fact that $\Sgstopped \subset \Sg$ (recalling $\stopat{T+1}$ from \cref{def:niceStopping}),
\begin{align*}
    b(T) 
    &\geq \EXP{\sum_{t\in\tSgstopped} \tetat \gradtsq}\\
    &= \EXP{\sum_{t\in\tSgstopped} \frac{\eta\gradtsq}{\sqrt{b_0^2 + \sigmaZero^2 + \sum_{s=1}^{t-1} \sgradsq{s} + \gradtsq}}}\\
    &\geq \EXP{\sum_{t\in\tSgstopped} \frac{\eta\gradtsq}{\sqrt{b_0^2 + \sigmaZero^2 + \sum_{s=1}^{t-1} (2\normSq{\sgrad{s} - \grad{s}} + 2\gradsq{s}) + \gradtsq}}}\\
    &\geq \EXP{\tetalb\sum_{t\in\tSgstopped} \gradtsq}.
\end{align*}
where we denote
\begin{align*}
    \tetalb = \frac{\eta}{\sqrt{b_0^2 + \sigmaZero^2 + 2\sum_{s < \stopat{T+1}}\normSq{\sgrad{s}-\grad{s}} + \gradsq{s}}}.
\end{align*}
Now, applying H\"older's inequality to the above, we have:
\begin{align*}
    \frac{\EXP{\sqrt{\sum_{t\in\tSgstopped} \gradtsq}}^2}{\EXP{\nicefrac{1}{\tetalb}}}
    \leq b(T),
\end{align*}
where we used the following version of H\"older's:
\begin{align*}
    \EXP{X^2} \geq \frac{\EXP{X Y}^2}{\EXP{Y^2}},
\end{align*}
with $X = \sqrt{\tetalb \sum_{t\in\tSgstopped} \gradtsq}$ and $Y = \sqrt{\nicefrac{1}{\tetalb}}$.

Now, we have that
\begin{align*}
    \EXP{\nicefrac{\eta}{\tetalb}} 
    &= \EXP{\sqrt{b_0^2 + \sigmaZero^2 + 2\sum_{s < \stopat{T+1}} \normSq{\sgrad{s} - \grad{s}} + \gradsq{s}}}\\
    &\leq \EXP{\sqrt{b_0^2 + \sigmaZero^2 + 2\Cbadone + 2\sum_{s < \stopat{T+1}} \normSq{\sgrad{s} - \grad{s}} + 2(1 + \cbadtwo)\sum_{s\in\tSgstopped}\gradsq{s}}}\\
    &\leq \EXP{\sqrt{b_0^2 + \sigmaZero^2 + 2\Cbadone + 2\sum_{s < \stopat{T+1}} \normSq{\sgrad{s} - \grad{s}}}}\\
    &\quad+ \sqrt{2(1+\cbadtwo)}\EXP{\sqrt{\sum_{s\in\tSgstopped}\gradsq{s}}}\\
    &\leq \sqrt{b_0^2 + \sigmaZero^2 + 2\Cbadone + 2\EXP{\sum_{s < \stopat{T+1}} \normSq{\sgrad{s} - \grad{s}}}}\\
    &\quad+ \sqrt{2(1+\cbadtwo)}\EXP{\sqrt{\sum_{s\in\tSgstopped}\gradsq{s}}},
\end{align*}
Recalling that $\braces{s < \stopat{T+1}}\in\F_{s-1}$ by \cref{item:niceStopping2} of \cref{lem:niceStopping}, we may apply \cref{assump:affineVariance} to obtain
\begin{align*}
    \EXP{\normSq{\sgrad{s} - \grad{s}}\1{s < \stopat{T+1}}}
    &=\EXP{\Et{\normSq{\sgrad{s} - \grad{s}}}\1{s<\stopat{T+1}}}\\
    &\leq\EXP{\sigmaZero^2 + \sigmaOne^2\gradsq{s}\1{s < \stopat{T+1}}}.
\end{align*}
Using this bound, we obtain
\begin{align*}
    \EXP{\nicefrac{\eta}{\tetalb}} 
    &\leq \sqrt{b_0^2 + (2T+1)\sigmaZero^2 + 2\cbadone + 2\sigmaOne^2\EXP{\sum_{s < \stopat{T+1}} \gradsq{s}}} \\
    &\quad+ \sqrt{2(1+\cbadtwo)}\EXP{\sqrt{\sum_{s\in\tSgstopped}\gradsq{s}}}\\
    &\leq \sqrt{b_0^2 + (2T+1)\sigmaZero^2 + 2\cbadone + 2\sigmaOne^2 \Eg + 2\sigmaOne^2\Eb}\\
    &\quad+ \sqrt{2(1+\cbadtwo)}\EXP{\sqrt{\sum_{s\in\tSgstopped}\gradsq{s}}},
\end{align*}
where we denote $\Eg = \EXP{\sum_{s\in\tSgstopped}\gradsq{s}}$ and $\Eb = \EXP{\sum_{s\in\tSgstopped^c}\gradsq{s}}$.
Collecting our results so far, and denoting $Z = \sqrt{\sum_{s\in\tSgstopped} \gradsq{s}}$, we have that
\begin{align*}
    \eta\EXP{Z}^2 \leq b(T)\parens{\sqrt{b_0^2 + (2T + 1)\sigmaZero^2 + 2\cbadone + 2\sigmaOne^2 \Eg + 2\sigmaOne^2\Eb} +\sqrt{2(1+\cbadtwo)}\EXP{Z}}.
\end{align*}
Thus, we obtain a quadratic inequality in $\EXP{Z}$, which we may solve to conclude that
\begin{align*}
    \EXP{Z} 
    &\leq \frac{\sqrt{2(1+\cbadtwo)}b(T)}{2\eta} \\& \quad + \frac{1}{2\eta}\sqrt{2(1+\cbadtwo) b(T)^2 + 4\eta b(T)\sqrt{b_0^2 + (2T+1)\sigmaZero^2 + 2\cbadone + 2\sigmaOne^2 \Eg + 2\sigmaOne^2\Eb}}\\
    &\leq \frac{\sqrt{2(1+\cbadtwo)}b(T)}{\eta} + \sqrt{\frac{b(T)}{\eta}}\sqrt[4]{b_0^2 + (2T+1)\sigmaZero^2 + 2\cbadone + 2\sigmaOne^2 \Eg + 2\sigmaOne^2\Eb}\\
    &\leq \frac{\sqrt{2(1+\cbadtwo)}b(T)}{\eta} + \sqrt{\frac{b(T)}{\eta}}\sqrt[4]{b_0^2 + (2T+1)\sigmaZero^2 + 2(1 + \sigmaOne^2)\cbadone + 2\sigmaOne^2(1+\cbadtwo) \Eg}
\end{align*}
where in the last inequality, we used the fact that $\Eb \leq \cbadone + \cbadtwo\Eg$. Thus, applying the bound on $\Eg$ from \cref{lem:linearSumBound}, we obtain:
\begin{align*}
    \EXP{Z}^2
    &\leq \frac{4(1+\cbadtwo)b(T)^2}{\eta^2} + \frac{2 b(T)}{\eta}\sqrt{b_0^2 + \sigmaZero^2 + 2(1 + \sigmaOne^2)\cbadone + \frac{4 b(T)}{\eta}\sigmaOne^2(1+\cbadtwo)\sqrt{b_0^2 + 2\eta^2\Lzero^2}}\\
    &\quad+\frac{2 b(T)}{\eta}\sqrt{2T\sigmaZero^2 + \frac{4\sigmaOne^2(1+\cbadtwo)b(T)}{\eta\sqrt{\delta}}\parens{\frac{(1+\cbadtwo)(1+\constOneStep+\sigmaOne^2)b(T)}{\eta \sqrt{\delta}} + \sqrt{T\sigmaZero^2 + (1+\constOneStep+\sigmaOne^2)\cbadone}}},
\end{align*}
where $\delta\in(0,1)$ is a parameter of our choosing. In particular, choosing (with foresight) $\delta = \nicefrac{\delta'}{4 T}$ for any $\delta'\in(0,1)$, the above can be rewritten as:
\begin{align*}
    \EXP{Z}^2
    &\leq \frac{4(1+\cbadtwo)b(T)^2}{\eta^2} + \frac{2 b(T)}{\eta}\sqrt{b_0^2 + \sigmaZero^2 + 2(1 + \sigmaOne^2)\cbadone + \frac{4 b(T)}{\eta}\sigmaOne^2(1+\cbadtwo)\sqrt{b_0^2 + 2\eta^2\Lzero^2}}\\
    &\quad+\frac{4 b(T)^{\nicefrac{3}{2}} \sqrt[4]{T}}{\eta^{\nicefrac{3}{2}}(\delta')^{\nicefrac{1}{4}}}\sqrt{2\sigmaOne^2(1+\cbadtwo)\sqrt{(1 + \constOneStep+\sigmaOne^2)\cbadone}}\\
    &\quad+\frac{2 b(T) \sqrt{T}}{\eta}\sqrt{2\sigmaZero^2 + \frac{8\sigmaOne^2(1+\cbadtwo)b(T)}{\eta\sqrt{\delta}}\parens{\frac{2(1+\cbadtwo)(1+\constOneStep+\sigmaOne^2)b(T)}{\eta \sqrt{\delta}} + \sigmaZero}}\\
    &:= C_T^2
\end{align*}
Observe that, by \cref{item:niceStopping8} of \cref{lem:niceStopping},
\begin{align*}
    \EXP{|[T]\setminus\tSgstopped|} 
    = \EXP{|[\stopat{T+1}, T] \cup \tSgstopped^c|} 
    &\leq \EXP{T - \stopat{T+1} + 1} + \EXP{|\tSgstopped^c|}\\
    &\leq \frac{\delta T(T+1)}{2} + \EXP{|\tSgstopped^c|}
\end{align*}
To obtain a convergence rate, we begin by noting, for any $\delta'\in(0,1)$, we can decompose
\begin{align*}
    \PRO{\min_{t\in [T]} \gradtsq > \frac{8 C_T^2}{(\delta')^2 T}}
    &=\PRO{\min_{t\in [T]} \gradtsq > \frac{8 C_T^2}{(\delta')^2 T}, |\tSgstopped| \leq \nicefrac{T}{2}}\\
    &\quad+\PRO{\min_{t\in [T]} \gradtsq > \frac{8 C_T^2}{(\delta')^2 T}, |\tSgstopped| > \nicefrac{T}{2}}
\end{align*}
The first term is easy to bound via Markov's inequality, since, choosing $\delta = \frac{\delta'}{4T} \leq \frac{\delta'}{2(T+1)}$ (since $T\geq 1$),
\begin{align*}
    \PRO{\min_{t\in [T]} \gradtsq > \frac{8 C_T^2}{(\delta')^2 T}, |\tSgstopped| \leq \nicefrac{T}{2}}
    &\leq\PRO{|\tSgstopped| \leq \nicefrac{T}{2}}\\
    &\leq \frac{2\EXP{|[T]\setminus\tSgstopped|}}{T}\\
    &\leq \frac{2}{T}\parens{\frac{\delta T(T+1)}{2} + \EXP{|\tSgstopped^c|}}\\
    &\leq \frac{\delta'}{2} + \frac{2 \EXP{|\tSgstopped^c|}}{T}.
\end{align*}
To bound the second term, we note that, whenever $|\tSgstopped| > \nicefrac{T}{2}$, then
\begin{align*}
    \min_{t\in[T]} \gradtsq
    \leq \min_{t\in \tSgstopped} \gradtsq
    &\leq \frac{1}{|\tSgstopped|} \sum_{t\in \tSgstopped} \gradtsq\\
    &\leq \frac{2}{T} \underbrace{\sum_{t\in \tSgstopped} \gradtsq}_{Z^2}
\end{align*}
Hence, we have by Markov's inequality and our previous bounds,
\begin{align*}
    \PRO{\min_{t\in [T]} \gradtsq > \frac{8 C_T^2}{(\delta')^2 T}, |S| > \nicefrac{T}{2}}
    \leq \PRO{Z > \frac{2 C_T}{\delta'}}
    \leq \delta'\frac{\EXP{Z}}{2 C_T}
    \leq \frac{\delta'}{2}.
\end{align*}
Now, we may apply the bound on $\Eg$ from \cref{lem:linearSumBound} with $\delta = \frac{\delta'}{4(T+1)}$ to conclude that, with probability at least $1-\delta' - \nicefrac{2\EXP{|\tSgstopped^c|}}{T}$,
\begin{align*}
    &\min_{t\in[T]} \gradtsq\\
    &\leq \frac{8 C_T^2}{(\delta')^2 T}\\
    &\leq \frac{32(1+\cbadtwo)b(T)^2}{\eta^2(\delta')^2 T} + \frac{16 b(T)}{\eta(\delta')^2 T}\sqrt{b_0^2 + \sigmaZero^2 + 2(1 + \sigmaOne^2)\cbadone + \frac{4 b(T)}{\eta}\sigmaOne^2(1+\cbadtwo)\sqrt{b_0^2 + 2\eta^2\Lzero^2}}\\
    &\quad+\frac{32 b(T)^{\nicefrac{3}{2}}}{\eta^{\nicefrac{3}{2}}(\delta')^{2.25} T^{\nicefrac{3}{4}}}\sqrt{2\sigmaOne^2(1+\cbadtwo)\sqrt{(1 + \constOneStep+\sigmaOne^2)\cbadone}}\\
    &\quad+\frac{16 b(T)}{\eta (\delta')^2 \sqrt{T}}\sqrt{2\sigmaZero^2 + \frac{8\sigmaOne^2(1+\cbadtwo)b(T)}{\eta\sqrt{\delta'}}\parens{\frac{2(1+\cbadtwo)(1+\constOneStep+\sigmaOne^2)b(T)}{\eta \sqrt{\delta'}} + \sigmaZero}},
\end{align*}
as claimed
\end{proof}

\subsection{A deferred proof for establishing \cref{lem:startingPoint}}
\label{sec:app-deferred}

Here, we give a bound which is used in proving \cref{lem:startingPoint}. We remark that this inequality is an extension of a similar one from \citep{FTCMSW22} (in the $\Lzero$-smooth setting) to the more general $\smoothBoth$-smooth setting. We additionally note that this bound has a better dependence on $\sigmaOne\bias$ than the analogous one in theirs.

\begin{lemma}\label{lem:intermediateStartingPoint}
Fix any $\eps \in (0,1)$.
Suppose that $\eta \leq \nicefrac{1}{\Lone}$. Then, for any time $t$, the iterates of \eqref{eq:alg} satisfy
\begin{align*}
    \Et{\ftplus - \ft} 
    &\leq - \tetat\parens{1 - \eps - \sigmaOne\bias}\gradtsq
    + \cCommon \Et{\frac{\sgradtsq}{b_t^2}}\\
    &\quad+\frac{\Lone \gradtnorm}{2} \Et{\etat^2 \sgradtsq},
\end{align*}
where 
\begin{align*}
    \cCommon = \frac{\eta \sigmaZero}{2\eps} + \frac{\eta^2 \Lzero}{2}
    \quad\text{and}\quad
    \bias = \constBias\sqrt{\Et{\frac{\sgradtsq}{b_t^2}}}.
\end{align*}
\end{lemma}
\begin{proof}
The proof proceeds using similar arguments as in \citep[Lemma 5]{FTCMSW22}. By \cref{lem:localSmoothnessBound} and the definition of \eqref{eq:alg}, we know that
\begin{align*}
    \Et{\ftplus - \ft}
    &\leq -\Et{\etat\innerProd{\gradt}{\sgradt}} + \frac{\Lzero + \Lone\gradtnorm}{2} \Et{\etat^2\sgradtsq}\\
    &\leq -\tetat\gradtsq - \Et{(\etat - \tetat)\innerProd{\gradt}{\sgradt}}\\
    &\quad+ \frac{\Lzero + \Lone\gradtnorm}{2} \Et{\etat^2\sgradtsq}.
\end{align*}
We begin by bounding the inner product term above as:
\begin{align*}
    -(\etat - \tetat)\innerProd{\gradt}{\sgradt} \leq \abs{\etat - \tetat}\gradtnorm \sgradtnorm.
\end{align*}
To bound this quantity, we begin by rewriting $\etat - \tetat$. Denoting $\tbt^2 := b_t^2 + \tgradtsq$, we have that
\begin{align*}
    \abs{\etat - \tetat}
    = \eta\abs{\frac{1}{\sqrt{b_{t-1}^2 + \sgradtsq}} - \frac{1}{\sqrt{b_{t-1}^2 + \tgradtsq}}}
    = \eta\frac{\abs{\tgradtsq - \sgradtsq}}{\tbt b_t(\tbt + \tbt)}
    = \eta\frac{\abs{\tgradtnorm - \sgradtnorm}\parens{\tgradtnorm + \sgradtnorm}}{\tbt b_t(\tbt + \tbt)}.
\end{align*}
Combining the above arguments, and applying H\"older's inequality, we have that
\begin{align*}
    &-\Et{(\etat-\tetat)\innerProd{\gradt}{\sgradt}}\\
    &\leq \Et{\abs{\etat-\tetat}\gradtnorm\sgradtnorm}\\
    &= \tetat\gradtnorm \Et{\frac{\sgradtnorm\parens{\sgradtnorm + \tgradtnorm}}{b_t(\tbt + b_t)}\abs{\tgradtnorm - \sgradtnorm}}\\
    &\leq \tetat\gradtnorm \sqrt{\Et{\frac{\sgradtsq\parens{\sgradtnorm + \tgradtnorm}^2}{b_t^2(\tbt + b_t)^2}}}\sqrt{\Et{\abs{\tgradtnorm - \sgradtnorm}^2}}.
\end{align*}
By \eqref{eq:sgradBound}, $\Et{\sgradtsq}\leq \sigmaZero^2 + (1+\sigmaOne^2)\gradtsq$, and by \cref{assump:unbiasedGrad} and Jensen's inequality, $\Et{\sgradtnorm} \geq \norm{\Et{\sgradt}} = \gradtnorm$. Therefore,
\begin{align*}
    \Et{\abs{\tgradtnorm - \sgradtnorm}^2}
    &= \tgradtsq + \Et{\sgradtsq} - 2\tgradtnorm\Et{\sgradtnorm}\\
    &\leq \tgradtsq + \sigmaZero^2 + (1+\sigmaOne^2)\gradtsq - 2\tgradtnorm\gradtnorm\\
    &\leq 2\sigmaZero^2 + \sigmaOne^2\gradtsq,
\end{align*}
where the last step comes from $\tgradtnorm \geq \gradtnorm$.
Collecting our bounds so far yields:
\begin{align*}
    &-\Et{(\etat - \tetat)\innerProd{\gradt}{\sgradt}}\\
    &\leq \tetat\gradtnorm\sqrt{\Et{\frac{\sgradtsq(\tgradtnorm + \sgradtnorm)^2}{b_t^2(\tbt+b_t)^2}}}\sqrt{2\sigmaZero^2 + \sigmaOne^2\gradtsq}
\end{align*}
Focusing on the term depending on $\sigmaZero$, we have that for any $\eps > 0$,
\begin{align*}
    &\sqrt{2}\sigmaZero\tetat \gradtnorm\sqrt{\Et{\frac{\sgradtsq\parens{\tgradtnorm + \sgradtnorm}^2}{b_t^2(\tbt + b_t)^2}}}\\
    &\leq \eps\tetat \gradtsq + \frac{\sigmaZero^2\tetat}{2\eps}\Et{\frac{\sgradtsq\parens{\tgradtnorm + \sgradtnorm}^2}{b_t^2(\tbt + b_t)^2}}\\
    &\leq \eps\tetat \gradtsq + \frac{\sigmaZero\eta}{2\eps}\Et{\frac{\sgradtsq}{b_t^2}}.
\end{align*}
Thus, denoting $\bias = \constBias\sqrt{\Et{\nicefrac{\sgradtsq}{b_t^2}}}$ and $\cCommon = \nicefrac{\eta\sigmaZero}{2\eps} + \nicefrac{\eta^2\Lzero}{2}$, we have that
\begin{align*}
    \Et{\ftplus - \ft}
    &\leq -\tetat\parens{1-\eps - \sigmaOne\bias}\gradtsq + \cCommon \Et{\frac{\sgradtsq}{b_t^2}}\\
    &\quad+ \frac{\Lone\gradtnorm}{2} \Et{\etat^2\sgradtsq},
\end{align*}
as claimed by the lemma.
\end{proof}

\section{Proofs for Polynomially-bounded functions for general $\sigmaOne$}
\label{sec:polyBound-app}

In this section, we show that \cref{thm:main} can be used to establish a $\tO(\nicefrac{1}{\sqrt{T}})$ convergence rate without the restriction of $\sigmaOne < 1$. The key is to restrict our attention to $\smoothBoth$-smooth functions which satisfy the following additional property:

\subsection{The key definition and its properties}

\restatePolyBoundedDef*

The following result provides a characterization of these functions relative to $\Lzero$-smooth functions and $\smoothBoth$-smooth functions. In particular, it tells us that \cref{def:polynomiallyBounded} is a richer function class than $\smoothBoth$-smooth functions. However, not all $\smoothBoth$-smooth functions satisfy \cref{def:polynomiallyBounded}.

\begin{restatable}{proposition}{restatePolyBoundProperties}\label{prop:polyBoundProperties}
    We have the following:
    \begin{enumerate}
        \item Every $\Lzero$-smooth function satisfies \cref{def:polynomiallyBounded} with $k=2$, $c_k=1$, and $c_k' = \Lzero$.
        \item Every $\smoothBoth$-smooth function satisfies \cref{def:polynomiallyBounded} \emph{locally} (i.e., when $\norm{\w - \w'} \leq \nicefrac{1}{\Lone}$) with $k=2, c_k=2$ and $c_k' = \Lzero$.
        \item There is a $(0,\Lone)$-smooth function which does not satisfy \cref{def:polynomiallyBounded} for any fixed $k, c_k, c_k'$.
        \item For any $k\geq 2$, $F(\w) = \norm{\w - \w^*}^k$ satisfies \cref{def:polynomiallyBounded} with $k=k$, $c_k = 2^{k-2}$, and $c_k' = k 2^{k-2}$. Additionally, for any $\Lone > 0$, $F(\w)$ is $(\nicefrac{2k(k-1)}{\Lone^{k-2}}, (e-1)(k-1)\Lone)$-smooth. However, this $F(\w)$ is not $\Lzero$-smooth when $k > 2$. 
    \end{enumerate}
    In particular, this implies that:
    \begin{align*}
        \braces{\text{$\Lzero$-smooth functions}}
        \subsetneq
        \braces{\text{$\smoothBoth$-smooth functions satisfying \cref{def:polynomiallyBounded}}}
        \subsetneq
        \braces{\text{$\smoothBoth$-smooth functions}}
    \end{align*}
\end{restatable}

\begin{proof}
    The first claim follows by noting that $\Lzero$-smooth functions satisfy, for every $\w,\w'\in\R^d$,
    \begin{align*}
        \gradnormat{\w} - \gradnormat{\w'}
        \leq \abs{\gradnormat{\w} - \gradnormat{\w'}}
        \leq \norm{\gradat{\w} - \gradat{\w'}}
        \leq \Lzero\norm{\w - \w'}.
    \end{align*}
    The second follows since, for any $\smoothBoth$-smooth function, for every $\norm{\w-\w'}\leq \Lone$,
    \begin{align*}
        \gradnormat{\w} - \gradnormat{\w'}
        \leq \norm{\gradat{\w} - \gradat{\w'}}
        &\leq (\Lzero + \Lone\gradnormat{\w'})\norm{\w - \w'}\\
        &\leq \Lzero\norm{\w - \w'} + \gradnormat{\w'}.
    \end{align*}

    For the third claim, consider the function $F(w) = \exp(\Lone w)$. Since $F''(w) = \Lone^2\exp(\Lone w) = \Lone F'(w)$. Suppose there were some $k, c_k, c_k'$ such that \cref{def:polynomiallyBounded} is satisfied. Then, it must be the case that, for any $x > 0$:
    \begin{align*}
        1 \geq \lim_{\alpha\to\infty} \frac{\exp(\Lone \alpha x) - c_k \exp(\Lone 0)}{c_k' (\alpha x)^{k-1}}
        & = \lim_{\alpha\to\infty} \frac{\Lone x \exp(\Lone \alpha x)}{c_k'(k-1) x^{k-1}\alpha^{k-2}} \\
        & = \frac{\Lone}{c_k'(k-1)x^{k-2}}\lim_{\alpha\to\infty} \frac{\exp(\Lone \alpha x)}{\alpha^{k-2}},
    \end{align*}
    where the inequality follows from the definition of \cref{def:polynomiallyBounded}, the first equality by L'H\^opital's rule, and the second by rewriting the previous expression. Repeating this argument $k-1$ times, this implies that
    \begin{align*}
        1 &\geq \lim_{\alpha\to\infty} \frac{\exp(\Lone \alpha x) - c_k \exp(\Lone 0)}{c_k' (\alpha x)^{k-1}}
        = \frac{\Lone^{k-1}}{c_k'(k-1)!}\lim_{\alpha\to\infty} \exp(\Lone \alpha x)
        = \infty,
    \end{align*}
    a contradiction. Hence, $\exp(\Lone x)$ cannot satisfy \cref{def:polynomiallyBounded}.
    
    For the final claim, we see that $F(\w)$ satisfies \cref{def:polynomiallyBounded} with $c_k = 2^{k-2}$ and $c_k'=k 2^{k-2}$ since, by Jensen's inequality,
    \begin{align*}
        \gradnormat{\w} 
        = k\norm{\w - \w^*}^{k-1}
        &= k 2^{k-1}\norm{\frac{1}{2}(\w - \w') + \frac{1}{2}(\w' - \w^*)}^{k-1}\\
        &\leq k 2^{k-2}(\norm{\w - \w'}^{k-1} + \norm{\w' - \w^*}^{k-1})\\
        &\leq 2^{k-2}(k \norm{\w - \w'}^{k-1} + \gradnormat{\w'}).
    \end{align*}
    Further, $F(\w)$ is also $(2k(k-1),(e-1)(k-1))$-smooth, since simple calculations yield that
    \begin{align*}
        \hessat{\w} = k(k-2)\norm{\w - \w^*}^{k-4}(\w - \w^*)(\w - \w^*)^\top + k\norm{\w - \w^*}^{k-2} I.
    \end{align*}
    In particular, this implies that $\w - \w^*$ is an eigenvector with largest eigenvalue, so, for any $\Lone > 0$,
    \begin{align*}
        \norm{\hessat{\w}} 
        = k(k-1)\norm{\w - \w^*}^{k-2}
        &\leq k(k-1)\max\braces{\Lone\norm{\w - \w^*}^{k-1}, \frac{1}{\Lone^{k-2}}}\\
        &\leq \frac{k(k-1)}{\Lone^{k-2}} + (k-1)\Lone\gradnormat{\w}.
    \end{align*}
    Therefore, by \citep[Corollary A.4]{ZJFW20}, for any $\norm{\w - \w'} \leq \nicefrac{1}{(k-1)\Lone}$,
    \begin{align*}
        \norm{\gradat{\w} - \gradat{\w'}} \leq \parens{\frac{2 k (k-1)}{\Lone^{k-2}} + (e-1)(k-1)\Lone\gradnormat{\w'}}\norm{\w - \w'}.
    \end{align*}
    Hence, $F$ is $(\nicefrac{2k(k-1)}{\Lone^{k-2}}, (e-1)(k-1)\Lone)$-smooth, as claimed. It is clear that this $F$ is not $\Lzero$-smooth for any $\Lzero$ when $k > 2$, since for any $\w\in\R^d$ such that $\norm{\w} > 0$,
    \begin{align*}
        \lim_{\alpha\to\infty} \frac{\norm{\gradat{\alpha\w + \w^*} - \gradat{\w^*}}}{\Lzero\norm{\alpha\w + \w^* - \w^*}}
        =\lim_{\alpha\to\infty} \frac{\gradnormat{\alpha\w + \w^*}}{\alpha\Lzero\norm{\w}}
        =\lim_{\alpha\to\infty} \frac{k\alpha^{k-2}\norm{\w}^{k-1}}{\Lzero\norm{\w}}
        = \infty.
    \end{align*}
\end{proof}

The following result demonstrates the difference in worst-case gradient norm scaling that $\smoothBoth$-smooth functions provide, versus the worst-case scaling of functions satisfying \cref{def:polynomiallyBounded}.

\begin{restatable}{proposition}{restatePolyBoundedScaling}\label{prop:polyBoundedScaling}
For any function satisfying \cref{assump:smooth}, and any algorithm producing iterates $(\w_s)_{s\geq 1}$ satisfying $\norm{\w_{s+1}-\w_s}\leq \eta\leq \nicefrac{1}{\Lone}$ for every $s\geq 1$, the following inequality holds for every $t > t'$:
\begin{align*}
    \gradtnorm - (1 + \eta\Lone)^{t-t'}\gradnorm{t'} \leq ((1+\eta\Lone)^{t-t'} - 1)\frac{\Lzero}{\Lone}.
\end{align*}
Moreover, this inequality is essentially unimprovable, in the sense that there exists a $(0,\O(\Lone))$-smooth function and $\eta \leq \nicefrac{1}{\Lone}$ such that $\gradnorm{T+1} = (1+\O(\eta\Lone))^{T+1}\gradnorm{1}$ for any $T\geq 1$, and a $(\O(\Lzero), \O(\Lone))$-smooth function such that $\gradnorm{1}=0$ and $\gradnorm{T+1} = \O(\nicefrac{\Lzero}{\Lone})((1 + \O(\eta\Lone))^T - 1)$. By contrast, any function satisfying \cref{def:polynomiallyBounded} satisfies:
\begin{align*}
    \gradtnorm - c_k\gradnorm{t'} \leq \max\braces{c_k' \eta^{k-1}(t-t')^{k-1}, \Lzero\eta(t-t')}.
\end{align*}
\end{restatable}

\begin{proof}
We begin by proving the first claim by by induction on $t - t'$. The base case of $t - t' = 1$ holds by definition, since
\begin{align*}
    \gradtnorm - \gradnorm{t-1}
    \leq \abs{\gradtnorm - \gradnorm{t-1}}
    &\leq \norm{\gradt - \grad{t-1}}\\
    &\leq (\Lzero + \Lone\gradnorm{t-1})\norm{\w_t - \w_{t-1}}\\
    &\leq \eta(\Lzero + \Lone\gradnorm{t-1}).
\end{align*}
Now, supposing the claim holds for $t - t' = 1,\ldots,s$, we have that:
\begin{align*}
    \gradtnorm 
    &\leq ((1 + \eta\Lone)^s - 1)\frac{\Lzero}{\Lone} + (1 + \eta\Lone)^s\gradnorm{t-s}\\
    &\leq ((1 + \eta\Lone)^s - 1)\frac{\Lzero}{\Lone} + (1 + \eta\Lone)^s\parens{\eta\Lzero + (1 + \eta\Lone)\gradnorm{t-(s+1)}}\\
    &= ((1 + \eta\Lone)^{s+1} - 1)\frac{\Lzero}{\Lone} + (1 + \eta\Lone)^{s+1}\gradnorm{t-(s+1)},
\end{align*}
where the first inequality follows by applying the induction hypothesis for $t-t'=s$, the second by applying the induction hypothesis for $t-t'=1$, and the final equality follows by rearranging the prior line. Thus, the inequality holds also at $t-t'=s+1$, and thus our claim holds by induction.

To see that this inequality is essentially unimprovable, let us consider first consider, for any $\Lone > 0$, the function:
\begin{align*}
    \fat{x} = \exp(\Lone x).
\end{align*}
Since $F''(x) = \Lone^2 \exp(\Lone x) = \Lone F'(x)$, it follows from \cref{prop:smoothPartialEquiv} that $\fat{\cdot}$ is $(0,(e-1)\Lone)$-smooth. Notice that, if $x_{s+1} - x_s = \eta = \nicefrac{1}{(e-1)\Lone}$, then, taking $x_1 = 0$ and $t \geq 1$,
\begin{align*}
    (1 + (e-1)(e^{\nicefrac{1}{(e-1)}}-1)\eta\Lone)^{t} F'(x_1)
    = \exp(\nicefrac{t}{(e-1)})
    = \exp(\eta\Lone t)
    = F'(x_{t+1}).
\end{align*}
Further, for any $\Lzero,\Lone > 0$, consider the function:
\begin{align*}
    \fat{x} = \frac{\Lzero}{2}x^2\exp(\Lone x) - \frac{\Lzero x}{\Lone}.
\end{align*}
Clearly,
\begin{align*}
    F'(x) &=  \Lzero x \exp(\Lone x) + \frac{\Lone\Lzero}{2} x^2 \exp(\Lone x) - \frac{\Lzero}{\Lone}\\
    F''(x) &= \Lzero \exp(\Lone x) + 2\Lone\Lzero x \exp(\Lone x) + \frac{\Lone^2\Lzero x^2}{2}\exp(\Lone x)\\
    &= \Lzero\parens{1 - \frac{\Lone^2 x^2}{2}} \exp(\Lone x) + 2\Lzero + 2 \Lone F'(x).
\end{align*}
Noting that $F''(x) \leq 2\Lzero + 2\Lone F'(x)$ when $\abs{x} \geq \nicefrac{\sqrt{2}}{\Lone}$, and $F''(x) \leq \Lone(2+\exp(\sqrt{2})) + 2\Lone F'(x)$ otherwise, it follows that $F$ is $(2(2+\exp(\sqrt{2}))\Lzero, 2(e-1)\Lone)$-smooth (by \cref{prop:smoothPartialEquiv}). Therefore, whenever $\eta = \nicefrac{1}{2(e-1)\Lone}$,
\begin{align*}
    \fat{x_{T+1}}
    &= \frac{\Lzero}{\Lone}\parens{\exp\parens{\eta\Lone T + \log(\Lone\eta T + \frac{\Lone^2\eta^2 T^2}{2})} - 1}\\
    &= \frac{\Lzero}{\Lone}\parens{\exp\parens{\O(\eta\Lone T)} - 1}\\
    &= \frac{\Lzero}{\Lone}\parens{\exp\parens{T\log(1 + \O(2(e-1)\eta\Lone))} - 1}\\
    &= \O\parens{\frac{2(2 + \exp(\sqrt{2}))\Lzero}{2(e-1)\Lone}}\parens{\parens{1 + \O(2(e-1)\eta\Lone)}^T - 1},
\end{align*}
where the first equality follows by rearranging the definition, the second since $\eta\Lone = \Theta(1)$, the third since $2\frac{c}{1 + c}(e-1)\eta\Lone =\frac{2c(e-1)\eta\Lone}{1 + 2c(e-1)\eta\Lone} \leq \log(1 + c 2(e-1)\eta \Lone) \leq 2c(e-1)\eta\Lone$, and the fourth by rearranging.

The final inequality follows immediately from \cref{fact:boundedStep,def:polynomiallyBounded}, which together imply that
\begin{align*}
    \gradtnorm - c_k\gradnorm{t'} \leq \max\braces{c_k'\eta^{k-1}(t-t')^{k-1}, \Lzero \eta (t-t')}.
\end{align*}
\end{proof}

\subsection{Bounding $\comp{\tau}$ from \cref{lem:descentLemma}}

In order to use \cref{thm:main}, recall that we must be able to bound the quantity $\compstop$. To accomplish this, we show that, if one can find ``good'' times $t'$ near to the ``bad'' time $t$ (that is, $t-t'$ is ``small''), then it is possible to bound $\comp{\tau}$. We remark that this result generalizes the compensation argument of \citep{FTCMSW22} to functions satisfying \cref{def:polynomiallyBounded}.

\begin{restatable}{lemma}{restateCompensationInsight}\label{lem:compensationInsight}
    Suppose that $\fat{\cdot}$ satisfies \cref{def:polynomiallyBounded} for some constants $k \geq 2, c_k \geq 1, c_k' > 0$. Fix any time $t \in [T]$, and let $\Scomp{t} \subset [T]$ be any set such that $t > \max(\Scomp{t})$ and $|\Scomp{t}| \leq \ncomp := \ceil{\frac{4 c_k^3(\sigmaOne - (1-\eps-\eps'))_+}{\eps'''}}$ (where $(x)_+:=\max\braces{0,x}$). Then, assuming $\braces{\w_t}_{t\geq 1}$ are the iterates corresponding to \eqref{eq:alg}, we have that either $|\Scomp{t}| < \ncomp$, or:
    \begin{align*}
    &(\sigmaOne - (1-\eps-\eps'))\tetat\gradtsq - \sum_{t'\in\Scomp{t}} \eps'''\tetaat{t'}\gradsq{t'}\\
    &\leq \frac{\eps''' \eta \ncomp }{2 c_k^2} \max\braces{c_k'\eta^{k-1}, \Lzero\eta}(t-\min(\Scomp{t}))^{k-1}.
    \end{align*}
\end{restatable}

\begin{proof}
    We first show that
    \begin{align}\label{eq:compensation}
        \frac{\tetat}{4 c_k^3} \gradtsq - \tetaat{t'}\gradsq{t'} 
        \leq 
        \frac{\eta}{2 c_k^2} \max\braces{c_k'\norm{\w_t - \w_{t'}}^{k-1}, \Lzero\norm{\w_t-\w_{t'}}}.
    \end{align}
    To see this, first observe that, recalling the definition of $\tetat$ from \cref{def:stepSizeProxy},
    \begin{align*}
        \frac{\frac{\tetat}{c_k} - \tetaat{t'}}{\eta}
        &= \frac{1}{c_k\tbt} - \frac{1}{\tbtp}\\
        &= \frac{\tbtp^2 - c_k^2 \tbt^2}{c_k\tbt\tbtp(c_k \tbt + \tbtp)}\\
        &= \frac{b_{t'-1}^2 - c_k^2 b_{t-1}^2 + \gradsq{t'} - c_k^2\gradtsq}{c_k\tbt\tbtp(c_k \tbt + \tbtp)}\\
        &\leq \frac{\gradsq{t'} - c_k^2 \gradtsq}{c_k\tbt\tbtp(c_k \tbt + \tbtp)}\\
        &= \frac{(\gradnorm{t'} - c_k \gradtnorm)(\gradnorm{t'} + c_k \gradtnorm)}{c_k\tbt\tbtp(c_k \tbt + \tbtp)}\\
        &\leq \frac{\max\braces{c_k'\norm{\w_t - \w_{t'}}^{k-1}, \Lzero\norm{\w_t-\w_{t'}}}(\gradnorm{t'} + c_k \gradtnorm)}{c_k\tbt\tbtp(c_k \tbt + \tbtp)}\\
        &\leq \frac{\max\braces{c_k'\norm{\w_t - \w_{t'}}^{k-1}, \Lzero\norm{\w_t-\w_{t'}}}}{c_k\gradtnorm\gradnorm{t'}}
    \end{align*}
    where we use the fact that $c_k \geq 1$ and $b_{t-1}^2 \geq b_{t'-1}^2$ (since $t \geq t'$) for the first inequality, the definition of \cref{def:polynomiallyBounded} for the second, and the definition of $\tbt$ from \cref{def:stepSizeProxy} for the third.
    Now, either $\gradtnorm \geq 2 \max\braces{c_k'\norm{\w_t - \w_{t'}}^{k-1}, \Lzero\norm{\w_t-\w_{t'}}}$, or not. In the first case,  we note that
    \begin{align*}
        c_k\gradnorm{t'} 
        &\geq \gradtnorm - \max\braces{c_k'\norm{\w_t - \w_{t'}}^{k-1}, \Lzero\norm{\w_t-\w_{t'}}}\\
        &\geq \gradtnorm - \frac{1}{2}\gradtnorm
        = \frac{1}{2}\gradtnorm,
    \end{align*}
    from which we may conclude that
    \begin{align*}
        \frac{\tetat}{4c_k^3}\gradtsq - \tetaat{t'}\gradsq{t'}
        &\leq \frac{1}{4 c_k^2}\parens{\frac{\tetat}{c_k} - \tetaat{t'}}\gradtsq\\
        &\leq \frac{\eta \max\braces{c_k'\norm{\w_t - \w_{t'}}^{k-1}, \Lzero\norm{\w_t-\w_{t'}}}}{4 c_k^3 \gradtnorm\gradnorm{t'}}\gradtsq\\
        &\leq \frac{\eta \max\braces{c_k'\norm{\w_t - \w_{t'}}^{k-1}, \Lzero\norm{\w_t-\w_{t'}}}}{4 c_k^3 \gradtnorm\frac{1}{2 c_k}\gradtnorm}\gradtsq\\
        &= \frac{\eta \max\braces{c_k'\norm{\w_t - \w_{t'}}^{k-1}, \Lzero\norm{\w_t-\w_{t'}}}}{2 c_k^2 }.
    \end{align*}
    In the alternate case that $\gradtnorm < 2 \max\braces{c_k'\norm{\w_t - \w_{t'}}^{k-1}, \Lzero\norm{\w_t-\w_{t'}}}$, we obtain:
    \begin{align*}
        \frac{\tetat}{4c_k^3}\gradtsq - \tetaat{t'}\gradsq{t'}
        &\leq\frac{\tetat}{4c_k^3}\gradtsq\\
        &\leq\frac{\eta}{4c_k^3}\gradtnorm\\
        &<\frac{\eta}{4c_k^3}2 \max\braces{c_k'\norm{\w_t - \w_{t'}}^{k-1}, \Lzero\norm{\w_t-\w_{t'}}}\\
        &=\frac{\eta}{2 c_k^3} \max\braces{c_k'\norm{\w_t - \w_{t'}}^{k-1}, \Lzero\norm{\w_t-\w_{t'}}},
    \end{align*}
    which, since $c_k \geq 1$, establishes \eqref{eq:compensation}.
    The lemma follows straightforwardly from \eqref{eq:compensation}. Indeed, note that the claimed inequality is trivially true whenever $|\Scomp{t}| = \ncomp = 0$, since this implies that $\sigmaOne \leq 1 - \eps - \eps'$. Otherwise, when $\ncomp > 0$, we have that
    \begin{align*}
        &(\sigmaOne - (1-\eps-\eps'))\tetat\gradtsq - \sum_{t'\in\Scomp{t}} \eps'''\tetaat{t'}\gradsq{t'}\\
        &=\sum_{t'\in\Scomp{t}} \frac{\sigmaOne - (1-\eps-\eps')}{\ncomp}\tetat\gradtsq - \eps'''\tetaat{t'}\gradsq{t'}\\
        &\leq\eps'''\sum_{t'\in\Scomp{t}} \frac{\tetat}{4 c_k^2}\gradtsq - \tetaat{t'}\gradsq{t'}\\
        &\leq\eps'''\sum_{t'\in\Scomp{t}} \frac{\eta \max\braces{c_k'\norm{\w_t - \w_{t'}}^{k-1}, \Lzero\norm{\w_t-\w_{t'}}}}{2 c_k^2}.
    \end{align*}
    Thus, by \cref{fact:boundedStep}, we conclude that:
    \begin{align*}
        &(\sigmaOne - (1-\eps-\eps'))\tetat\gradtsq - \sum_{t'\in\Scomp{t}} \eps'''\tetaat{t'}\gradsq{t'}\\
        &\leq\ncomp\eps''' \frac{\eta \max\braces{c_k'\eta^{k-1}(t-\min(\Scomp{t}))^{k-1}, \Lzero\eta(t-\min(\Scomp{t}))}}{2 c_k^2}\\
        &\leq \ncomp\eps''' \frac{\eta \max\braces{c_k'\eta^{k-1}, \Lzero\eta}}{2 c_k^2}(t-\min(\Scomp{t}))^{k-1}\\
    \end{align*}
    as claimed.
\end{proof}

We now show how to translate \cref{lem:compensationInsight} directly into a bound on $\comp{\tau}$. This shows that, in order to bound $\comp{\tau}$, it suffices to bound $\EXP{|\Sgstopped^c|^k}$ by a ``sufficiently small'' quantity (say, $\O(\log(T))$).

\restateCompBound*
\begin{proof}
First, let us construct $\tSgat{\tau}$ in the same manner as in \citep[Lemma 11]{FTCMSW22}. In particular, denote $\tb{i}$ as the $i$th largest ``bad'' time in $\Sgat{\tau}^c$, i.e., $\tb{1} = \max(\Sgat{\tau}^c)$, and, for every $i\in [2, |\Sgat{\tau}^c|]$,
    \begin{align*}
        \tb{i} = \max\braces{t \in \Sgat{\tau}^c : t < \tb{i-1}}.
    \end{align*}
    Then, to every ``bad'' time $\tb{i}$, associate a set $\Scomp{i}$ of the largest (at most) $\ncomp = \max\braces{0, \ceil{\frac{4 c_k^3 (\sigmaOne - (1-\eps-\eps'))}{\eps'''}}}$ ``good'' times before $\tb{i}$ that are not assigned to another $\tb{i'} > \tb{i}$. That is, denoting 
    \begin{align*}
        \tg{0,\ncomp} &:= +\infty\\
        \tg{i,1} &:= \max\braces{\Sgat{\tau} \cap [1,\min\braces{\tb{i},\tg{i-1,\ncomp}})}\\
        \tg{i,j+1} &:= \max\braces{t \in \Sgat{\tau} : t < \min\braces{\tb{i},\tg{i,j}}},
    \end{align*}
    where, when the maximum does not exist, we take $\tg{i,j} = -\infty$. We can then take
    \begin{align*}
        \Scomp{i} := \braces{\tg{i,j} : j \in [\ncomp], \tg{i,j} > - \infty}.
    \end{align*}
    Then, by \citep[Lemma 11]{FTCMSW22}, we have that, for some index $i^* \in [|\Sgat{\tau}^c|]$, and for every $i < i^*$,
    \begin{align}\label{eq:compBoundTimeDiff}
        |\Scomp{i}| = \ncomp \text{ and } \tb{i} - \min(\Scomp{i}) \leq \ncomp|\Sgat{\tau}^c| \text{ if $\ncomp > 0$.}
    \end{align}
    For the remaining $i \geq i^*$, $\tb{i} \leq \tb{i^*} \leq \ncomp|\Sgat{\tau}^c|$. Finally, we take:
    \begin{align*}
        \Scat{\tau} = \cup_{i \in [|\Sgat{\tau}^c|]} \Scomp{i}
        \quad\text{and}\quad
        \tSgat{\tau} = \Sgat{\tau} \setminus \Scat{\tau}
    \end{align*}

    We use these compensation sets to bound the quantity $\comp{\tau}$ from \cref{lem:descentLemma}. Indeed, we can decompose this quantity as follows:
    \begin{align*}
        \comp{\tau}
        &= \EXP{\sum_{t\in\Sgat{\tau}^c} (\sigmaOne - (1-\eps-\eps'))\tetat\gradtsq - \sum_{t'\in\Scat{\tau}} \eps'''\tetaat{t'}\gradsq{t'}}\\
        &= \EXP{\sum_{i < i^*} (\sigmaOne - (1-\eps-\eps'))\tetaat{\tb{i}}\gradsq{\tb{i}} - \sum_{t'\in\Scomp{i}} \eps'''\tetaat{t'}\gradsq{t'}}\\
        &\quad+ \EXP{\sum_{i\geq i^*} (\sigmaOne - (1-\eps-\eps'))\tetaat{\tb{i}}\gradsq{\tb{i}} - \sum_{t'\in\Scomp{i}} \eps'''\tetaat{t'}\gradsq{t'}},
    \end{align*}
    To obtain a bound on the first term, we can use \cref{lem:compensationInsight}. For the second, we trivially lower-bound $\sum_{t'\in\Scomp{i}}\eps'''\tetaat{t'}\gradsq{t'} \geq 0$. The resulting bound is:
    \begin{align*}
        \comp{\tau} 
        &\leq 
        \frac{\eps''' \eta \max\braces{c_k'\eta^{k-1}, \Lzero\eta} \ncomp}{2 c_k^2} \EXP{\sum_{i < i^*} (\tb{i} - \min(\Scomp{i}))^{k-1}}\\
        &\quad+(\sigmaOne - (1-\eps-\eps'))\EXP{\sum_{i\geq i^*} \tetaat{\tb{i}}\gradsq{\tb{i}}}.
    \end{align*}
    Next, using \eqref{eq:compBoundTimeDiff} to bound $\tb{i} - \min(\Scomp{i})$ for each $i<i^*$, and recalling $\tetat\gradtsq\leq \eta\gradtnorm$ for every $t$, the above bound becomes:
    \begin{align*}
        \comp{\tau} 
        &\leq 
        \frac{\eps''' \eta \max\braces{c_k'\eta^{k-1}, \Lzero\eta} \ncomp^k}{2 c_k^3} \EXP{\sum_{i < i^*} |\Sgat{\tau}^c|^{k-1}}\\
        &\quad+(\sigmaOne - (1-\eps-\eps'))_+\EXP{\sum_{i\geq i^*} \eta\gradnorm{\tb{i}}}.
    \end{align*}
    Now, notice that both summation ranges $i<i^*$ and $i\geq i^*$ are of size at most $|\Sgat{\tau}^c|$. Thus, the first term can be bounded as:
    \begin{align*}
        \EXP{\sum_{i < i^*} |\Sgat{\tau}^c|^{k-1}}
        \leq\EXP{|\Sgat{\tau}^c|^{k}}.
    \end{align*}
    To bound the second term, we apply \cref{def:polynomiallyBounded,fact:boundedStep}, together with the above construction, to obtain:
    \begin{align*}
        \EXP{\sum_{i\geq i^*}\gradnorm{\tb{i}}}
        &\leq \EXP{\sum_{i\geq i^*}c_k\gradzeronorm + \max\braces{c_k'\eta^{k-1},\Lzero\eta}(\tb{i})^{k-1}}\\
        &\leq \EXP{\sum_{i\geq i^*}c_k\gradzeronorm + \max\braces{c_k'\eta^{k-1},\Lzero\eta}\ncomp^{k-1}|\Sgat{\tau}^c|^{k-1}}\\
        &\leq c_k\gradzeronorm\EXP{|\Sgat{\tau}^c|} + \max\braces{c_k'\eta^{k-1},\Lzero\eta}\ncomp^{k-1}\EXP{|\Sgat{\tau}^c|^{k}}
    \end{align*}
    Collecting results, we have that:
    \begin{align*}
        \comp{\tau}
        &\leq \eta(\sigmaOne - (1-\eps-\eps'))_+ c_k\gradzeronorm \EXP{|\Sgat{\tau}^c|}\\
        &\quad+ \eta\ncomp^{k-1}\max\braces{c_k'\eta^{k-1},\Lzero\eta}\parens{(\sigmaOne - (1-\eps-\eps'))_+ + \frac{\eps'''\ncomp}{2 c_k^3}}\EXP{|\Sgat{\tau}^c|^k},
    \end{align*}
    as claimed.
\end{proof}

The next result, combined with \cref{lem:compBound}, completes our goal of bounding $\compstop$ by $\poly\log(T)$.

\restateBadSetBound*

\begin{proof}
    Note that we can write $|\Sgstopped^c|$ as:
    \begin{align*}
        |\Sgstopped^c| = \sum_{t < \stopat{T+1}} \1{t \in\Sgstopped^c}.
    \end{align*}
    Thus, by the Multinomial theorem, we have that
    \begin{align*}
        |\Sgstopped^c|^k
        &=\sum_{\substack{k_1,\ldots,k_{\stopat{T+1}-1} \geq 0\\k_1 + \ldots + k_{\stopat{T+1}-1} = k}} {k \choose k_1,\ldots,k_{\stopat{T+1}-1}} \prod_{t < \stopat{T+1}} \1{t \in \Sgstopped^c}^{k_t}\\
        &=\sum_{s=1}^k \sum_{1 \leq t_1 < \ldots < t_s \leq \stopat{T+1}-1}\sum_{\substack{k_{t_1},\ldots,k_{t_s} > 0\\k_{t_1} + \ldots + k_{t_s} = k}} {k \choose k_{t_1},\ldots,k_{t_s}} \prod_{\ell\in [s]} \1{t_\ell \in \Sgstopped^c}^{k_{t_\ell}}\\
        &=\sum_{s=1}^k \sum_{1 \leq t_1 < \ldots < t_s \leq \stopat{T+1}-1}\sum_{\substack{k_{t_1},\ldots,k_{t_s} > 0\\k_{t_1} + \ldots + k_{t_s} = k}} {k \choose k_{t_1},\ldots,k_{t_s}} \prod_{\ell\in [s]} \1{t_\ell \in \Sgstopped^c}\\
        &=\sum_{s=1}^k \sum_{1 \leq t_1 < \ldots < t_s \leq \stopat{T+1}-1} \prod_{\ell\in [s]} \1{t_\ell \in \Sgstopped^c} \sum_{\substack{k_{t_1},\ldots,k_{t_s} > 0\\k_{t_1} + \ldots + k_{t_s} = k}} {k \choose k_{t_1},\ldots,k_{t_s}},
    \end{align*}
    where in the second line, we rewrite the first summation as a sum over all possible support sets $\braces{t_1,\ldots,t_s} \subset [\stopat{T+1}-1]$ of size $s \in [k]$ of terms included in the summation. The third equality follows immediately from the second, since each $k_\ell > 0$. The final equality follows by rearranging the terms in the prior one. Now, by another application of the Multinomial theorem, we have that
    \begin{align*}
        \sum_{\substack{k_{t_1},\ldots,k_{t_s} > 0\\k_{t_1} + \ldots + k_{t_s} = k}} {k \choose k_{t_1},\ldots,k_{t_s}}
        \leq\sum_{\substack{k_{t_1},\ldots,k_{t_s} \geq 0\\k_{t_1} + \ldots + k_{t_s} = k}} {k \choose k_{t_1},\ldots,k_{t_s}}
        = s^k.
    \end{align*}
    Combining this with the above, we have the following:
    \begin{align*}
        |\Sgstopped^c|^k
        \leq \sum_{s=1}^k s^k \sum_{1 \leq t_1 < \ldots < t_s \leq \stopat{T+1}-1} \prod_{\ell\in [s]} \1{t_\ell \in \Sgstopped^c}.
    \end{align*}
    We claim that, for any $s\geq 1$, the inner summation term above is bounded in expectation by:
    \begin{align}\label{eq:innerTermBound}
        &\EXP{\sum_{1\leq t_1 < \ldots <t_s \leq \stopat{T+1}-1} \prod_{\ell\in[s]}\1{t_\ell \in \Sgstopped^c}}\nonumber\\
        &\leq \parens{\frac{\sigmaOne^2}{(1 - (\eps + \eps' + \eps'' + \eps'''))^2}
        \logp{f(\stopat{T+1})}}^s,
    \end{align}
    where $f(T) = e + \frac{e\sigmaZero^2(T-1) +e(1+\sigmaOne^2 + \constOneStep)\EXP{\sum_{t<\stopat{T}}\gradtsq}}{\delta b_0^2}$.
    We prove \eqref{eq:innerTermBound} via induction on $s$. We begin by observing that, for any $t'\geq 0$,
    \begin{align}\label{eq:logSumBound}
        \Econd{t'-1}{\sum_{t = t' + 1}^{\stopat{T+1}-1} \1{t\in\Sgstopped^c}}
        \leq \frac{\sigmaOne^2}{(1 - (\eps + \eps' + \eps'' + \eps'''))^2}
        \logp{f(\stopat{T+1})}.
    \end{align}
    To see this, first note that, by \cref{def:goodTimes}, and since $\braces{t < \stopat{T+1}} \in \F_{t-1}$ by \cref{lem:niceStopping}, for any $t' \geq 0$,
    \begin{align*}
        &\frac{(1 - (\eps + \eps'+\eps''+\eps'''))^2)}{\sigmaOne^2} \Econd{t'-1}{\sum_{t=t'+1}^T \1{t\in\Sgstopped^c}}\\
        &\leq \sum_{t=t'+1}^T \Econd{t'-1}{\Et{\frac{\sgradtsq}{b_t^2}}\1{t\in\Sgstopped^c}}\\
        &\leq \sum_{t=t'+1}^T \Econd{t'-1}{\Et{\frac{\sgradtsq}{b_t^2}}\1{t < \stopat{T+1}}}\\
        &= \sum_{t=t'+1}^T \Econd{t'-1}{\Et{\frac{\sgradtsq}{b_t^2}\1{t < \stopat{T+1}}}}\\
        &= \Econd{t'-1}{\sum_{t=t'+1}^{\stopat{T+1}-1} \frac{\sgradtsq}{b_t^2}}.
    \end{align*}
    Now, by \cref{lem:logSumIneq}, we have that
    \begin{align*}
        \sum_{t=t'+1}^{\stopat{T+1}-1} \frac{\sgradtsq}{b_t^2}
        \leq\sum_{t=t'+1}^{\stopat{T+1}-1} \frac{\sgradtsq}{b_0^2 + \sum_{s=t'+1}^t\sgradsq{s}}
        &\leq 1 + \sum_{t=t'+1}^{\stopat{T+1}-2} \frac{\sgradtsq}{b_0^2 + \sum_{s=t'+1}^t\sgradtsq}\\
        &\leq 1 + \logp{\frac{b_0^2 + \sum_{t=t'+1}^{\stopat{T+1}-2} \sgradtsq}{b_0^2}}.
    \end{align*}
    Now, by \cref{item:niceStopping4,item:niceStopping5,item:niceStopping6} of \cref{lem:niceStopping}, we have that, almost surely,
    \begin{align*}
        \sum_{t=t'+1}^{\stopat{T+1}-2} \sgradtsq
        &\leq \sum_{t < \stopat{T+1}-1} \sgradtsq + \constOneStep\gradtsq
        = \sum_{t < \stopat{\stopat{T+1}-1}} \sgradtsq + \constOneStep\gradtsq\\
        &= S_{\stopat{T+1}-1}
        \leq \frac{\EXP{S_{T}}}{\delta}
        \leq \frac{(T-1)\sigmaZero^2 + (1+\sigmaOne^2 + \constOneStep)\EXP{\sum_{t < \stopat{T}}\gradtsq}}{\delta}.
    \end{align*}
    Therefore, collecting these results, we conclude that, for any $t'\geq 0$,
    \begin{align*}
        &\Econd{t'-1}{\sum_{t=t'+1}^{\stopat{T+1}-1} \1{t \in\Sgstopped^c}}\\
        &\leq \frac{\sigmaOne^2}{(1-(\eps+\eps'+\eps''+\eps'''))^2}\Econd{t'-1}{\sum_{t=t'+1}^{\stopat{T+1}-1} \frac{\sgradtsq}{b_t^2}}\\
        &\leq \frac{\sigmaOne^2}{(1-(\eps+\eps'+\eps''+\eps'''))^2}\Econd{t'-1}{1 + \logp{1 + \frac{(T-1)\sigmaZero^2 + (1+\sigmaOne^2 + \constOneStep)\EXP{\sum_{t<\stopat{T}}\gradtsq}}{\delta b_0^2}}}\\
        &= \frac{\sigmaOne^2}{(1-(\eps+\eps'+\eps''+\eps'''))^2}\logp{f(T)},
    \end{align*}
    as claimed.
    
    Now, the base case of $s=1$ for \eqref{eq:innerTermBound} follows immediately from \eqref{eq:logSumBound} with $t'=0$. Let us now suppose that the claim \eqref{eq:innerTermBound} holds for some $s\geq 1$. Then, to apply the induction hypothesis, we begin by decomposing:
    \begin{align*}
        &\EXP{\sum_{1\leq t_1 < \ldots <t_{s+1} \leq \stopat{T+1}-1} \prod_{\ell\in[s+1]}\1{t_\ell \in \Sgstopped^c}}\\
        &= \sum_{1\leq t_1 < \ldots <t_{s} \leq T} \EXP{\1{t_i\in\Sgstopped^c~\forall i\in[s]}\sum_{t_{s+1} = t_{s}+1}^{\stopat{T+1}-1} \1{t_{s+1}\in\Sgstopped^c}}.
    \end{align*}
    Notice that the above expectation is a product of two terms: indicators depending of times $t_1,\ldots,t_s$, and those depending on $t_{s+1}>t_s$. Therefore, since, by \cref{lem:niceStopping,def:goodTimes},
    \begin{align*}
        \braces{t \in \Sgstopped^c}=\braces{t < \stopat{T+1}} \cap \braces{\text{$t$ is ``good''}}\in \F_{t_{s}-1},
    \end{align*}
    we may apply the tower rule of expectations and the inequality from \eqref{eq:logSumBound}:
    \begin{align*}
        &\EXP{\1{t_i\in\Sgstopped^c~\forall i\in[s]}\sum_{t_{s+1} = t_{s}+1}^{\stopat{T+1}-1} \1{t_{s+1}\in\Sgstopped^c}}\\
        &=\EXP{\Econd{t_{s}-1}{\1{t_i\in\Sgstopped^c~\forall i\in[s]}\sum_{t_{s+1} = t_{s}+1}^{\stopat{T+1}-1} \1{t_{s+1}\in\Sgstopped^c}}}\\
        &=\EXP{\1{t_i\in\Sgstopped^c~\forall i\in[s]}\Econd{t_{s}-1}{\sum_{t_{s+1} = t_{s}+1}^{\stopat{T+1}-1} \1{t_{s+1}\in\Sgstopped^c}}}\\
        &\leq \frac{\sigmaOne^2}{(1-(\eps+\eps'+\eps''+\eps'''))^2}\log(f(T))\EXP{\1{t_i\in\Sgstopped^c~\forall i\in[s]}}.
    \end{align*}
    Therefore, summing the above expression over $1\leq t_1<\ldots<t_s\leq T$ and applying the induction hypothesis, we conclude that:
    \begin{align*}
        &\EXP{\sum_{1\leq t_1 < \ldots <t_{s+1} \leq \stopat{T+1}-1} \prod_{\ell\in[s+1]}\1{t_\ell \in \Sgstopped^c}}\\
        &\leq \frac{\sigmaOne^2}{(1-(\eps+\eps'+\eps''+\eps'''))^2}\log(f(T))\EXP{\sum_{1\leq t_1 < \ldots <t_{s} \leq \stopat{T+1}-1} \1{t_i\in\Sgstopped^c~\forall i\in[s]}}\\
        &= \frac{\sigmaOne^2}{(1-(\eps+\eps'+\eps''+\eps'''))^2}\log(f(T))\EXP{\sum_{1\leq t_1 < \ldots <t_{s} \leq \stopat{T+1}-1} \prod_{\ell\in[s]}\1{t_\ell \in \Sgstopped^c}}\\
        &\leq \parens{\frac{\sigmaOne^2}{(1-(\eps+\eps'+\eps''+\eps'''))^2}\log(f(T))}^{s+1},
    \end{align*}
    which establishes \eqref{eq:innerTermBound} by induction.

    Finally, using \eqref{eq:innerTermBound}, we conclude that
    \begin{align*}
        \EXP{|\Sgstopped^c|^k} 
        &\leq \sum_{s\in[k]} s^k \EXP{\sum_{1 \leq t_1 < \ldots < t_s \leq T} \prod_{\ell\in [s]} \1{t_\ell\not\in\Sgstopped}}\\
        &\leq \sum_{s\in[k]} s^k \parens{\frac{\sigmaOne^2 \log(f(T))}{(1-(\eps+\eps'+\eps''+\eps'''))^2}}^s.
    \end{align*}
    Now, finally noting that, for any $x \geq 1$,
    \begin{align*}
        \sum_{s\in[k]} s^k x^s
        \leq x^k\sum_{s\in[k]} s^k
        \leq x^k\int_1^{k+1} s^k
        = \frac{x^k ((k+1)^{k+1} - 1)}{k+1}
        \leq x^k (k+1)^{k},
    \end{align*}
    we conclude that
    \begin{align*}
        \EXP{|\Sgstopped^c|^k} \leq \parens{\frac{(k+1) \sigmaOne^2 \log(f(T))}{(1- (\eps+\eps'+\eps''+\eps'''))^2}}^k,
    \end{align*}
    as claimed.
\end{proof}

\subsection{Bounding the sum of ``bad'' gradients by the sum of ``good'' ones}

We recall from \cref{thm:main} that, in order to use this bound, we need to show that the sum of ``bad'' gradients can be upper-bounded (relatively) by the sum of ``good'' ones. It turns out, for functions satisfying \cref{def:polynomiallyBounded}, this is possible, as we now show.

\begin{restatable}{lemma}{restateBadGradBound}\label{lem:badGradBound}
    Let $\tau\geq 1$ be any (possibly random) time, and consider any (possibly random) set $S(\tau) \subseteq [\tau-1]$. Denote $S(\tau)^c = [\tau - 1]\setminus S(\tau)$. Then, assuming $\fat{\cdot}$ satisfies \cref{def:polynomiallyBounded}, the following is satisfied deterministically:
    \begin{align*}
        \sum_{t\in S(\tau)^c} \gradtsq 
        &\leq 2 \max\braces{c_k'^2\eta^{2(k-1)}, \Lzero^2\eta^2}|S(\tau)^c|^{2k - 1} + 2 c_k^2\gradzerosq |S(\tau)^c|\\
        &\quad+ 2 c_k^2\sum_{t\in S(\tau)} \gradtsq.
    \end{align*}
    In particular, recalling $\stopat{T+1}$ as the stopping time from \cref{def:niceStopping} and $\Sgstopped$ the set of ``good'' times before $\stopat{T+1}$ from \cref{def:goodTimes}, we have that, for any $\tSgstopped \subseteq \Sgstopped$ such that $\EXP{|\tSgstopped^c|} \leq (1+\ncomp)\EXP{|\Sgstopped^c|}$, we have that:
    \begin{align*}
        \EXP{\sum_{t\in\tSgstopped^c} \gradtsq} 
        \leq \cbadone + \cbadtwo\EXP{\sum_{t\in\tSgstopped} \gradtsq},
    \end{align*}
    where
    \begin{align*}
        \cbadone
        &=2\max\braces{c_k'^2\eta^{2(k-1)}, \Lzero^2\eta^2}(\ncomp+1)^{2k-1}\parens{\frac{2k\sigmaOne^2\log(f(\stopat{T+1}))}{(1-(\eps+\eps'+\eps''+\eps'''))^2}}^{2k-1}\\
        &\quad+ 2c_k^2\gradzerosq(\ncomp+1)\parens{\frac{2\sigmaOne^2\log(f(\stopat{T+1}))}{(1-(\eps+\eps'+\eps''+\eps'''))^2}},\\
        \cbadtwo &= 2 c_k^2
    \end{align*}
\end{restatable}

\begin{proof}
    The proof of this result follows a similar argument as used in \cref{lem:compBound}. The main idea here is to, for every $t\in S(\tau)^c$ in decreasing order, find the first available time $t'\in S(\tau)$ which has not been associated with an earlier time from $S(\tau)^c$. Then, using \cref{def:polynomiallyBounded}, we show that, as long as $t$ and $t'$ are not too far apart, then $\gradtsq$ and $\gradsq{t'}$ must also be close. For some times $t\in S(\tau)^c$, there may not be such a $t'\in S(\tau).$ However, because of the greedy construction, these times must be relatively small (roughly within the first $|S(\tau)^c|$ time steps). Thus, as long as $|S(\tau)^c|$ is not ``too big'' (in expectation), then we can still bound these remaining terms. We now make these arguments precise.

    To begin, note that for every $t,t'\geq 1$, by \cref{def:polynomiallyBounded},
    \begin{align*}
        \gradtsq \leq 2 c_k^2\gradsq{t'} + 2 c_k'^2\norm{\w_t - \w_{t'}}^{2(k-1)}.
    \end{align*}
    We use this bound as follows: let us index the times in $S(\tau)^c$, denoting $\ttb{i}$ to be the $i$th largest time in $S(\tau)^c$, i.e., 
    \begin{align*}
        \ttb{1} = \max(S(\tau)^c) 
        \quad\text{and}\quad 
        \ttb{i} = \max(S(\tau)^c\setminus\parens{\cup_{i'=1}^{i-1}\braces{\ttb{i'}}}) ~ \forall i\in[2,|S(\tau)^c|].
    \end{align*}
    To each $\ttb{i}$ in decreasing order of time, associate the largest time $\ttg{i}$ in $S(\tau)$ before $\ttb{i}$ which has not already been associated with some other $\ttb{i'} > \ttb{i}$, as long as such a time exists. In particular, we take
    \begin{align*}
        \ttg{i} = \max\braces{t \in S(\tau) : t < \min\braces{\ttb{i}, \ttg{i-1}}},
    \end{align*}
    if such a time exists, and $\ttg{i} = -\infty$ otherwise. Let $i^*$ be the index of the largest time $\ttb{i^*}$ such that $\ttg{i^*}$ does not exist, i.e.,
    \begin{align*}
        i^* = \min\braces{i \in [|S(\tau)^c|] : S(\tau) \cap \left[1,\min\braces{\ttb{i}, \ttg{i-1}}\right) = \emptyset}.
    \end{align*}
    Notice that $\ttg{i} = -\infty$ for every $i \geq i^*$, and $\ttg{i} \in S(\tau)$ otherwise. Notice that, for every $i < i^*$, we have that
    \begin{align}
        \ttb{i} - \ttg{i} \leq |S(\tau)^c|.
    \end{align}
    Indeed, this follows by first decomposing
    \begin{align*}
        \ttb{i} - \ttg{i} 
        = |(\ttg{i}, \ttb{i}) \cap S(\tau)^c| + |(\ttg{i}, \ttb{i}) \cap S(\tau)| + 1.
    \end{align*}
    Notice that $|(\ttg{i}, \ttb{i}) \cap S(\tau)| \leq i-1$, since there are exactly $i-1$ times $\ttb{i'} > \ttb{i}$, and each has a time $\ttg{i'}\in S(\tau)$, which may lie on that interval. Note that there cannot be more than $i-1$ times $t\in S(\tau)$ on this interval, since this would violate our choice of $\ttg{i}$ as the largest time in $S(\tau)$ smaller than $\ttb{i}$ which wasn't assigned to an earlier $\ttb{i'}$. Further, notice that $|(\ttg{i}, \ttb{i}) \cap S(\tau)^c| \leq |S(\tau)^c| - i$ by definition of $\ttb{i}$. Combining these two bounds yields the claim.

    Next, notice that, for every $i \geq i^*$, 
    \begin{align}
        \ttb{i} \leq \ttb{i^*} \leq |S(\tau)^c|,
    \end{align}
    where the first inequality is by definition of $\ttb{i}$. To see the second inequality, we follow a similar argument as before. Indeed, observe that
    \begin{align*}
        \ttb{i^*} 
        = |[1, \ttb{i^*}) \cap S(\tau)^c| + |[1, \ttb{i^*}) \cap S(\tau)| + 1.
    \end{align*}
    By definition of $i^*$, $|[1, \ttb{i^*}) \cap S(\tau)| \leq i^* - 1$, since the only times $t\in S(\tau)$ on this interval can be $\ttg{1},\ldots,\ttg{i^*-1}$ by definition of $i^*$ (otherwise, we would have $\ttg{i^*} > -\infty$). Further, $|[1, \ttb{i^*}) \cap S(\tau)^c| \leq |S(\tau)^c| - i^*$ by definition of $\ttb{i^*}$. Combining these two bounds yields the claim.

    As a result, we have the following:
    \begin{align*}
        &\sum_{t\in S(\tau)^c} \gradtsq
        = \sum_{i =1}^{i^*-1} \gradsq{\ttb{i}} + \sum_{i=i^*}^{|S(\tau)^c|} \gradsq{\ttb{i}}\\
        &\leq \sum_{i=1}^{i^*-1} 2 c_k^2\gradsq{\ttg{i}} + 2 c_k'^2\norm{\w_{\ttb{i}} - \w_{\ttg{i}}}^{2(k-1)}\\
        &\quad+ \sum_{i=i^*}^{|S(\tau)^c|} 2 c_k^2\gradzerosq + 2 \max\braces{c_k'^2\norm{\w_{\ttb{i}}-\wzero}^{2(k-1)}, \Lzero^2\norm{\w_{\ttb{i}}-\wzero}^{2}}.
    \end{align*}
    Hence, by \cref{fact:boundedStep}, we have the bound
    \begin{align*}
        &\sum_{t\in S(\tau)^c} \gradtsq\\
        &\leq \sum_{t\in S(\tau)} 2c_k^2\gradsq{t} 
        + \sum_{i=1}^{i^*-1}2 \max\braces{c_k'^2\eta^{2(k-1)}, \Lzero^2\eta^2}(\ttb{i} - \ttg{i})^{2(k-1)}\\
        &\quad+ 2c_k^2 \gradzerosq |S(\tau)^c|
        + \sum_{i=i^*}^{|S(\tau)^c|} 2 \max\braces{c_k'^2\eta^{2(k-1)}, \Lzero^2\eta^2}\ttb{i^*}^{2(k-1)}\\
        &\leq \sum_{t\in S(\tau)} 2c_k^2\gradsq{t} 
        + 2 \max\braces{c_k'^2\eta^{2(k-1)}, \Lzero^2\eta^2}(i^*-1)|S(\tau)^c|^{2(k-1)}\\
        &\quad+ 2c_k^2 \gradzerosq |S(\tau)^c|
        + 2 \max\braces{c_k'^2\eta^{2(k-1)}, \Lzero^2\eta^2}(|S(\tau)^c| - (i^*-1))|S(\tau)^c|^{2(k-1)}\\
        &= \sum_{t\in S(\tau)} 2c_k^2\gradsq{t} 
        + 2 \max\braces{c_k'^2\eta^{2(k-1)}, \Lzero^2\eta^2}|S(\tau)^c|^{2k-1}
        + 2c_k^2 \gradzerosq |S(\tau)^c|,
    \end{align*}
    which is the first stated bound.

    To obtain the second, we apply the first, where we choose $\tau := \stopat{T+1}$ (the stopping time from \cref{def:niceStopping}) and $S(\tau) := \Sgstopped$ (the set of ``good'' times before $\stopat{T+1}$ from \cref{def:goodTimes}). Thus, for any $\tSgstopped \subseteq \Sgstopped$ for which $\EXP{|\tSgstopped^c|} \leq (1+\ncomp)\EXP{|\Sgstopped|}$, we conclude that:
    \begin{align*}
        &\EXP{\sum_{t\in\tSgstopped^c} \gradtsq} \\
        &\leq 2c_k^2\EXP{\sum_{t\in\tSgstopped} \gradtsq}\\ 
        &\quad+2\max\braces{c_k'^2\eta^{2(k-1)}, \Lzero^2\eta^2}(\ncomp+1)^{2k-1}\EXP{|\tSgstopped^c|^{2k-1}}\\
        &\quad+ 2c_k^2\gradzerosq(\ncomp+1)\EXP{|\Sgstopped^c|}\\
        &\leq 2c_k^2\EXP{\sum_{t\in\tSgstopped} \gradtsq}\\
        &\quad+2\max\braces{c_k'^2\eta^{2(k-1)}, \Lzero^2\eta^2}(\ncomp+1)^{2k-1}\parens{\frac{2k\sigmaOne^2\log(f(\stopat{T}))}{(1-(\eps+\eps'+\eps''+\eps'''))^2}}^{2k-1} \\
        &\quad+ 2c_k^2\gradzerosq(\ncomp+1)\parens{\frac{2\sigmaOne^2\log(f(\stopat{T}))}{(1-(\eps+\eps'+\eps''+\eps'''))^2}},
    \end{align*}
    where in the second inequality, we applied \cref{lem:badSetBound}.
    Thus, we obtain the claimed result.
\end{proof}

\subsection{Applying \cref{thm:main} to polynomially-bounded functions with no restriction on $\sigmaOne$}

Now that we have shown in the previous results how to upper bound $\compstop$, the sum of ``bad'' gradients, and the moments of the size of the ``bad'' set, we are now ready to establish our second main result: a convergence guarantee for functions satisfying $\smoothBoth$-smoothness and \cref{def:polynomiallyBounded}, which holds for arbitrary $\sigmaZero,\sigmaOne \geq 0$.

\begin{restatable}[of \cref{thm:main}; Formal statement of \cref{thm:polyBoundedConvergenceMain}]{corollary}{repPolyBoundedConvergence}\label{cor:polyBoundedConvergence}
    Fix any $\eps,\eps',\eps'',\eps'''\in (0,1)$ satisfying $\eps+\eps'+\eps''+\eps'''<1$. Consider \eqref{eq:alg} with any parameters $\eta \leq \nicefrac{2\eps'}{\Lone(4+\sigmaOne^2)}$ and $b_0^2 > 0$, running for $T\geq 1$ time steps on an objective function satisfying \cref{assump:smooth} as well as \cref{def:polynomiallyBounded} for some constants $k\geq 2, c_k\geq 1, c_k'>0$. Suppose that the stochastic gradient oracle satisfies \cref{assump:affineVariance} for any $\sigmaZero,\sigmaOne \geq 0$. Then, for any $\delta' \in (0,1)$ and $T \geq 1$, with probability at least $1 - \delta' - \frac{4(1+\ncomp)\sigmaOne^2\log(f(T))}{(1-(\eps+\eps'+\eps''+\eps'''))^2 T}$, \eqref{eq:alg} satisfies:
    \begin{align*}
        &\min_{t\in[T]} \gradtsq\\
        &\leq \frac{32(1+\cbadtwo)b(T)^2}{\eta^2(\delta')^2 T} + \frac{16 b(T)}{\eta(\delta')^2 T}\sqrt{b_0^2 + \sigmaZero^2 + 2(1 + \sigmaOne^2)\cbadone + \frac{4 b(T)}{\eta}\sigmaOne^2(1+\cbadtwo)\sqrt{b_0^2 + 2\eta^2\Lzero^2}}\\
        &\quad+\frac{32 b(T)^{\nicefrac{3}{2}}}{\eta^{\nicefrac{3}{2}}(\delta')^{2.25} T^{\nicefrac{3}{4}}}\sqrt{2\sigmaOne^2(1+\cbadtwo)\sqrt{(1 + \constOneStep+\sigmaOne^2)\cbadone}}\\
        &\quad+\frac{16 b(T)}{\eta (\delta')^2 \sqrt{T}}\sqrt{2\sigmaZero^2 + \frac{8\sigmaOne^2(1+\cbadtwo)b(T)}{\eta\sqrt{\delta'}}\parens{\frac{2(1+\cbadtwo)(1+\constOneStep+\sigmaOne^2)b(T)}{\eta \sqrt{\delta'}} + \sigmaZero}},
    \end{align*}
where $\ncomp = \ceil{\frac{4 c_k^3(\sigmaOne - (1-\eps-\eps'))_+}{\eps'''}}$, $\constOneStep = 2(1+\eta\Lone)^2$, $\cbadtwo = 2 c_k^2$,
\begin{align*}
    b(T) &:= \frac{1}{\eps'''}\parens{\fzero - \fstar + 2 \tcCommon\logp{\frac{(2+\sigmaOne^2)\tcCommon T}{\eta \eps''b_0}} + \frac{2\eta\eps'' \sigmaZero}{(2+\sigmaOne^2)} + \compstopcor}\\
        \compstopcor
        &= \eta(\sigmaOne - (1-\eps-\eps'))_+c_k\gradzeronorm \ell_1(T)\\
        &\quad+ \eta\ncomp^{k-1}\max\braces{c_k'\eta^{k-1},\Lzero\eta}\parens{(2c_k+1)\sigmaOne + \nicefrac{1}{2}}\ell_k(T)\\
        \cbadone
        &=2c_k^2\gradzerosq(\ncomp+1)\ell_1(T)
        + 2\max\braces{c_k'\eta^{k-1}, \Lzero\eta}^2(\ncomp+1)^{2k-1}\ell_{2k-1}(T),\\
       \ell_k(T) &= \parens{\frac{(k+1)\sigmaOne^2\log(f(T))}{(1-(\eps+\eps'+\eps''+\eps'''))^2}}^k\\
       f(T) &= e + 4e T\frac{\sigmaZero^2 T + (1 + \sigmaOne^2 + \constOneStep)\parens{2 T c_k^2\gradzerosq + 2 \max\braces{c_k'\eta^{k-1},\Lzero\eta}^2 T^{2(k-1)}}}{b_0^2 \delta'},
\end{align*}
and $\tcCommon = \frac{\eta\sigmaZero}{2\eps} + \eta^2 \frac{\Lzero + \sigmaZero \Lone}{2}$ (where we use the notation $(x)_+ := \max\braces{0,x}$).
\end{restatable}
\begin{proof}
    We apply \cref{thm:main} as follows. First, we observe that, as a consequence of \cref{lem:compBound}, together with the bound on $\EXP{|\Sgstopped^c|^k}$ from \cref{lem:badSetBound}, we have that:
    \begin{align*}
        &\compstop\\
        &\leq \eta(\sigmaOne - (1-\eps-\eps'))_+ c_k\gradzeronorm \parens{\frac{2 \sigmaOne^2 \log(f(\stopat{T}))}{(1- (\eps+\eps'+\eps''+\eps'''))^2}}\\
        &\quad+ \eta\ncomp^{k-1}\max\braces{c_k'\eta^{k-1},\Lzero\eta}\parens{\frac{(k+1) \sigmaOne^2 \log(f(\stopat{T}))}{(1- (\eps+\eps'+\eps''+\eps'''))^2}}^k\parens{(\sigmaOne - (1-\eps-\eps'))_+ + \frac{\eps'''\ncomp}{2 c_k^3}}.
    \end{align*}
    Next, by \cref{lem:badGradBound}, we know that
    \begin{align*}
        \sum_{t\in\tSgstopped^c} \gradtsq 
        \leq \Cbadone + \cbadtwo\sum_{t\in\tSgstopped} \gradtsq,
    \end{align*}
    where
    \begin{align*}
        \EXP{\Cbadone}
        \leq\cbadone
        &=2\max\braces{c_k'^2\eta^{2(k-1)}, \Lzero^2\eta^2}(\ncomp+1)^{2k-1}\parens{\frac{2k\sigmaOne^2\log(f(\stopat{T}))}{(1-(\eps+\eps'+\eps''+\eps'''))^2}}^{2k-1}\\
        &\quad+ 2c_k^2\gradzerosq(\ncomp+1)\parens{\frac{2\sigmaOne^2\log(f(\stopat{T}))}{(1-(\eps+\eps'+\eps''+\eps'''))^2}},\\
        \cbadtwo &= 2 c_k^2.
    \end{align*}
    Thus, the conditions to apply \cref{thm:main} are satisfied, and we obtain the convergence rate.
\end{proof}

\section{Many common algorithms for $\smoothBoth$-smooth optimization can diverge in the presence of multiplicative noise}
\label{sec:mostAlgsDiverge}

In this section, we consider the convergence behavior of several natural candidate algorithms which have been studied in the literature on $\smoothBoth$-smooth optimization.
These algorithms take the form $\w_{t+1} = \w_t - \ut$,
where $\ut$ takes a number of different forms, including: in Normalized SGD:
\begin{align}\label{eq:normSGD}\tag{NormSGD}
    \ut = \eta\frac{\sgradt}{\gamma + \sgradtnorm},
\end{align}
Clipped SGD:
\begin{align}\label{eq:clippedSGD}\tag{ClippedSGD}
    \ut = \eta\frac{\sgradt}{\max\braces{\gamma,\sgradtnorm}}
\end{align}
and Sign-SGD with Momentum (operations performed element-wise):
\begin{align}\label{eq:signSGD}\tag{SignSGD-M}
    \ut = \eta\frac{\mt}{\abs{\mt}}
    \quad\text{where}\quad
    \m_0 = \0,\quad
    \mt = \beta \m_{t-1} + (1-\beta)\sgradt
\end{align}
\citet{ZHSJ20,ZJFW20,CLOZZ22} prove $\O(\nicefrac{1}{\sqrt{T}})$ convergence of these algorithms in the setting of \eqref{eq:boundedSuppIntro}. 
In this section, we show that these step-size
choices for $\smoothBoth$-smooth optimization fail under
\eqref{eq:affineVarianceIntro}, despite working in the noiseless and
\eqref{eq:boundedSuppIntro} settings. Our negative results rely on the following stochastic gradient oracle construction:

\restateStochasticOracleConstruction*

\begin{proof}
Fix any $\eps, \sigmaOne \geq 0$.
We begin by establishing that \cref{assump:unbiasedGrad} holds for our construction of $\sgradat{\w}$. Begin by denoting $\delta = \nicefrac{(1+\eps)^2}{((1+\eps)^2 + \sigmaOne^2)}$. Under this notation, we have that
\begin{align*}
    1 + \frac{\sigmaOne^2}{1+\eps}
    = \frac{(1+\eps)^2 + \sigmaOne^2(1+\eps)}{(1+\eps)^2}
    = \frac{1}{\delta} + \eps\frac{(1-\delta)}{\delta}
    = \frac{1 + \eps(1-\delta)}{\delta}.
\end{align*}
Therefore, it follows that
\begin{align*}
    \EXP{\multnoise(\w)}
    = \parens{-\eps(1-\delta) + \eps(1-\delta)}
    = 1.
\end{align*}
Further, $\EXP{\addnoise(\w)} = 0$ by construction. Therefore, $\EXP{\sgradat{\w}} = \gradat{\w}$,
which establishes \cref{assump:unbiasedGrad}.  As for \cref{assump:affineVariance}, denote $c = 1+\eps$, then we have that
\begin{align*}
    \EXP{\multnoise(\w)^2} 
    &= (c-1)^2 \frac{\sigmaOne^2}{c^2 + \sigmaOne^2} + \parens{1 + \frac{\sigmaOne^2}{c}}^2 \frac{c^2}{c^2 + \sigmaOne^2}\\
    &= \frac{(c^2 + 1 - 2c)\sigmaOne^2 + c^2 + \sigmaOne^4 + 2 c \sigmaOne^2}{c^2 + \sigmaOne^2}\\
    &= (1+\sigmaOne^2).
\end{align*}
Further, $\EXP{\normSq{\addnoise(\w)}} = \sigmaZero^2$ by construction. Therefore, since $\multnoise(\w)$ and $\addnoise(\w)$ are independent, we conclude that
\begin{align*}
    \EXP{\sgradsqat{\w}}
    &= \EXP{\multnoise(\w)^2}\gradsqat{\w} + \EXP{\normSq{\addnoise(\w)}} + 2 \EXP{\innerProd{\multnoise(\w) \gradat{\w}}{\addnoise(\w)}}\\
    &= (1+\sigmaOne^2)\gradsqat{\w} + \sigmaZero^2 + 2 \innerProd{\EXP{\multnoise(\w)} \gradat{\w}}{\EXP{\addnoise(\w)}}\\
    &= (1+\sigmaOne^2)\gradsqat{\w} + \sigmaZero^2,
\end{align*}
which establishes \cref{assump:affineVariance} for any $\sigmaZero, \sigmaOne \geq 0$.
\end{proof}

\subsection{Overview of main negative results}

We establish all of the following negative results using the stochastic gradient oracle described in \cref{lem:stochasticOracleConstruction}. Before stating our results, let us briefly discuss some intuition behind why one should expect \eqref{eq:normSGD}, \eqref{eq:clippedSGD}, and \eqref{eq:signSGD} to fail under \cref{lem:stochasticOracleConstruction}. Consider the setting where $\sigmaOne \gg 1+\eps$. Then, notice that the stochastic gradient $\sgradt$ only has the same sign as $\gradt$ with roughly $\nicefrac{1}{\sigmaOne^2}$ probability. Otherwise, $\sgradt$ has the opposite sign as $\gradt$. Now, for an algorithm which incorporates the \emph{magnitude} of the stochastic gradients together with the signs, the oracle in \cref{lem:stochasticOracleConstruction} may not be so problematic -- indeed, even though the updates with correct sign are somewhat ``rare'', they are also of significantly larger magnitude compared to the updates with proper sign. However, notice that \eqref{eq:normSGD}, \eqref{eq:clippedSGD}, and \eqref{eq:signSGD} are (effectively) unit step-length algorithms (at least, in the setting where $\sgradtnorm \geq \gamma$). Thus, in many parameter regimes, all of these algorithms effectively disregard the magnitude of the stochastic gradients and only use their signs. This results in a biased random walk which never finds an iterate better than the initial one with constant probability. We formalize this intuition in the following:

\begin{lemma}[Informal statement of \cref{lem:divergenceNormSGD}]\label{lem:divergenceNormSGDMain}
    Fix any smoothness parameter $\Lzero > 0$, initial gap $\Delta > 0$, and affine variance parameter $\sigmaOne > 2\sqrt{2}$.
    Suppose that either: (i) \eqref{eq:signSGD} is run with parameter $0\leq \beta \leq 1 - \nicefrac{2\sqrt{2}}{3}\approx 0.057$ and $\eta > 0$ for $T\geq 1$ time steps, or (ii) \eqref{eq:normSGD} or \eqref{eq:clippedSGD} is run with $0\leq \gamma \leq \nicefrac{\sqrt{\sigmaOne^2 \Delta\Lzero}}{2}$ and $\eta > 0$ for $T\geq 1$ time steps, where, in either case, the algorithms are allowed an arbitrary initialization $x_1\in \R$, and each of these parameters can depend on $\Lzero, \Delta$ and $\sigmaOne$. Then, there exists a $1$-dimensional $(\Lzero,0)$-smooth function (which is also $\Lzero$-strongly convex) with $\fat{x_1}-\inf_{x\in\R}\fat{x}=\Delta$, and stochastic gradient oracle satisfying \cref{assump:unbiasedGrad,assump:affineVariance} with $\sigmaZero=0$ and the specified $\sigmaOne$, and for which, with constant probability (independent of $T$), $\min_{t\in[T]}\gradsqat{x_t} = \gradsqat{x_1}$.
\end{lemma}
We note that the statement \cref{lem:divergenceNormSGDMain} follows from \cref{lem:divergenceNormSGD} by choosing the parameter $\eps = \nicefrac{\sigmaOne}{2\sqrt{2}}$. The main takeaway here is that, for a reasonably wide range of parameters, \eqref{eq:normSGD}, \eqref{eq:clippedSGD}, and \eqref{eq:signSGD} can diverge in the affine variance setting, even for very simple smooth and strongly convex problems (in fact, even on a $1$-dimensional quadratic function). In particular, this says that, whenever \eqref{eq:normSGD} is run with $\gamma = 0$ (or \eqref{eq:signSGD} with $\beta=0$), then there is no parameter tuning with respect to $\eta$ such that $\min_{t\in[T]}\gradsqat{x_t}$ converges!

We also give a (weaker) negative result for the \eqref{eq:alg} in the ``large
variance'' regime. This result establishes that, whenever $\eta$ is not carefully tuned with respect to both $\Lone$ and $\sigmaOne$, then the algorithm does not converge with constant probability. The intuition for this result is that, with constant probability, the first $\approx \sigmaOne^2$ stochastic gradients all have the wrong sign. Whenever $\sigmaOne$ is ``large'' (i.e., scaling as $\poly\log(T)$), then after only $\poly\log(T)$ steps, the algorithm can reach an objective value which is $\poly(T)$-times larger than the initial condition. Further, after reaching such a large gradient value, the step sizes are always too small for the algorithm to recover from these wrong initial steps. This is because the \eqref{eq:alg} updates are normalized by the large previous gradients.

\begin{lemma}[Informal statement of \cref{lem:divergenceAGNorm}]\label{lem:divergenceAGNormMain}
    Fix any $\Lone > 0$, time horizon $T > 1$, and affine variance parameter 
    $$\sigmaOne \geq \max\braces{\parens{\frac{4(1+\sqrt{2})^2-2}{\log(\nicefrac{4}{3})}}^{2/3}, \parens{\frac{16\log(T-1))^2}{\log(\nicefrac{4}{3})}}^2}.$$
    Suppose that \eqref{eq:alg} is initialized at $x_1\in\R$ and run with any parameters $\eta \geq\nicefrac{1}{(2\Lone\sqrt{\sigmaOne})}$ and $0 < b_0^2 \leq \sqrt{\sigmaOne}\Lone^2\exp(2\Lone x_1)$ (where these parameter choices may depend on $\Lone$).
    Then, there exists a $1$-dimensional $(0,(e-1)\Lone)$-smooth function such that $\inf_{x\in\R} \fat{x} = 0$, and a stochastic gradient oracle satisfying \cref{assump:unbiasedGrad,assump:affineVariance} with $\sigmaZero=0$ and the specified $\sigmaOne$, for which, with probability at least $\nicefrac{3}{4}$, $\min_{t\in[T]}\gradsqat{x_t} = \gradsqat{x_1}$.
\end{lemma}

We note that the statement \cref{lem:divergenceAGNormMain} follows from \cref{lem:divergenceAGNorm} by choosing the parameters $\alpha = \nicefrac{1}{(2\sqrt{\sigmaOne})}$, $\eps = \sqrt[4]{\sigmaOne}-1$, and $\delta=\nicefrac{1}{4}$.
Let us compare the negative result in \cref{lem:divergenceAGNormMain} with the
convergence result in the $\Lzero$-smooth regime for the same algorithm from
\citep{FTCMSW22}. Indeed, their main result was that a
$\tO(\nicefrac{1}{\sqrt{T}})$ convergence rate is achievable without tuning
the parameters of the algorithm with respect to $\sigmaZero, \sigmaOne,$ or $\Lzero$.
Since their convergence rate depends only polynomially on $\sigmaOne$, this
rate is maintained (up to poly-logarithmic factors) even when $\sigmaOne =
\poly\log(T)$, without adjusting the parameters $\eta$ or $b_0$ of the
algorithm. By contrast, \cref{lem:divergenceAGNormMain} tells us that, in the
$\smoothBoth$-smooth regime, such a result is no longer possible. Indeed, if
$\eta$ is not sufficiently small, then the algorithm does not converge
with constant probability when $\sigmaOne^2 \gtrsim \poly\log(T)$!

\subsection{Full statement and proof of negative results for \eqref{eq:signSGD}, \eqref{eq:normSGD}, and \eqref{eq:clippedSGD}}

Here, we give the complete negative result for \eqref{eq:signSGD}, \eqref{eq:normSGD}, and \eqref{eq:clippedSGD}, and formalize the intuition given there.

\begin{restatable}[Formal statement of \cref{lem:divergenceNormSGDMain}]{lemma}{restateDivergenceNormSGD}\label{lem:divergenceNormSGD}
    Fix any $\Lzero> 0$, $\eps > 0$, $\sigmaOne^2 > (1+\eps)^2$, and $\Delta > 0$. 
    Let $x_1\in \R$, $\eta > 0$, $\gamma \in
    [0,\eps\sqrt{2\Delta\Lzero}]$, $\beta \in \left[0,1-\sqrt{1-\frac{\eps}{1+\eps + \nicefrac{\sigmaOne^2}{(1+\eps)}}}\right)\supset \left[0,\frac{\eps}{2(1+\eps+\nicefrac{\sigmaOne^2}{(1+\eps)})}\right)$, and $T\geq 1$ be arbitrary
    parameters (possibly dependent on $\Lzero$, $\eps$, $\sigmaOne$, and $\Delta$). For
    any $t\in[T]$, consider the (one-dimensional) process $\braces{x_t}_{t\geq 1}$ given in \eqref{eq:signSGD}, \eqref{eq:normSGD},
    or \eqref{eq:clippedSGD}.
    where, in the case that $m_t = 0$ (in the case of \eqref{eq:signSGD}) or $\mu + \abs{g_t}=0$ (in the case of \eqref{eq:normSGD}), $u_t\in\braces{\pm \eta}$ may be chosen arbitrarily as a (possibly randomized) function of $\braces{g_1,\ldots,g_t}$.
    Then, assuming that $\sigmaOne^2 > (1+\nicefrac{\gamma}{\eps\sqrt{2\Delta\Lzero}})(1+\eps)^2 \in [(1+\eps)^2,2(1+\eps)^2]$, there exists an $1$-dimensional $(\Lzero,0)$-smooth function (which is also $\Lzero$-strongly convex) with $\fat{x_1} - \inf_{x\in\R}\fat{x} = \Delta$, and stochastic gradient oracle which outputs stochastic gradients $g_t$ of $\gradat{x_t}$ which satisfy \cref{assump:unbiasedGrad,assump:affineVariance}, and such that:
    \begin{align*}
        \PRO{\min_{t\in[T]} \gradsqat{x_t} = \gradsqat{x_1}}
        \geq \parens{1-\delta}^{t_0 + 1},
    \end{align*}
    where
    \begin{align*}
        t_0 = \ceil{\frac{\sqrt{2\delta(1-\delta)\log(1+\nicefrac{2}{\delta})}}{\delta_0}\parens{1+\frac{2\delta(1-\delta)}{(\delta-\delta_0)^2}\logp{\frac{4(1-\delta)}{(\delta_0-\delta)^2}}}},
    \end{align*}
    and $\delta = \frac{1}{(1+\nicefrac{\sigmaOne^2}{(1+\eps)^2})} < \frac{1}{1+\nicefrac{1}{\clipStepMult}} = \delta_0\in [\nicefrac{1}{3}, \nicefrac{1}{2}]$ and $\nicefrac{1}{\clipStepMult} = 1 + \nicefrac{\gamma}{(\eps\sqrt{2\Delta\Lzero})}\in[1,2]$.
\end{restatable}
\begin{proof}
    Let us choose, for arbitrary $\Lzero > 0$ and $\Delta > 0$, the $(\Lzero, 0)$-smooth objective $\fat{x} = \nicefrac{\Lzero}{2}\ x^2$, and assume without loss of generality that $x_1 = -\sqrt{\nicefrac{2\Delta}{\Lzero}}$ (indeed, if this is not the case, then we can always translate the function $\fat{x}$ to be $\fat{x} = \nicefrac{\Lzero}{2}(x - x_1 -\sqrt{\nicefrac{2\Delta}{\Lzero}})^2$, and our arguments remain unchanged). Notice that $\fat{x_1} - \fstar = \fat{x_1} = \Delta$.
    
    Consider, for any $\eps > 0$ and $\sigmaOne^2 > (1+\eps)^2$, the stochastic gradient oracle from \cref{lem:stochasticOracleConstruction}, i.e.,
    \begin{align*}
        g(x) := \begin{cases}
            \parens{1 + \frac{\sigmaOne^2}{1+\eps}}\Lzero x & \text{w.p. } \frac{1}{1 + \frac{\sigmaOne^2}{(1+\eps)^2}} := \delta\\
            -\eps \Lzero x & \text{w.p. } 1-\frac{1}{1 + \frac{\sigmaOne^2}{(1+\eps)^2}} = 1-\delta,
        \end{cases}
    \end{align*}
    where the multiplicative noise is sampled i.i.d for each $x$.
    Since $\gradat{x} = \Lzero x$, this construction satisfies \cref{assump:unbiasedGrad,assump:affineVariance} by \cref{lem:stochasticOracleConstruction}.
    Further, denoting $\xClip := -\nicefrac{\gamma}{\eps\Lzero}$, our assumption that $\gamma \leq \eps\sqrt{2\Delta\Lzero} = -\eps\Lzero x_1$ and $\sigmaOne^2 > (1+\eps)^2$ (and thus also $\eps < 1+\nicefrac{\sigmaOne^2}{(1+\eps)}$) ensures:
    \begin{align}\label{eq:negResultStoppingTime}
        x_1 \leq \xClip < 0
        \quad\text{and}\quad
        \abs{g(\xClip)} \geq \eps\Lzero\abs{\xClip} = \gamma.
    \end{align}
    Let $\tau^*$ be the first time when an iterate becomes larger than the original one, i.e.,
    \begin{align*}
        \tau^* = \min\braces{t > 1 : x_1 \leq x_t}.
    \end{align*}
    Notice that this implies that, for any $1 \leq t < \tau^*$:
    \begin{align}\label{eq:negResultsBeforeStopping}
        x_t \leq x_1 \leq \xClip < 0
        \quad\text{and}\quad
        \abs{g_t} \geq \gamma
        \quad\text{and}\quad
        \gradsqat{x_t} \geq \gradsqat{x_1}.
    \end{align}
    This guarantees that, before $\tau^*$, (i) the iterates are always to the left of the minimizer, (ii) that the algorithm \eqref{eq:clippedSGD} never ``clips'' (i.e., $u_t = \nicefrac{\eta g_t}{\abs{g_t}}$), and (iii), $\min_{t < \tau^*} \gradsqat{x_t} = \gradsqat{x_1}$ (i.e., the algorithm never achieves any nontrivial target minimization criterion).
    Additionally, it must be the case that: 
    \begin{align}
        u_{\tau^*-1} < 0,
    \end{align}
    since $x_{\tau^*-1} < x_1$ and $x_{\tau^*}=x_{\tau^*-1}-u_{\tau^*-1} \geq x_1$.

    Now, let us distinguish the updates of \eqref{eq:signSGD}, \eqref{eq:normSGD}, and \eqref{eq:clippedSGD} as $(x_t(1),u_t(1))$, $(x_t(2),u_t(2))$, and $(x_t(3),u_t(3))$, respectively. Now, instead of reasoning about the dynamics of each of these algorithms individually, we instead reason about an algorithm with simpler dynamics, and draw conclusions about each of these processes via a stochastic dominance argument.
    
    To do this, we utilize the coupling of these algorithms defined in \cref{lem:negativeResultCoupling} -- namely, we let $x_1(i) = x_1 = \sqrt{\nicefrac{2\Delta}{\Lzero}}$ (as discussed above), and $g(x_t(i)) = \multnoiset \gradat{x_t(i)}$ for every $i$, where $\multnoiset$ is $-\eps$ with probability $1-\delta$, and $1+\nicefrac{\sigmaOne^2}{(1+\eps)}$ otherwise. That is, each process starts from the same initial iterate, and receives the same multiplicative noise on the stochastic gradient at time $t$.
    
    Similarly, let us define the ``simpler'' comparison process as:
    \begin{align*}
        u_s(4) = \begin{cases}
            \clipStepMult\eta & \text{if $\multnoiseat{s} = -\eps$}\\
            -\eta & \text{o.w.}
        \end{cases}
        \quad\text{and}\quad
        \clipStepMult := \frac{1}{1 + \frac{\gamma}{\eps\Lzero|x_1|}} = \frac{1}{1+\frac{\gamma}{\eps\sqrt{2\Delta\Lzero}}} \in [\nicefrac{1}{2}, 1]
    \end{align*}
    and take $x_1(4)=x_1$ and $x_{t+1}(4) = x_t(4)-u_t(4)$.
    Now, denote $\tau^*(i)$ as the stopping time from \eqref{eq:negResultStoppingTime} corresponding to the process $i\in [4]$.
    Then, by \cref{lem:negativeResultCoupling}, we have that, under our coupling of these algorithms, $\tau^*(4) \leq \min_{i\in[3]}\tau^*(i)$, which implies that, for each algorithm $i\in[3]$:
    \begin{align*}
        \PRO{\min_{t\in[T]} \gradsqat{x_t(i)} = \gradsqat{x_1}}
        \geq \PRO{\tau^*(i) > T} 
        \geq \PRO{\tau^*(4) > T},
    \end{align*}
    where the first inequality follows from \eqref{eq:negResultsBeforeStopping}. Thus, to lower bound the failure probability of algorithm $i$, it suffices to lower bound $\PRO{\tau^*(4) > T}$, and thus to reason only about the dynamics of this ``simpler'' process.

    By \cref{lem:negResultStoppingTimeProbBound}, we have that, for any $t_0 \geq 0$:
    \begin{align*}
        \PRO{\tau^*(4) > T}
        \geq (1-\delta)^{t_0}\parens{1-\sum_{t=t_0+2}^T\PRO{\frac{1}{t-t_0-1}(X_t - \EXP{X_t})  \leq -\parens{\delta_0-\delta}-\frac{\delta_0 t_0}{t-t_0-1}}},
    \end{align*}
    where $X_t$ is a sum of $t-t_0-1$ i.i.d Bernoulli random variables, each with mean $1-\delta = 1-\frac{1}{1+\nicefrac{\sigmaOne^2}{(1+\eps)^2}} > \nicefrac{1}{2}$ (since, by assumption, $\sigmaOne > (1+\eps)$), and $\delta_0 = \nicefrac{\clipStepMult}{(1+\clipStepMult)}$. We may therefore apply the Chernoff-Hoeffding inequality \citep[Theorem 1, Eq. (2.2)]{H63} to obtain:
    \begin{align*}
        &\PRO{\frac{1}{t-t_0-1}(X_t - \EXP{X_t})  \leq -\parens{\delta_0-\delta}-\frac{\clipStepMult t_0}{(1+\clipStepMult)(t-t_0-1)}}\\
        &\leq \exp\parens{-\frac{t-t_0-1}{2\delta(1-\delta)}\parens{\delta_0-\delta + \frac{\delta_0 t_0}{t-t_0-1}}^2}.
    \end{align*}
    Notice that, since $\sigmaOne^2 > \nicefrac{(1+\eps)^2}{\clipStepMult}=(1+\eps)^2(1+\nicefrac{\gamma}{\eps\sqrt{2\Delta\Lzero}})$,
        $\delta_0 - \delta
        = \frac{1}{2+\frac{\gamma}{\eps\sqrt{2\Delta\Lzero}}} - \frac{1}{1 + \frac{\sigmaOne^2}{(1+\eps)^2}} > 0,$
    which implies the above bound is always nontrivial.
    Thus, we can use that bound to obtain, for any $\ell > 0$:
    \begin{align*}
        &\sum_{t=t_0+2}^T\PRO{\frac{1}{t-t_0-1}(X_t - \EXP{X_t})  \leq -\parens{\delta_0-\delta}-\frac{\delta_0 t_0}{t-t_0-1}}\\
        &\leq \sum_{t=t_0+2}^{t_0+1+\floor{\frac{\delta_0 t_0}{\ell}}} \exp\parens{-\frac{t-t_0-1}{2\delta(1-\delta)}\parens{\delta_0-\delta + \ell}^2}
        + \sum_{t=t_0+2+\floor{\frac{\delta_0 t_0}{\ell}}}^{T} \exp\parens{-\frac{t-t_0-1}{2\delta(1-\delta)}\parens{\delta_0-\delta}^2}\\
        &= \sum_{t=0}^{\floor{\frac{\delta_0 t_0}{\ell}}-1} \exp\parens{-\frac{t + 1}{2\delta(1-\delta)}\parens{\delta_0-\delta + \ell}^2}
        + \sum_{t=0}^{T-t_0-2-\floor{\frac{\delta_0 t_0}{\ell}}} \exp\parens{-\frac{t + 1 + \floor{\frac{\delta_0 t_0}{\ell}}}{2\delta(1-\delta)}\parens{\delta_0-\delta}^2}\\
    \end{align*}
    Thus, using the geometric summation formula, we can bound the above summations as:
    \begin{align*}
        &\sum_{t=t_0+2}^T\PRO{\frac{1}{t-t_0-1}(X_t - \EXP{X_t})  \leq -\parens{\delta_0-\delta}-\frac{\delta_0 t_0}{t-t_0-1}}\\
        &\leq \exp\parens{-\frac{1}{2\delta(1-\delta)}\parens{\delta_0-\delta + \ell}^2} \frac{1 - \exp\parens{-\frac{\floor{\frac{\delta_0 t_0}{\ell}}}{2\delta(1-\delta)}\parens{\delta_0-\delta + \ell}^2}}{1 - \exp\parens{-\frac{1}{2\delta(1-\delta)}\parens{\delta_0-\delta + \ell}^2}}\\
        &\quad+ \exp\parens{-\frac{1 + \floor{\frac{\delta_0 t_0}{\ell}}}{2\delta(1-\delta)}\parens{\delta_0-\delta}^2} \frac{1}{1-\exp\parens{-\frac{1}{2\delta(1-\delta)}\parens{\delta_0-\delta}^2}}\\
        &\leq \frac{\exp\parens{-\frac{\ell^2}{2\delta(1-\delta)}}}{1 - \exp\parens{-\frac{\ell^2}{2\delta(1-\delta)}}} \parens{1 - \exp\parens{-\frac{\floor{\frac{\delta_0 t_0}{\ell}}}{2\delta(1-\delta)}\parens{\delta_0-\delta + \ell}^2}}\\
        &\quad+ \frac{\exp\parens{-\frac{1}{2\delta(1-\delta)}\parens{\delta_0-\delta}^2}}{1-\exp\parens{-\frac{1}{2\delta(1-\delta)}\parens{\delta_0-\delta}^2}}\exp\parens{-\frac{\floor{\frac{\delta_0 t_0}{\ell}}}{2\delta(1-\delta)}\parens{\delta_0-\delta}^2}.
    \end{align*}
    Now, let us focus on bounding the two terms in the above expression. To do this, we first observe that, for any $\mu,x>0$ and $i\geq 0$,
    \begin{align}\label{eq:negResultUsefulEquiv}
        \frac{\exp(-(1+i)x)}{1-\exp(-x)} \leq \mu
        \iff
        i \geq \frac{1}{x}\logp{\frac{\exp(-x)}{\mu(1-\exp(-x))}}
        \quad\text{or}\quad
        i=0 \quad\text{and}\quad x \geq \logp{1+\frac{1}{\mu}}.
    \end{align}
    Taking $i=0$ and $x = \nicefrac{\ell^2}{(2\delta(1-\delta))}$, the above implies that the first term is upper-bounded by $\mu=\nicefrac{\delta}{2}$ whenever $\ell\geq\sqrt{2\delta(1-\delta)\log(1+\nicefrac{2}{\delta})}$. For the second term, we take $i=\floor{\frac{\delta_0 t_0}{\ell}}$, $x=\nicefrac{(\delta_0-\delta)^2}{(2\delta(1-\delta))}$, and conclude that the second term is upper-bounded by $\mu=\nicefrac{\delta}{2}$ whenever
    \begin{align*}
        \floor{\frac{\delta_0 t_0}{\ell}} \geq \frac{2\delta(1-\delta)}{(\delta-\delta_0)^2}\logp{\frac{2\exp\parens{-\frac{(\delta_0-\delta)^2}{2\delta(1-\delta)}}}{\delta\parens{1-\exp\parens{-\frac{(\delta_0-\delta)^2}{2\delta(1-\delta)}}}}}.
    \end{align*}
    In particular, since $\exp(-x)<\nicefrac{1}{(1+x)}$ for any $x>0$, and thus also $\nicefrac{\exp(-x)}{1-\exp(-x)} < \nicefrac{1}{x}$, since $\floor{x} > x-1$, we have that the above inequality is satisfied whenever:
    \begin{align*}
        t_0 \geq \frac{\ell}{\delta_0}\parens{1+\frac{2\delta(1-\delta)}{(\delta-\delta_0)^2}\logp{\frac{4(1-\delta)}{(\delta_0-\delta)^2}}}.
    \end{align*}
    Therefore, we can choose $\ell = \sqrt{2\delta(1-\delta)\log(1+\nicefrac{2}{\delta})}$ and:
    \begin{align*}
        t_0 = \ceil{\frac{\sqrt{2\delta(1-\delta)\log(1+\nicefrac{2}{\delta})}}{\delta_0}\parens{1+\frac{2\delta(1-\delta)}{(\delta-\delta_0)^2}\logp{\frac{4(1-\delta)}{(\delta_0-\delta)^2}}}},
    \end{align*}
    and, combining our results, we conclude that, for any algorithm $i\in[3]$:
    \begin{align*}
        \PRO{\min_{t\in[T]} \gradsqat{x_t(i)} = \gradsqat{x_1}} > (1-\delta)^{t_0+1},
    \end{align*}
    as claimed.
\end{proof}

\begin{lemma}\label{lem:signSGDBadStepCharacterization}
    Consider the process $\braces{x_t}_{t\geq 1}$ from \eqref{eq:signSGD} as defined in \cref{lem:divergenceNormSGD}, where $x_1 < 0$, $\fat{x} := \nicefrac{\Lzero}{2}\ x^2$ for some $\Lzero > 0$, and $g_t$ are the stochastic gradients output by the oracle from \cref{lem:stochasticOracleConstruction}. Suppose that the parameter $\beta$ of \eqref{eq:signSGD} satisfies:
    \begin{align*}
        \beta \in \left[0, \frac{\eps}{1+\eps+\frac{\sigmaOne^2}{1 +\eps}}\right).
    \end{align*}
    Let $\tau^* = \min\braces{t > 1 : x_1 \leq x_t}$. Then, if $t < \tau^*$ and $g_t = -\eps\gradat{x_t}$, then $u_t = \eta$.
\end{lemma}
\begin{proof}
    Recall that, by construction of the stochastic gradient oracle from \cref{lem:stochasticOracleConstruction}, and since $\gradat{x} = \Lzero x$:
    \begin{align*}
        g(x) := \begin{cases}
            \parens{1 + \frac{\sigmaOne^2}{1+\eps}}\Lzero x & \text{w.p. } \frac{1}{1 + \frac{\sigmaOne^2}{(1+\eps)^2}} := \delta\\
            -\eps \Lzero x & \text{w.p. } 1-\frac{1}{1 + \frac{\sigmaOne^2}{(1+\eps)^2}} = 1-\delta.
        \end{cases}
    \end{align*}
    We wish to show that the process from \eqref{eq:signSGD} has the property that, whenever $g_t = -\eps\Lzero x_t$ and $x_s < 0$ for every $s\in [t]$, then $u_t = \eta$. 
    We consider any initialization $x_1 < 0$, and denote $\tau^*$ to be the first time when an iterate becomes non-negative, i.e.,
    \begin{align*}
        \tau^* = \min\braces{t > 1 : x_1 \leq x_t}.
    \end{align*}
    Further, take:
    \begin{align*}
        \tau_0 := 0
        \quad\text{and}\quad
        \tau_{i+1} := \min\braces{t > \tau_i : \sgradt = -\eps\Lzero x_t \text{ or } u_t = \eta}.
    \end{align*}
    Notice that, since $u_t \in \braces{\pm \eta}$ by definition of \eqref{eq:signSGD}, and by construction of the stochastic gradient oracle:
    \begin{align}\label{eq:intermediateDescentSteps}
        \sgradt = \parens{1 + \frac{\sigmaOne^2}{1+\eps}}\Lzero x_t
        \quad\text{and}\quad
        u_t = -\eta
        \quad
        \forall t \in (\tau_i, \tau_{i+1}) \ \forall i \geq 0.
    \end{align}
    Thus, it suffices to prove by induction that, for any $i\geq 0$, either $\tau_i \geq \tau^*$, or $u_{\tau_i} = \eta$, as long as
    \begin{align*}
        \beta < 1-\sqrt{1-\frac{\eps}{1 +\eps+\frac{\sigmaOne^2}{1+\eps}}} = 1-\sqrt{\frac{\eps}{1+\eps}\delta}.
    \end{align*}
    For the base case of $i=0$, we may assume without loss of generality that $u_{\tau_0} = u_0 = \eta$, since $m_0=0$ and the dynamics of the update rule do not depend on $u_0$ (i.e., the dynamics begin at time $t=1$ and $x_1$ is the starting point of the process). Thus, the base case is true by construction.

    Now, suppose the claim holds for some $i\geq 1$. Either $\tau_{i+1} \geq \tau^*$ or not. In the former case, the claim follows trivially, so let us assume that $\tau_{i+1} < \tau^*$. Since $\tau_i < \tau_{i+1} < \tau^*$ by construction, $u_{\tau_i} = \eta$ by the induction hypothesis. Further, let us assume that $g_{\tau_{i+1}} = -\eps\Lzero x_{\tau_{i+1}}$, since otherwise the claim again follows trivially by definition of $\tau_{i+1}$. Thus, we can write:
    \begin{align*}
        m_{\tau_{i+1}} 
        &= \beta^{\tau_{i+1} - \tau_i}m_{\tau_i} 
        +(1- \beta)\sum_{t=\tau_i+1}^{\tau_{i+1}} \beta^{\tau_{i+1} - t} g_t\\ 
        &= \beta^{\tau_{i+1} - \tau_i}m_{\tau_i} 
        +(1- \beta)\Lzero\parens{1+\frac{\sigmaOne^2}{1+\eps}}\sum_{t=\tau_i+1}^{\tau_{i+1}-1} \beta^{\tau_{i+1} - t} (x_{\tau_i+1} + \eta(t - \tau_i - 1))\\ 
        &\quad-(1-\beta)\Lzero \eps (x_{\tau_i+1} + \eta(\tau_{i+1} - \tau_i - 1))
    \end{align*}
    where the first equality is the definition of $m_{\tau_{i+1}}$. The second inequality follows from observation \eqref{eq:intermediateDescentSteps}. Further, since $u_{\tau_i} = \eta$, then by definition of \eqref{eq:signSGD}, either $m_{\tau_i} > 0$, or $m_{\tau_i} = 0$ and the algorithm chooses $u_{\tau_i}=\eta$. In either case, $m_{\tau_i} \geq 0$. Therefore, since, for $\beta\in [0,1)$:
    \begin{align*}
        &(1- \beta)\sum_{t=\tau_i+1}^{\tau_{i+1}-1} \beta^{\tau_{i+1} - t} (x_{\tau_i+1} + \eta(t - \tau_i - 1))\\
        &= \beta(x_{\tau_i+1} + \eta(\tau_{i+1}-\tau_i-1))
        -\beta^{\tau_{i+1}-\tau_i} x_{\tau_i+1}
        -\beta\eta\frac{1-\beta^{\tau_{i+1}-\tau_i-1}}{1-\beta},
    \end{align*}
    we obtain, using the fact that $x_{\tau_{i+1}} = x_{\tau_i+1} + \eta(\tau_{i+1} - (\tau_i+1))$ and $m_{\tau_i}\geq 0$:
    \begin{align*}
        \frac{m_{\tau_{i+1}}}{\Lzero} 
        &=\frac{\beta^{\tau_{i+1}-\tau_i}m_{\tau_i}}{\Lzero} -\parens{(1-\beta)\eps - \beta\parens{1+\frac{\sigmaOne^2}{1+\eps}}}(x_{\tau_i+1}+\eta(\tau_{i+1}-\tau_i-1))\\
        &\quad-\beta^{\tau_{i+1}-\tau_i}\parens{1+\frac{\sigmaOne^2}{1+\eps}}x_{\tau_i+1}\\
        &\quad-\beta\eta\parens{1+\frac{\sigmaOne^2}{1+\eps}}\frac{1-\beta^{\tau_{i+1}-\tau_i-1}}{1-\beta}\\
        &\geq-\parens{(1-\beta)\eps - \beta\parens{1+\frac{\sigmaOne^2}{1+\eps}}}(x_{\tau_{i+1}}+\eta)\\
        &\quad-\beta^{\tau_{i+1}-\tau_i}\parens{1+\frac{\sigmaOne^2}{1+\eps}}x_{\tau_i+1}\\
        &\quad+\eta\parens{(1-\beta)\eps - \beta\parens{1+\frac{\sigmaOne^2}{1+\eps}}\parens{1+\frac{1-\beta^{\tau_{i+1}-\tau_i-1}}{1-\beta}}}.
    \end{align*}
    Thus, since $x_{\tau_{i}+1} \leq x_{\tau_{i+1}} < 0$, and since $\tau_{i+1} < \tau^*$ (which implies, since each update of \eqref{eq:signSGD} satisfies $u_t\in\braces{\pm \eta}$ and by definition of $\tau^*$, $x_{\tau_{i+1}} \leq x_{\tau^*-1} = x_1-\eta < 0$), the above inequality implies that $m_{\tau_{i+1}} > 0$ as long as:
    \begin{align*}
        (1-\beta)\eps - \beta\parens{1+\frac{\sigmaOne^2}{1+\eps}}\parens{1+\frac{1-\beta^{\tau_{i+1}-\tau_i-1}}{1-\beta}}
        > (1-\beta)\eps - \beta\parens{1+\frac{\sigmaOne^2}{1+\eps}}\parens{1+\frac{1}{1-\beta}} > 0.
    \end{align*}
    Since we require $0 \leq \beta < 1$, the second inequality is equivalent to:
    \begin{align*}
        (1-\beta)^2\eps > \parens{1+\frac{\sigmaOne^2}{1+\eps}}\beta(2-\beta),
    \end{align*}
    which is satisfied as long as:
    \begin{align*}
        \beta 
        < 1 - \sqrt{\frac{1+\frac{\sigmaOne^2}{1+\eps}}{1+\eps+\frac{\sigmaOne^2}{1+\eps}}}
        = 1 - \sqrt{1 - \frac{\eps}{1+\eps}\frac{1}{1+\frac{\sigmaOne^2}{1+\eps}}}
        = 1 - \sqrt{1 - \frac{\eps}{1+\eps}\delta}.
    \end{align*}
    Thus, since $\sqrt{1-x} < 1-\nicefrac{x}{2}$ for $0 < x\leq 1$, it suffices to choose $\beta$ as:
    \begin{align*}
        \beta \leq \frac{\eps}{2(1+\eps)}\delta < 1-\sqrt{1-\frac{\eps}{1+\eps}\delta}.
    \end{align*}
    In this case, $m_{\tau_{i+1}} > 0$, and thus $u_t = \eta$, which establishes the induction step. Thus, for every $i\geq 0$, either $\tau_i \geq \tau^*$ or $u_{\tau_i} = \eta$, as claimed.
\end{proof}

\begin{lemma}\label{lem:negativeResultCoupling}
    Let us recall the i.i.d random process $\braces{\multnoiset}_{t\geq 1}$ from \cref{lem:stochasticOracleConstruction}, where each $\multnoiset$ is $-\eps$ with probability $1-\delta$, and $(1+\nicefrac{\sigmaOne^2}{(1+\eps)})$ otherwise.
    Let us distinguish the three processes from \cref{lem:divergenceNormSGD} (\cref{eq:signSGD,eq:clippedSGD,eq:normSGD}) as, respectively, $\braces{x_t(i)}_{t\geq 1}$ for $i\in[3]$.
    Consider the coupling of these three processes, where $x_1(i) = x_1 := -\sqrt{\nicefrac{2\Delta}{\Lzero}}$ for every $i\in[4]$, and for each $t\geq 1$ and $i\in [3]$, $g(x_t(i)) = \multnoiset \gradat{x_t(i)}$.
    Further, let us denote, for each $t\geq 1$:
    \begin{align*}
        u_t(4) = \begin{cases}
            \clipStepMult \eta & \text{if $\multnoiset=-\eps$}\\
            -\eta & \text{o.w.}
        \end{cases}
        \quad\text{where}\quad
        \clipStepMult := \frac{1}{1+\frac{\gamma}{\eps\sqrt{2\Delta\Lzero}}} \in [\nicefrac{1}{2},1],
    \end{align*}
    and take $x_1(4)=x_1$ and $x_{t+1}(4)=x_t(4)-u_t(4)$.
    Further, let, for each $i\in[4]$,
    \begin{align*}
        \tau^*(i) = \min\braces{t > 1 : x_1 \leq x_t(i)}.
    \end{align*}
    Then, under the constraints on parameters of the three algorithms as imposed in \cref{lem:divergenceNormSGD}, we have that:
    \begin{align*}
        \tau^*(4) \leq \min_{i\in[3]}\tau^*(i).
    \end{align*}
\end{lemma}
\begin{proof}
    We claim that, for each $i\in [3]$, and any $t < \tau^*(i)$, $u_t(4) \leq u_t(i)$. Notice that, supposing this claim is true, then $\tau^*(4)\leq \tau^*(i)$ for each $i\in[3]$, since, by definition of $\tau^*(i)$:
    \begin{align*}
        x_1 
        \leq x_{\tau^*(i)}(i)
        = x_1 - \sum_{s=1}^{\tau^*(i)-1} u_s(i)
        \leq x_1 - \sum_{s=1}^{\tau^*(i)-1} u_s(4)
        = x_{\tau^*(i)}(4).
    \end{align*}
    Thus, since $\tau^*(i)$ is the first time $t>1$ for which $x_{t}(i) \geq x_1$, it follows that $\tau^*(4) \leq \tau^*(i)$.
    Having established this implication, it suffices to prove the claim for each of the $u_t(i)$s.
    
    For the case of $i=1$ (i.e., algorithm \eqref{eq:signSGD}), this follows immediately from \cref{lem:signSGDBadStepCharacterization}, since this result tells us that whenever $t < \tau^*(1)$ and $\multnoiset = -\eps$, then $u_t(1) = \eta > \eta\clipStepMult = u_t(4)$. Otherwise, whenever $t < \tau^*(1)$ and $\multnoiset = (1 + \nicefrac{\sigmaOne^2}{(1+\eps)})$, then by construction, $u_t(4) = -\eta$, while $u_t(1) \in \braces{\pm \eta}$.

    For the case of $i=2$ (i.e., algorithm \eqref{eq:normSGD}), for every $t < \tau^*(2)$, since $\abs{g(x_t(2))} \geq \eps\Lzero|x_1| \geq \gamma$ by \eqref{eq:negResultsBeforeStopping} and since $\nicefrac{x}{(x + y)}$ is non-decreasing in $x$ on the interval $x\in(0,\infty)$ for any fixed $y \geq 0$,
    \begin{align*}
        \eta \geq \abs{u_t(2)} 
        = \eta \frac{\abs{g(x_t(2))}}{\gamma + \abs{g(x_t(2))}}
        \geq \eta\frac{\eps\Lzero\abs{x_1}}{\gamma + \eps\Lzero\abs{x_1}}
        = \frac{\eta}{\frac{\gamma}{\eps\Lzero\abs{x_1}} + 1}
        = \clipStepMult\eta
        \geq \frac{\eta}{2}.
    \end{align*}
    Thus, when $t < \tau^*(2)$ and $\multnoiset= -\eps$, $u_t(2) \geq \clipStepMult\eta = u_t(4)$, and when $t < \tau^*(2)$ and $g_t = (1+\nicefrac{\sigmaOne^2}{(1+\eps)})\Lzero x_t$, $u_t(2) \geq - \eta = \tu_t$.

    For the case of $i=3$ (i.e., algorithm \eqref{eq:clippedSGD}), for every $t < \tau^*$, $\abs{g_t} > \gamma$ by \eqref{eq:negResultsBeforeStopping}, which implies that $\abs{u_t} = \nicefrac{\eta\abs{g_t}}{\abs{g_t}} = \eta$. 
    Thus, when $t < \tau^*$ and $g_t = -\eps\Lzero x_t$ (notice $g_t > 0$ in this case), $u_t(3) = \eta \geq \clipStepMult\eta = \tu_t$, and when $t < \tau^*$ and $g_t = (1+\nicefrac{\sigmaOne^2}{(1+\eps)})\Lzero x_t$, $u_t(3) = - \eta = \tu_t$. 
    
    Therefore, the claim is established in all three cases, which also concludes the proof.
\end{proof}

\begin{lemma}\label{lem:negResultStoppingTimeProbBound}
    Consider the algorithm $4$ as defined in \cref{lem:negativeResultCoupling}. Then, under the assumptions of \cref{lem:divergenceNormSGD}, we have that, for any $T\geq 1$ and any $t_0 \geq 0$,
    \begin{align*}
        \PRO{\tau^*(4) > T}
        \geq (1-\delta)^{t_0}\parens{1-\sum_{t=t_0+2}^T \PRO{\frac{1}{t-t_0-1}(X_t - \EXP{X_t}) \leq - (\delta_0-\delta) - \frac{\delta_0 t_0}{t - t_0 -1}}},
    \end{align*}
    where $X_t = \sum_{s=t_0+1}^{t-1}\1{\E_s}$ is a sum of $t-t_0-1$ i.i.d Bernoulli random variables with mean $1-\delta=1-\frac{1}{(1+\nicefrac{\sigmaOne^2}{(1+\eps)^2})}$. $\nicefrac{1}{\clipStepMult} := 1 + \nicefrac{\gamma}{\eps\sqrt{2\Delta\Lzero}})$ and $\delta_0 = \nicefrac{1}{1+\nicefrac{1}{\clipStepMult}}$.
\end{lemma}
\begin{proof}
    Recall the construction of algorithm $4$ from \cref{lem:negativeResultCoupling}.
    Denote $\E_s = \braces{\multnoiseat{s} = (1+\nicefrac{\sigmaOne^2}{(1+\eps)})}$, and recall that $\PRO{\E_s}=\delta =\frac{1}{1+\nicefrac{\sigmaOne^2}{(1+\eps)^2}}$. Let us write:
    \begin{align*}
        \tN_{t_1,t_2}
        = -(x_t(4) - x_1)
        =\sum_{s=t_1}^{t_2} u_s(4)
        =\sum_{s=t_1}^{t_2} -\eta\1{\E_s} + \clipStepMult\eta\1{\E_s^c},
    \end{align*}
    as the ``net movement'' of algorithm $4$ to the left of $x_{t_1}$ after $t_2-t_1 +1$ time steps.
    and observe that
    \begin{align*}
        \EXP{\tN_{t_1,t_2}} 
        = \sum_{s=t_1}^{t_2} -\eta\delta + \clipStepMult\eta(1-\delta)
        = \clipStepMult\eta\parens{1-\parens{1 + \frac{1}{\clipStepMult}}\delta} (t_2-t_1+1).
    \end{align*}
    Additionally, note that, recalling the definition of $\tau^*(4)$ from \cref{eq:negResultStoppingTime},
    \begin{align*}
        \braces{\tau^*(4) > T}
        = \braces{\forall t \in [2,T] : x_t(4) < x_1}
        &= \braces{\forall t \in [2,T] : -(x_t(4)-x_1) > 0}\\
        &= \braces{\forall t \in [2,T] : \tN_{1,t-1} > 0}.
    \end{align*}
    Therefore, we have that, for any $t_0\geq 0$,
    \begin{align*}
        \PRO{\tau^*(4) > T}
        = \PRO{\forall t \in [2,T] : \tN_{1,t-1} > 0}
        \geq \PRO{\braces{\forall t \in [2,T] : \tN_{1,t-1} > 0} \cap \bigcap_{s\in[t_0]} \E_s^c}.
    \end{align*}
    Further, since the stochastic gradient of algorithm $4$ uses i.i.d multiplicative noise at each round (i.e., the events $\braces{\E_s}_{s\in[T]}$ are mutually independent and $\PRO{\E_s}=\delta$ for every $s$), and since the event $\E_s^c$ implies that $x_{s+1}(i) = x_s(i) - \clipStepMult\eta$ for each algorithm $i$, we have that for any $t_0 \geq 0$,
    \begin{align*}
        &\PRO{\braces{\forall t \in [2,T] : \tN_{1,t-1} > 0} \cap \bigcap_{s\in[t_0]} \E_s^c}\\
        &=\PRO{\braces{\forall t \in [t_0+2,T] : \tN_{t_0+1,t-1} > -\clipStepMult t_0\eta} \cap \bigcap_{s\in[t_0]} \E_s^c}\\
        &=(1-\delta)^{t_0}\PRO{\forall t \in [t_0+2,T] : \tN_{t_0+1,t-1} > -\clipStepMult t_0\eta}.
    \end{align*}
    Now, since
    \begin{align*}
        &\PRO{\forall t \in [t_0+2,T] : \tN_{t_0+1,t-1} > -\clipStepMult t_0\eta}\\
        &=1-\PRO{\exists t \in [t_0+2,T] : \tN_{t_0+1,t-1} \leq -\clipStepMult t_0\eta}\\
        &\geq 1-\sum_{t=t_0+2}^T\PRO{\tN_{t_0+1,t-1} \leq -\clipStepMult t_0\eta},
    \end{align*}
    it remains only to upper-bound each probability inside of the above summation.
    To do this, let us denote, for any $t \in [t_0+2, T]$, 
    \begin{align*}
        X_t 
        = \sum_{s=t_0+1}^{t-1} \1{\E_s^c}
        = \frac{1}{(1+\clipStepMult)\eta} \tN_{t_0+1,t-1} + \frac{t-t_0-1}{1+\clipStepMult}.
    \end{align*}
    Thus, $X_t$ is a sum of i.i.d Bernoulli random variables, each with mean $1-\delta = 1-\frac{1}{1+\nicefrac{\sigmaOne^2}{(1+\eps)^2}} > \nicefrac{1}{2}$ (since, by assumption, $\sigmaOne > (1+\eps)$). We may therefore apply \citep[Theorem 1, Eq. (2.2)]{H63}, denoting $\delta_0 := 1-\frac{1}{1+\clipStepMult} = \nicefrac{\clipStepMult}{(1+\clipStepMult)}$, to obtain:
    \begin{align*}
        \PRO{\tN_{t_0+1,t-1} \leq -\clipStepMult t_0\eta}
        &=\PRO{(1+\clipStepMult)\eta X_t - \eta (t-t_0-1) \leq -\clipStepMult t_0\eta}\\
        &=\PRO{\frac{1}{t-t_0-1}(X_t - \EXP{X_t})  \leq -\parens{\delta_0-\delta}-\frac{\delta_0 t_0}{t-t_0-1}}.
    \end{align*}
    Collecting the above results, we arrive at the claimed lower bound.
\end{proof}

\subsection{Full statement and proof for negative result for \eqref{eq:alg} in the ``large $\sigmaOne$'' regime}

\begin{restatable}[Formal statement of \cref{lem:divergenceAGNormMain}]{lemma}{restateDivergenceAGNorm}\label{lem:divergenceAGNorm}
    Fix any $\Lone > 0$, $x_1 \in \R$, and $\sigmaOne > 1$. Let $T\geq 1$, $\eta > 0$, $\eps\in (0,1)$ and $0 < b_0^2 \leq \eps^2\Lone^2\exp(2\Lone x_1)$ be arbitrary parameters (possibly dependent on $\Lone, x_1$, and $\sigmaOne$). Then, there exists a $1$-dimensional $(0, (e-1)\Lone)$-smooth function such that $\fstar = 0$, and a stochastic gradient oracle satisfying \cref{assump:unbiasedGrad,assump:affineVariance} with $\sigmaZero=0$ and the specified $\sigmaOne$, such that, if \eqref{eq:alg} is run for $T$ time steps using parameters $\eta$ and $b_0^2$, then then the resulting iterates $\braces{x_t}_{t\in[T]}$ satisfy:
    \begin{align*}
        \PRO{\min_{t\in[T]} \gradsqat{x_t} = \gradsqat{x_1}} \geq \parens{1 - \frac{1}{1 + \frac{\sigmaOne^2}{(1+\eps)^2}}}^{t_0},
    \end{align*}
    where 
    \begin{align*}
        t_0 = \parens{1+\sqrt{2} + \frac{\log(T-1)}{2\eta\Lone}}^2 - 2.
    \end{align*}
    In particular, whenever $\eta \geq \nicefrac{\alpha}{\Lone}$ for some $\alpha > 0$, and, for any $\delta \in (0,1)$,
    \begin{align*}
        \sigmaOne^2 \geq \frac{1}{\log(\nicefrac{1}{(1-\delta)})}(1+\eps)^2 \parens{\parens{1 + \sqrt{2} + \frac{\log(T-1)}{\alpha}}^2 - 2},
    \end{align*}
    then
    \begin{align*}
        \PRO{\min_{t\in[T]} \gradsqat{x_t} = \gradsqat{x_1}} \geq 1-\delta.
    \end{align*}
\end{restatable}
\begin{proof}
    Let $\fat{x} = \exp(\Lone x)$. Notice that, since $\hessat{x} = \Lone \gradat{x}$, it follows from \cref{prop:smoothPartialEquiv} that $\fat{\cdot}$ is $(0,(e-1)\Lone)$-smooth. Clearly $\fstar = \inf_{x\in\R}\exp(\Lone x) = 0$. Further, consider the stochastic gradient oracle from \cref{lem:stochasticOracleConstruction}, which, for the iterate $x_t$ at time $t$, first draws an i.i.d sample:
    \begin{align*}
        \multnoiset = \begin{cases}
            - \eps & \text{w.p. $1-\delta = 1-\frac{1}{1+\nicefrac{\sigmaOne^2}{(1+\eps)^2}}$}\\
            \parens{1 + \frac{\sigmaOne^2}{1 + \eps}} & \text{w.p. $\delta = \frac{1}{1+\nicefrac{\sigmaOne^2}{(1+\eps)^2}}$},
        \end{cases}
    \end{align*}
    and $g(x_t) = \multnoiset \gradat{x_t}$.
    As established in \cref{lem:stochasticOracleConstruction}, this oracle satisfies \cref{assump:unbiasedGrad,assump:affineVariance} with $\sigmaZero = 0$ and the specified $\sigmaOne > 1$.

    Let us define, for a parameter $t_0 \geq 1$ to be determined shortly:
    \begin{align*}
        \ncEvent := \braces{\forall t\in[t_0] : g(x_t) = -\eps\gradat{x_t}}.
    \end{align*}
    Now, since the noise is sampled i.i.d at each time step, we have that, for any $t_0 \geq 0$:
    \begin{align*}
        \PRO{\ncEvent}
        =\PRO{\forall t\in[t_0] : \multnoiset = -\eps}
        =(1-\delta)^{t_0}.
    \end{align*}
    Whenever $\ncEvent$ is true, notice that:
    \begin{align*}
        \gradat{x_{t_0+1}}
        = \Lone\exp\parens{\Lone x_{t_0 + 1}}
        &= \Lone\exp\parens{\Lone x_{1} + \Lone\sum_{t=1}^{t_0} x_{t+1}-x_t}\\
        &= \Lone\exp\parens{\Lone x_{1} + \Lone\eta\sum_{t=1}^{t_0} \frac{g(x_t)}{\sqrt{b_0^2 + \sum_{s=1}^t \normSq{g(x_s)}}}}\\
        &= \Lone\exp\parens{\Lone x_{1} + \Lone\eta\sum_{t=1}^{t_0} \frac{\eps\gradat{x_t}}{\sqrt{b_0^2 + \sum_{s=1}^t \eps^2\normSq{\gradat{x_s}}}}}.
    \end{align*}
    Now, using the fact that, whenever $\ncEvent$ is true, then $\gradat{x_t} \leq \gradat{x_{t+1}}$ for each $t\in[t_0]$, and assuming $b_0^2 \leq \eps^2 \normSq{\gradat{x_1}}$, we can bound
    \begin{align*}
        \sum_{t=1}^{t_0} \frac{\eps\gradat{x_t}}{\sqrt{b_0^2 + \sum_{s=1}^t \eps^2\normSq{\gradat{x_s}}}}
        &\geq \sum_{t=1}^{t_0} \frac{\eps^2\gradat{x_t}}{\sqrt{\eps\normSq{\gradat{x_1}} + \eps^2 t\normSq{\gradat{x_t}}}}\\
        &\geq \sum_{t=1}^{t_0} \frac{1}{\sqrt{t+1}}\\
        &\geq \int_{2}^{t_0+2} \frac{1}{\sqrt{t}} dt\\
        &= 2(\sqrt{t_0 + 2} - \sqrt{2}).
    \end{align*}
    Thus, we conclude that:
    \begin{align*}
        \gradat{x_{t_0+1}} \geq \Lone\exp(\Lone(x_1 + 2\eta\sqrt{t_0+2} - 2\eta\sqrt{2})).
    \end{align*}
    Now, for a parameter $\alpha > 0$ to be determined shortly, let us define:
    \begin{align*}
        \tau_0 = \min\braces{t \geq t_0 : \gradat{x_{t+1}} \leq \gradat{x_{t_0+1}}\exp(-\Lone\eta\alpha)},
    \end{align*}
    and let, for each $i\geq 0$,
    \begin{align*}
        \tau_{i+1} = \min\braces{t \geq \tau_{i} : \gradat{x_{t+1}} < \gradat{x_{\tau_{i}}}}.
    \end{align*}
    Notice that, by construction, $\multnoiseat{\tau_i} = 1+\nicefrac{\sigmaOne^2}{(1+\eps)}$ for every $i\geq 0$. Further, $x_{\tau_i+1} \leq x_{\tau_{i+1}}$ since $\tau_{i+1}$ is the first time after $\tau_i$ satisfying $\gradat{x_{\tau_{i+1}+1}} < \gradat{x_{\tau_i+1}}$, or equivalently, $x_{\tau_{i+1}+1} < x_{\tau_{i}+1}$. This implies that
    \begin{align*}
        x_{\tau_i + 1}
        &= x_{\tau_0+1} + \sum_{j=0}^{i-1} x_{\tau_{j+1}+1} - x_{\tau_j+1}
        \geq x_{\tau_0+1} + \sum_{j=0}^{i-1} x_{\tau_{j+1}+1} - x_{\tau_{j+1}}\\
        &= x_{\tau_0+1} - \sum_{j=0}^{i-1} \frac{\eta \parens{1 + \frac{\sigmaOne^2}{1+\eps}}\gradat{x_{\tau_{j+1}}}}{\sqrt{b_0^2 + \sum_{s=1}^t g_t^2}}
        \geq x_{\tau_0+1} - \sum_{j=0}^{i-1} \frac{\eta \parens{1 + \frac{\sigmaOne^2}{1+\eps}}\gradat{x_{\tau_{j+1}}}}{\sqrt{\parens{1+\frac{\sigmaOne^2}{1+\eps}}^2\normSq{\gradat{x_{t_{0}+1}}}}}.
    \end{align*}
    Now, notice that:
    \begin{align*}
         \gradat{x_{\tau_{j+1}}} 
         &= \Lone\exp(\Lone x_{\tau_{j+1}+1} + \Lone(x_{\tau_{j+1}} - x_{\tau_{j+1}+1}))\\
         &= \gradat{x_{\tau_{j+1}+1}}\exp(\Lone(x_{\tau_{j+1}} - x_{\tau_{j+1}+1}))\\
         &< \gradat{x_{\tau_{0}+1}}\exp(\Lone(x_{\tau_{j+1}} - x_{\tau_{j+1}+1}))\\
         &\leq \gradat{x_{t_{0}+1}}\exp(\Lone(x_{\tau_{j+1}} - x_{\tau_{j+1}+1} - \eta\alpha))\\
         &\leq \gradat{x_{t_{0}+1}}\exp(-\eta\Lone(\alpha-1)),
    \end{align*}
    from which we obtain the bound:
    \begin{align*}
        x_{\tau_i + 1} 
        \geq x_{\tau_0+1} - \eta\sum_{j=0}^{i-1}\exp(-\eta\Lone(\alpha-1))
        = x_{\tau_0+1} - i\eta\exp(-\eta\Lone(\alpha-1)).
    \end{align*}
    Now, by construction of the $\tau_i$, we have that, assuming $\ncEvent$ is true, then
    \begin{align*}
        \min_{t\in[T]} \gradsqat{x_t}
        &= \min_{t\in[t_0+2, T]} \min\braces{\gradsqat{x_1}, \gradsqat{x_t}}\\
        &\geq \min_{i\in[0,T-1]} \min\braces{\gradsqat{x_1}, \gradsqat{x_{\tau_i+1}}}.
    \end{align*}
    Thus, to ensure that $\min_{t\in[T]}\gradsqat{x_t} = \gradsqat{x_1}$, it suffices to have that, for every $i < T$, $\gradsqat{x_{\tau_i}} \geq \gradsqat{x_1}$. Now, notice that
    \begin{align*}
        \gradat{x_{\tau_i+1}}
        &= \Lone \exp(\Lone(x_{\tau_i+1}))\\
        &\geq \Lone \exp(\Lone(x_{\tau_0+1}) - i\eta\exp(-\eta\Lone(\alpha-1)))\\
        &= \gradat{x_{\tau_0 + 1}}\exp(-i\Lone\eta\exp(-\eta\Lone(\alpha-1)))\\
        &= \gradat{x_{\tau_0}}\exp(\Lone(x_{\tau_0+1} - x_{\tau_0}) - i\Lone\eta\exp(-\eta\Lone(\alpha-1)))\\
        &> \gradat{x_{t_0+1}}\exp(-\Lone\eta\alpha + \Lone(x_{\tau_0+1} - x_{\tau_0}) - i\Lone\eta\exp(-\eta\Lone(\alpha-1)))\\
        &\geq \gradat{x_{1}}\exp(2\eta\Lone(\sqrt{t_0+2} - \sqrt{2})-\Lone\eta(\alpha+1) - i\Lone\eta\exp(-\eta\Lone(\alpha-1))).
    \end{align*}
    Thus, it suffices to establish conditions under which
    \begin{align*}
        2\sqrt{t_0 + 2} \geq 2\sqrt{2} + \alpha+1 + i\exp(-\Lone\eta(\alpha-1)).
    \end{align*}
    Thus, if we choose $\alpha = \nicefrac{\log(T-1)}{\eta\Lone}$, then it suffices to take:
    \begin{align*}
        t_0 = \parens{1+\sqrt{2} + \frac{\log(T-1)}{2\eta\Lone}}^2 - 2,
    \end{align*}
    in which case $\min_{t\in[T]} \gradsqat{x_t} = \gradsqat{x_1}$ under $\ncEvent$. Hence,
    \begin{align*}
        \PRO{\min_{t\in[T]} \gradsqat{x_t} = \gradsqat{x_1}}
        &\geq \PRO{\ncEvent}\\
        &= \parens{1-\frac{1}{1 + \frac{\sigmaOne^2}{1+\eps}}}^{t_0}.
    \end{align*}
    In particular, using the fact that $1 - x > \exp(\nicefrac{-x}{(1-x)})$ for $x<1$, it follows that:
    \begin{align*}
        \PRO{\min_{t\in[T]} \gradsqat{x_t} = \gradsqat{x_1}}
        &\geq \exp\parens{-\frac{(1+\eps)^2 t_0}{\sigmaOne^2}}\\
        &\geq \exp\parens{-\frac{(1+\eps)^2 \parens{\parens{1 + \sqrt{2} + \frac{\log(T-1)}{\eta\Lone}}^2 - 2}}{\sigmaOne^2}}
    \end{align*}
    Hence, as long as, for some $\delta \in (0,1)$,
    \begin{align*}
        \sigmaOne^2 \geq \frac{1}{\log(\nicefrac{1}{(1-\delta)})}(1+\eps)^2 \parens{\parens{1 + \sqrt{2} + \frac{\log(T-1)}{\eta\Lone}}^2 - 2},
    \end{align*}
    then $\PRO{\min_{t\in[T]} \gradsqat{x_t} = \gradsqat{x_1}} \geq 1-\delta$.
\end{proof}

\end{document}